\documentclass
[12pt,letterpaper]
{article}

\usepackage{setspace}
\usepackage{etex}
\usepackage{xspace,enumerate}
\usepackage[dvipsnames]{xcolor}
\usepackage[T1]{fontenc}
\usepackage[full]{textcomp}
\usepackage[american]{babel}
\usepackage{mathtools}
\usepackage{amsthm}
\newtheorem{theorem}{Theorem}[section]
\newtheorem*{theorem*}{Theorem}

\newtheorem*{proposition*}{Proposition}
\newtheorem{lemma}[theorem]{Lemma}
\newtheorem*{lemma*}{Lemma}
\newtheorem{corollary}[theorem]{Corollary}
\newtheorem*{conjecture*}{Conjecture}
\newtheorem{fact}[theorem]{Fact}
\newtheorem*{fact*}{Fact}

\newtheorem*{hypothesis*}{Hypothesis}

\theoremstyle{definition}
\newtheorem{definition}[theorem]{Definition}
\newtheorem*{definition*}{Definition}

\newtheorem{algorithm}[theorem]{Algorithm}

\theoremstyle{remark}

\newtheorem*{claim*}{Claim}
\newtheorem{remark}[theorem]{Remark}
\newtheorem*{remark*}{Remark}

\newtheorem*{observation*}{Observation}
\usepackage[
letterpaper,
top=1.2in,
bottom=1.2in,
left=1in,
right=1in]{geometry}
\usepackage{newpxtext} 
\usepackage{textcomp} 
\usepackage[varg,bigdelims]{newpxmath}
\usepackage[scr=rsfso]{mathalfa}
\usepackage{bm} 
\let\mathbb\varmathbb
\usepackage{microtype}
\usepackage[
pagebackref,
colorlinks=true,
urlcolor=blue,
linkcolor=blue,
citecolor=OliveGreen,
]{hyperref}
\usepackage[capitalise,nameinlink]{cleveref}
\crefname{lemma}{Lemma}{Lemmas}
\crefname{fact}{Fact}{Facts}
\crefname{theorem}{Theorem}{Theorems}
\crefname{corollary}{Corollary}{Corollaries}
\crefname{claim}{Claim}{Claims}
\crefname{example}{Example}{Examples}
\crefname{algorithm}{Algorithm}{Algorithms}
\crefname{problem}{Problem}{Problems}
\crefname{definition}{Definition}{Definitions}
\usepackage{paralist}
\usepackage{turnstile}
\usepackage{mdframed}
\usepackage{algorithm}
\usepackage{algpseudocode}
\usepackage{tikz}

\usepackage{boxedminipage}

\newcommand{\Paren}[1]{\left(#1\right)}

\newcommand{\Brac}[1]{\left[#1\right]}

\newcommand{\Abs}[1]{\left\lvert#1\right\rvert}

\newcommand{\Norm}[1]{\left\lVert#1\right\rVert}

\newcommand{\iprod}[1]{\langle#1\rangle}

\newcommand{\Esymb}{\mathbb{E}}
\newcommand{\Psymb}{\mathbb{P}}

\DeclareMathOperator*{\E}{\Esymb}

\DeclareMathOperator*{\ProbOp}{\Psymb}
\renewcommand{\Pr}{\ProbOp}

\newcommand{\tensor}{\otimes}

\newcommand{\defeq}{\stackrel{\mathrm{def}}=}

\newcommand{\mper}{\,.}

\newcommand\bdot\bullet

\DeclareMathOperator{\Ind}{\mathbf 1}

\DeclareMathOperator{\Tr}{Tr}

\DeclareMathOperator{\poly}{poly}

\newcommand{\N}{\mathbb N}
\newcommand{\R}{\mathbb R}

\newcommand{\cD}{\mathcal D}

\newcommand{\cN}{\mathcal N}

\renewcommand{\leq}{\leqslant}
\renewcommand{\le}{\leqslant}
\renewcommand{\geq}{\geqslant}
\renewcommand{\ge}{\geqslant}
\let\epsilon=\varepsilon
\numberwithin{equation}{section}
\newcommand\MYcurrentlabel{xxx}
\newcommand{\MYstore}[2]{%
  \global\expandafter \def \csname MYMEMORY #1 \endcsname{#2}%
}
\newcommand{\MYload}[1]{%
  \csname MYMEMORY #1 \endcsname%
}
\newcommand{\MYnewlabel}[1]{%
  \renewcommand\MYcurrentlabel{#1}%
  \MYoldlabel{#1}%
}
\newcommand{\MYdummylabel}[1]{}
\newcommand{\torestate}[1]{%
  \let\MYoldlabel\label%
  \let\label\MYnewlabel%
  #1%
  \MYstore{\MYcurrentlabel}{#1}%
  \let\label\MYoldlabel%
}
\newcommand{\restatetheorem}[1]{%
  \let\MYoldlabel\label
  \let\label\MYdummylabel
  \begin{theorem*}[Restatement of \cref{#1}]
    \MYload{#1}
  \end{theorem*}
  \let\label\MYoldlabel
}
\newcommand{\restatelemma}[1]{%
  \let\MYoldlabel\label
  \let\label\MYdummylabel
  \begin{lemma*}[Restatement of \cref{#1}]
    \MYload{#1}
  \end{lemma*}
  \let\label\MYoldlabel
}
\newcommand{\restateprop}[1]{%
  \let\MYoldlabel\label
  \let\label\MYdummylabel
  \begin{proposition*}[Restatement of \cref{#1}]
    \MYload{#1}
  \end{proposition*}
  \let\label\MYoldlabel
}
\newcommand{\restatefact}[1]{%
  \let\MYoldlabel\label
  \let\label\MYdummylabel
  \begin{fact*}[Restatement of \prettyref{#1}]
    \MYload{#1}
  \end{fact*}
  \let\label\MYoldlabel
}
\newcommand{\restate}[1]{%
  \let\MYoldlabel\label
  \let\label\MYdummylabel
  \MYload{#1}
  \let\label\MYoldlabel
}

\newcommand{\e}{\epsilon}
\newcommand{\eps}{\epsilon}

\allowdisplaybreaks
\newcommand*{\Sym}{\mathrm{sym}}
\newcommand*{\Pisym}{\Pi_{\Sym}}
\newcommand*{\Id}{\mathrm{Id}}
\newcommand*{\ot}{\otimes}

\newcommand*{\T}{\mathbf{T}}

\newcommand*{\loweredwidetilde}[1]{\mathpalette{\loweredwidetildehelper}{#1}}
\newcommand*{\loweredwidetildehelper}[2]{\hbox{\csname dimen@\endcsname\accentfontxheight#1%
  \accentfontxheight#11.25\csname dimen@\endcsname
  $\csname m@th\endcsname#1\widetilde{#2}$%
  \accentfontxheight#1\csname dimen@\endcsname
  }%
}
\newcommand*{\accentfontxheight}[1]{\fontdimen5\ifx#1\displaystyle \textfont \else\ifx#1\textstyle \textfont \else\ifx#1\scriptstyle \scriptfont \else \scriptscriptfont \fi\fi\fi3
}
\newcommand*{\Qish}{\loweredwidetilde{Q}}
\newcommand*{\Tish}{\loweredwidetilde{T}}
\newcommand*{\TTish}{\loweredwidetilde{\mathbf{T}}}
\newcommand*{\Sish}{\loweredwidetilde{S}}
\newcommand*{\Wish}{\loweredwidetilde{W}}
\newcommand*{\Piish}{\loweredwidetilde{\Pi}}

\newcommand*{\smin}{\sigma_{\min}}

\DeclareMathOperator{\diag}{diag}

\DeclareMathOperator{\img}{img}
\DeclareMathOperator{\Span}{Span}
\DeclareMathOperator{\topn}{top}
\newcommand*{\transpose}[1]{{#1}{}^{\mkern-1.5mu\mathsf{T}}}
\newcommand*{\dyad}[1]{#1#1{}^{\mkern-1.5mu\mathsf{T}}}

\title{
  A Robust Spectral Algorithm for Overcomplete Tensor Decomposition
}

\author{
    Samuel B. Hopkins\thanks{Cornell University. \texttt{samhop@cs.cornell.edu}. Supported by NSF awards 1350196 \& 1408673, and a Microsoft PhD Fellowship.}
      \and
    Tselil Schramm\thanks{MIT and Harvard. \texttt{tselil@mit.edu}. Supported by NSF awards CCF 1565264 \& CNS 1618026.}
      \and
    Jonathan Shi\thanks{Cornell University. \texttt{jshi@cs.cornell.edu}. Supported by NSF awards 1350196 \& 1408673.}
}

\begin{document}



\maketitle
\thispagestyle{empty} 


\begin{abstract}

    We give a spectral algorithm for decomposing overcomplete order-4 tensors, so long as their components satisfy an algebraic non-degeneracy condition that holds for nearly all (all but an algebraic set of measure $0$) tensors over $(\R^d)^{\otimes 4}$ with rank $n \le d^2$. Our algorithm is robust to adversarial perturbations of bounded spectral norm.

Our algorithm is inspired by one which uses the sum-of-squares semidefinite programming hierarchy (Ma, Shi, and Steurer STOC'16), and we achieve comparable robustness and overcompleteness guarantees under similar algebraic assumptions. However, our algorithm avoids semidefinite programming and may be implemented as a series of basic linear-algebraic operations.
We consequently obtain a much faster running time than semidefinite programming methods: our algorithm runs in time $\tilde O(n^2d^3) \le \tilde O(d^7)$, which is subquadratic in the input size $d^4$ (where we have suppressed factors related to the condition number of the input tensor).

\end{abstract}

\clearpage


  \microtypesetup{protrusion=false}\thispagestyle{empty}
  \begin{spacing}{1.05}
  \tableofcontents{}\thispagestyle{empty}
  \end{spacing}
  \microtypesetup{protrusion=true}
  \addtocontents{toc}{\protect{\vspace{-0.4\baselineskip}}}\addtocontents{toc}{\protect\enlargethispage{0.3\baselineskip}}

\addtocontents{toc}{\protect\thispagestyle{empty}}
\thispagestyle{empty}

\clearpage

\setcounter{page}{1}


\section{Introduction}
Tensors are higher-order analogues of matrices: multidimensional arrays of numbers.
They have broad expressive power: tensors may represent higher-order moments of a probability distribution \cite{DBLP:journals/jmlr/AnandkumarGHKT14}, they are natural representations of cubic, quartic, and higher-degree polynomials \cite{DBLP:conf/nips/RichardM14,DBLP:conf/colt/HopkinsSS15}, and they appear whenever data is multimodal (e.g. in medical studies, where many factors are measured) \cite{medicaltensor, neurotensor,chemtensor}.
Due to these reasons, in recent decades tensors have emerged as fundamental structures in machine learning and signal processing.

The notion of \emph{rank} extends from matrices to tensors: a rank-$1$ tensor in $(\R^d)^{\ot k}$ is a tensor that can be written as a tensor product $u^{(1)}  \ot \cdots \ot u^{(k)}$ of vectors $u^{(1)},\ldots,u^{(k)} \in \R^d$.
Any tensor $\T \in (\R^d)^{\otimes k}$ can be expressed as a sum of rank-$1$ tensors, and the {\em rank} of $\T$ is the minimum number of terms needed in such a sum.
As is the case for matrices, we are often interested in tensors of low rank: low-rank structure in tensors often carries interpretable meaning about underlying data sets or probability distributions, and the tensors that arise in many applications are low-rank \cite{DBLP:journals/jmlr/AnandkumarGHKT14}.

Tensor decomposition is the natural inverse problem in the context of tensor rank:
given a $d$-dimensional symmetric $k$-tensor $\T \in (\R^d)^{\tensor k}$ of the form
\[
  \T = \sum_{i \leq n} a_i^{\tensor k} + \mathbf{E},
\]
for vectors $a_1,\ldots,a_n \in \R^d$ and an (optional) error tensor $\mathbf{E} \in (\R^d)^{\tensor k}$, we are asked to output vectors $b_1,\ldots,b_n$ as close as possible to $a_1,\ldots,a_n$ (e.g. minimizing the Euclidean distance $\|b_i - a_i\|$).
The goal is to accomplish this with an algorithm that is as efficient as possible, under the mildest-possible assumptions on $k$,$a_1,\ldots,a_n$, and $\mathbf{E}$.

While tensor rank decomposition is a generalization of rank decomposition for matrices, decomposition for tensors of order $k \geq 3$ differs from the matrix case in several key ways.
\begin{enumerate}
\item (Uniqueness) Under mild assumptions on the vectors $a_1,\ldots,a_n$, tensor decompositions are unique (up to permutations of $[n]$), while matrix decompositions are often unique only up to unitary transformation.

\item (Overcompleteness) Tensor decompositions often remain unique even when the number of factors $n$ is larger than the ambient dimension $d$ (up to $n = O(d^{k-1})$), while a $d \times d$ matrix can have only $d$ eigenvectors or $2d$ singular vectors.
\end{enumerate}
These features make tensor decompositions suitable for many applications where matrix factorizations are insufficient.
However, there is another major difference:
\begin{enumerate}
  \setcounter{enumi}{2}
  \item (Computational Intractability) While many matrix decompositions --- eigendecompositions, singular value decompositions, $LU$-factorizations, and so on --- can be found in polynomial time, tensor decomposition is NP-hard in general \cite{DBLP:journals/jacm/HillarL13}.
\end{enumerate}

In spite of the NP-hardness of general tensor decomposition, many special cases turn out to admit polynomial-time algorithms.
A classical algorithm, often called \emph{Jennrich's algorithm}, recovers the components $a_1,\ldots,a_n$ from $\T$ when they are linearly independent (which requires $n \leq d$) and $\mathbf{E} = 0$ using simultaneous diagonalization \cite{harshman1970foundations, de1996blind}.

More sophisticated algorithms improve on Jennrich's in their tolerance to overcompleteness (and the resulting lack of linear independence) and robustness to nontrivial error tensors $\mathbf{E}$.
The literature now contains a wide variety of techniques for tensor decomposition: the major players are iterative methods (tensor power iteration, stochastic gradient descent, and alternating minimization), spectral algorithms, and convex programs.
Convex programs, and in particular the \emph{sum-of-squares} semidefinite programming hierarchy (SoS), require the mildest assumptions on $k,a_1,\ldots,a_n,\mathbf{E}$ among known polynomial-time algorithms \cite{DBLP:conf/focs/MaSS16}.
In pushing the boundaries of what is known to be achievable in polynomial time, SoS-based algorithms have been crucial.
However, the running times of these algorithms are large polynomials in the input, making them utterly impractical for applications.

The main contribution of this work is a tensor decomposition algorithm whose robustness to errors and tolerance for overcompleteness are similar to those of the SoS-based algorithms, but with \emph{subquadratic running time}.
Other algorithms with comparable running times either require higher-order tensors,\footnote{Higher-order tensors are costly because they are larger objects, and for learning applications they often require a polynomial increase in sample complexity.} are not robust in that they require the error $\mathbf{E} = 0$ or $\mathbf{E}$ exponentially small, or require linear independence of the components $a_1,\ldots,a_n$ and hence $n \leq d$.\footnote{There are also existing robust algorithms which tolerate some overcompleteness when $a_1,\ldots,a_n$ are assumed to be random; in this paper we study \emph{generic} $a_1,\ldots,a_n$, which is a much more challenging setting than random $a_1,\ldots,a_n$ \cite{DBLP:journals/corr/AnandkumarGJ14b, DBLP:conf/stoc/HopkinsSSS16}.}

Our algorithm is comparatively simple, and can be implemented with a small number of dense matrix and matrix-vector multiplication operations, which are fast not only asymptotically but also in practice.

Concretely, we study tensor decomposition of overcomplete $4$-tensors under algebraic nondegeneracy conditions on the tensor components $a_1,\ldots,a_n$.
Algebraic conditions like ours are the mildest type of assumption on $a_1,\ldots,a_n$ known to lead to polynomial time algorithms -- our algorithm can decompose all but a measure-zero set of $4$-tensors of rank $n \ll d^2$, and in particular we make no assumption that the components $a_1,\ldots,a_n$ are random.\footnote{Although decompositions of $4$-th order tensors can remain unique up to $n \approx d^3$, no polynomial-time algorithms are known which successfully decompose tensors of overcompleteness $n \gg d^2$.}

When $n \ll  d^2$, our algorithm approximately recovers a $1-o(1)$ fraction of $a_1,\ldots,a_n$ (up to their signs) from $T = \sum_{i \leq n} a_i^{\tensor 4} + E$, so long as the spectral norm $\|E \|$ is (significantly) less than the minimum singular value of a certain matrix associated to the $\{a_i\}$.
(In particular, nonsingularity of this matrix is our nondegeneracy condition on $a_1,\ldots,a_n$.)
The algorithm requires time $\tilde O(n^2 d^3) \leq O(d^7)$, which is subquadratic in the input size $d^4$.

\paragraph{Robustness, Overcompleteness, and Applications to Machine Learning}
Tensor decomposition is a common primitive in algorithms for statistical inference that leverage the \emph{method of moments} to learn parameters of latent variable models.
Examples of such algorithms exist for independent component analysis / blind source separation \cite{de2007fourth}, dictionary learning \cite{DBLP:conf/stoc/BarakKS15, DBLP:conf/focs/MaSS16, DBLP:conf/colt/SchrammS17}, overlapping community detection \cite{DBLP:conf/colt/AnandkumarGHK13, hopkins2017efficient}, mixtures of Gaussians \cite{MR3388256-Ge15}, and more.

In these applications, we receive samples $x \in \R^d$ from a model distribution $\cD(\rho)$ that is a function of parameters $\rho$.
The goal is to estimate $\rho$ using the samples.
The method-of-moments strategy is to construct the third- or fourth-order moment tensor $\E x^{\ot k}$ ($k=3,4$) from samples whose expectation $\E x^{\tensor k} = \sum_{i \leq n} a_i^{\tensor k}$ is a low rank tensor with components $a_1,\ldots,a_n$, from which the model parameters $\rho$ can be deduced.\footnote{Any constant $k$, rather than just $k=3,4$, may lead to polynomial-time learning algorithms, but the cost is typically gigantic polynomial sample complexity and running time, scaling like $d^k$, to estimate and store a $k$-th order tensor.}
Since $\E x^{\tensor k}$ is estimated using samples, the tensor decomposition algorithm used to extract $a_1,\ldots,a_n$ from $\E x^{\tensor k}$ must be {\em robust} to error from sampling.

The sample complexity of the resulting algorithm depends directly on the magnitude of errors tolerated by the decomposition algorithm.
In addition, the greater general error-robustness of our result suggests better tolerance of model misspecification error.

Some model classes give rise to {\em overcomplete} tensors; roughly speaking, this occurs when the number of parameters (the size of the description of $\rho$) far exceeds $d^2$, where $d$ is the ambient dimension.
Typically, in such cases, $\rho$ consists of a collection of vectors $a_1,\ldots,a_n \in \R^d$ with $n \gg d$.
Such overcomplete models are widely used; for example, in the {\em dictionary learning} setting, we are given a data set $S$ and are asked to find a \emph{sparse representation} of $S$.
This is a powerful preprocessing tool, and the resulting representations are more robust to perturbations, but assembling a truly sparse, effective dictionary often requires representing $d$-dimensional data in a basis with $n \gg d$ elements \cite{Lewicki,Elad10}.
Recent works also relate the problem of learning neural networks with good generalization error to tensor decomposition, showing a connection between overcompleteness and the width of the network \cite{DBLP:journals/corr/abs-1802-07301}.\footnote{Strictly speaking, this work shows a reduction {\em from} tensor decomposition {\em to} learning neural nets, but the connection between width and overcompleteness is direct regardless.}

Using tensor decomposition in such settings requires algorithms with practical running times, error robustness, and tolerance to overcompleteness.
The strongest polynomial-time guarantees for overcomplete dictionary learning and similar models currently rely on overcomplete tensor decomposition via the SoS method \cite{DBLP:conf/focs/MaSS16}; our work is an important step towards giving lightweight, spectral algorithms for such problems.

\subsection{Our Results}

Our contribution is a robust, lightweight spectral algorithm for tensor decomposition in the overcomplete regime.
We require that the components satisfy an algebraic non-degeneracy assumption satisfied by all but a measure-$0$ set of inputs.
At a high level, we require that a certain matrix associated with the components of the tensor have full rank.
Though the assumption may at first seem complicated, we give it formally here:
\begin{definition}
\label[definition]{def:intro-Pi23perp}
Let $\Pi_{2,3}^{\perp}$ be the projector to the orthogonal complement of the subspace of $(\R^d)^{\ot 3}$ that is symmetric in its latter two tensor modes.
Equivalently, $\Pi_{2,3}^{\perp} = \tfrac{1}{2}(\Id - P_{2,3})$, where $P_{2,3}$ is the linear operator that interchanges the second and third modes of $(\R^d)^{\ot 3}$.
\end{definition}
\begin{definition}
\label[definition]{def:intro-whitening}
Let $\Pi_{\img(M)}$ denote the projector to the column space of the matrix $M$.
Equivalently, $\Pi_{\img(M)} = (MM^{\top})^{-1/2}M = M(M^{\top}M)^{-1/2}$, where $(MM^{\top})^{-1/2}$ is the \emph{whitening} transform of $M$ and is equal to the Moore-Penrose pseudoinverse of $(MM^{\top})^{1/2}$.
\end{definition}
\begin{definition}
\label[definition]{def:kappa}
   Vectors $a_1,\ldots,a_n \in \R^d$ are {\em $\kappa$-non-degenerate} if the matrix $K(a_1,\ldots,a_n)$, defined below, has minimum singular value at least $\kappa > 0$.
    If $\kappa = 0$, we say that the $\{a_i\}$ are {\em degenerate}.

    The matrix $K(a_1,\ldots,a_n)$ is given by choosing for each $i\in[n]$ a matrix $B_i$ whose columns form a basis for the orthogonal complement of $a_i$ in $\R^d$, assembling the $d^3 \times n(d-1)$ matrix $H$ whose rows are given by $a_i\ot a_i \ot B_i^{(j)}$ as
    \[
	H = \left[\begin{array}{c} a_1^\top \ot a_1^\top \ot B_1^\top \\ \vdots \\ a_n^\top \ot a_n^\top \ot B_n^\top \end{array}\right]
 \]
    and letting 
	$K(a_1,\ldots,a_n) = \Pi_{2,3}^{\perp}\Pi_{\img(H^\top)}$.
\end{definition}

We note that when $n \ll d^2$ then all but a measure-zero set of unit $(a_1,\ldots,a_n) \in \R^{dn}$ satisfy the condition that $\kappa > 0$.
We expect also that for $n \ll d^2$, if $a_1,\ldots,a_n \in \R^d$ are independent random unit vectors then $\kappa \geq \Omega(1)$ -- we provide simulations in support of this in \cref{sec:sims}.\footnote{Furthermore, standard techniques in random matrix theory prove that when $a_1,\ldots,a_n$ are random then matrices closely related to $K(a_1,\ldots,a_n)$ are well-conditioned; for instance this holds (roughly speaking) if $(H_1^\top H)^{-1/2}$ and $(H_2^\top H)^{-1/2}$ are removed. However, inverses and pseudoinverses of random matrices, especially those with dependent entries like ours, are infamously challenging to analyze -- we leave this challenge to future work. See \cref{sec:sims} for details.}

Some previous works on tensor decomposition under algebraic nondegeneracy assumptions also give smoothed analyses of nondegeneracy, showing that small random perturbations of arbitrary vectors are $\tfrac 1 {\poly(d)}$-well-conditioned (for differing notions of well-conditioned-ness) \cite{DBLP:conf/stoc/BhaskaraCMV14, DBLP:conf/focs/MaSS16}.
We expect that a similar smoothed analysis is possible for $\kappa$-non-degeneracy, though because of the specific form of the matrix $K(a_1,\ldots,a_n)$ it does not follow immediately from known results.
We defer this technical challenge to future work.

Given this non-degeneracy condition, we robustly decompose the input tensor in time $\tilde{O}(\frac{n^{2}d^3}{\kappa})$, where we have suppressed factors depending on the smallest singular value of a matrix flattening of our tensor.
\begin{theorem*}[Special case of \cref{thm:main}]
   Suppose that $d \le n \le d^2$, and that $a_1,\ldots,a_n \in \R^d$ are $\kappa$-non-degenerate unit vectors for $\kappa > 0$, and suppose that $\T$ is their $4$-tensor perturbed by noise,
    $\T \in (\R^d)^{\ot 4}$ such that $\T = \sum_{i \in [n]} a_i^{\ot 4} + E$, where $E$ is a perturbation such that $\|E\| \le \frac{\epsilon}{\log d}$ in its $d^2 \times d^2$ reshaping.
    Suppose further that when reshaped to a $d^2\times d^2$ matrix, $\|T^{-1}\| \le O(1)$ and that $\|\sum_{i \in [n]} (a_i^{\ot 3})(a_i^{\ot 3})^\top\| \le O(1)$.

     There exists an algorithm \textsc{decompose} with running time $\tilde O(n^2 d^3 \kappa^{-1})$, so that for every such $\T$ there exists a subset $S \subseteq \{a_1,\ldots,a_n\}$ of size $|S| \geq 0.99n$, such that $\textsc{decompose}(\T)$ with high probability returns a set of $t = \tilde O(n)$ unit vectors $b_1,\ldots,b_t$ where every $a_i \in S$ is close to some $b_j$, and each $b_j$ is close to some $a_i \in S$:
    \[
	\forall a_i \in S, \,\, \max_{j} |\iprod{b_j,a_i}| \ge 1 - O\left(\frac{\epsilon}{\kappa^2}\right)^{1/8},\qquad \text{and} \qquad
	\forall j \in [t], \,\, \max_{a_i \in S} |\iprod{b_j,a_i}| \ge 1-O\left(\frac{\epsilon}{\kappa^2}\right)^{1/8}.
    \]

    Furthermore, if $a_1,\ldots,a_n$ are random unit vectors, then with high probability they satisfy the conditions of this theorem with $\kappa = \Omega(1)$.
\end{theorem*}
When $n \le d$, our algorithm still obtains nontrivial guarantees (though the runtime asymptotics are dominated by other terms); however in this regime, a combination of the simpler algorithm of \cite{DBLP:conf/colt/SchrammS17} and a whitening procedure gives comparable guarantees.

We remark that our full theorem,
\cref{thm:main}, does not pose as many restrictions on the $\{a_i\}$; we do not generally require that $\|T^{-1}\| \le O(1)$ or that $\|\sum_i (a_i^{\ot 3})(a_i^{\ot 3})^\top\| \le O(1)$.
However, allowing these quantities to depend on $d$ and $n$ affects our runtime and approximation guarantees, and so to simplify presentation we have made these restrictions here; we refer the reader to \cref{thm:main} for details.

Furthermore, in the theorem stated above we recover only a $0.99$-fraction of the vectors, and we require the perturbation to have magnitude $O(\frac{1}{\log d})$.
This is again a particular choice of parameters in \cref{thm:main}, which allows for a four-way tradeoff among accuracy, magnitude of perturbation, fraction of components recovered, and runtime.
For example, if the perturbation is $\frac{1}{\poly(d)}$ in spectral norm, then we may recover all components in time $\tilde O(n^2d^3\kappa^{-1})$; alternatively, if the perturbation has spectral norm $\eta^2 = \Theta(1)$, then we may recover an $0.99$-fraction of components in time $\tilde O(n^{2+O(\eta)} d^3 \kappa^{-1})$ up to accuracy $1-O(\frac{\eta}{\kappa^2})^{1/8}$.
Again, we refer the reader to \cref{thm:main} for the full tradeoff.

Finally, a note about our recovery guarantee:
we guarantee that every vector returned by the algorithm is close to {\em some} component, and furthermore that most components will be close to some vector.
It is possible to run a clean-up procedure after our algorithm, in which nearby approximate components $b_j$ are clustered to correspond to a specific $a_i$; depending on the proximity of the $a_i$ to each other, this may require stronger accuracy guarantees, and so we leave this procedure as an independent step.
Our guarantee does not include signs, but this is because the tensor $\T$ is an even-order tensor, so the decomposition is only unique up to signings as $(-a_i)^{\otimes 4} = a_i^{\ot 4}$.

\subsection{Related works}

The literature on tensor decomposition is broad and varied, and we will not attempt to survey it fully here (see e.g. the survey \cite{DBLP:journals/siamrev/KoldaB09} or the references within \cite{DBLP:journals/jmlr/AnandkumarGHKT14,ge2017optimization} for a fuller picture).
We will give an idea of the relationship between our algorithm and others with provable guarantees.

For simplicity we focus on order-$4$ tensors.
Algorithms with provable guarantees for tensor decomposition fall broadly into three classes: iterative methods, convex programs, and spectral algorithms.
For a brief comparison of our algorithm to previous works, we include \cref{fig:algs}.

\begin{table}
    \centering
\begin{tabular}{|c|c|c|c|c|c|}
    \hline
    {\bf Algorithm} & {\bf Type} & \bf{Rank} & {\bf Robustness} & {\bf Assumptions} & {\bf Runtime}\\
    \hline
    \cite{DBLP:journals/tsp/LathauwerCC07} & algebraic & $n \le d^2$ & $\|E\|_{\infty} \le 2^{-O(d)}$ & algebraic & $\tilde O(n^3d^4)$ \\
    \hline
    \cite{DBLP:journals/jmlr/AnandkumarGJ17} & iterative & $n \le o(d^{1.5})$ & $\|E\| \le o(\frac{n}{d^2})$ & random, warm start  & $\tilde O(nd^3)$ \\
    \hline

    \cite{ge2017optimization} & iterative & $n \le O(d^{2})$ & $E = 0$ & random, warm start  & $\tilde O(nd^4)$\\
    \hline
    \cite{DBLP:conf/focs/MaSS16} & SDP & $n \le d^2$ & $\|E\| \le 0.01$ & algebraic & $\ge nd^{24}$\\
    \hline
    \cite{DBLP:conf/colt/SchrammS17} & spectral & $n \le d$ & $\|E\| \le O(\frac{\log\log d}{\log d})$ & orthogonal & $\tilde O(d^{2+\omega})$\\
    \hline
    this paper & spectral & $n \le d^2$ & $\|E\| \le O(\frac{1}{\log d})$ & algebraic & $\tilde O(n^2d^{3})$\\
    \hline
\end{tabular}
    \caption{{\small A comparison of tensor decomposition algorithms for rank-$n$ $4$-tensors in $(\R^d)^{\otimes 4}$. Here $\omega$ denotes the matrix multiplication constant. A robustness bound $\|E\| \le \eta$ refers to the requirement that a $d^2 \times d^2$ reshaping of the error tensor $E$ have spectral norm at most $\eta$. Some of the algorithms' guarantees involve a tradeoff between robustness, runtime, and assumptions; where this is the case, we have chosen one representative setting of parameters. See \cref{sec:table-notes} for details. Above, ``random'' indicates that the algorithm assumes $a_1,\ldots,a_n$ are independent unit vectors (or Gaussians) and ``algebraic'' indicates that the algorithm assumes that the vectors avoid an algebraic set of measure $0$.}}
    \label[table]{fig:algs}
\end{table}

\paragraph{Iterative Methods.}
Iterative methods are a class of algorithms that maintain one (or sometimes several) estimated component(s) $b$, and update the estimate using a variety of update rules.
Some popular update rules include tensor power iteration \cite{DBLP:journals/jmlr/AnandkumarGHKT14}, gradient descent \cite{ge2017optimization}, and alternating-minimization \cite{DBLP:journals/corr/AnandkumarGJ14}.
Most of these methods have the advantage that they are fast; the update steps usually run in time linear in the input size, and the number of updates to convergence is often polylogarithmic in the input size.

The performance of the most popular iterative methods has been well-characterized in some restricted settings; for example, when the components $\{a_i\}$ are orthogonal or linearly independent \cite{DBLP:journals/jmlr/AnandkumarGHKT14,DBLP:conf/colt/GeHJY15, DBLP:conf/icml/SharanV17}, or are independently drawn random vectors \cite{DBLP:journals/jmlr/AnandkumarGJ17,ge2017optimization}.
Furthermore, many of these analyses require a ``warm start,'' or an initial estimate $b$ that is more correlated with a component than a typical random starting point.
Few provable guarantees are known for the non-random overcomplete regime, or in the presence of arbitrary perturbations.

\paragraph{Convex Programming.}
Convex programs based on the sum-of-squares (SoS) semidefinite programming (SDP) relaxation yield the most general provable guarantees for tensor decomposition.
These works broadly follow a \emph{method of pseudo-moments}: interpreting the input tensor $\sum_{i \in [n]} a_i^{\tensor k}$ as the $k$-th moment tensor $\E X^{\tensor k}$ of a distribution $X$ on $\R^d$, this approach uses SoS to generate \emph{surrogates} (or \emph{pseudo-moments}) for higher moment tensors, like $\E X^{\tensor 100k} = \sum_{i \in [n]} a_i^{\tensor 100k}$.
It is generally easier to extract the components $a_1,\ldots,a_n$ from $\sum_{i \in [n]} (a_i^{\tensor 100})^{\tensor k}$ than from $\sum_{i \in [n]} a_i^{\tensor k}$, because the vectors $\{a_i^{\tensor 100}\}$ have fewer algebraic dependencies than the vectors $\{a_i\}$, and are farther apart in Euclidean distance.
Of course, $\E X^{\tensor 100k} = \sum_{i \in [n]} a_i^{\tensor 100k}$ is not given as input, and even in applications where the input is negotiable, it may be expensive or impossible to obtain such a high-order tensor. The SoS method uses semidefinite programming to generate a surrogate which is good enough to be used to find the vectors $a_1,\ldots,a_n$

Work on sum-of-squares relaxations for tensor decomposition began with the quasi-polynomial time algorithm of \cite{DBLP:conf/stoc/BarakKS15}; this algorithm requires only mild well-conditioned-ness assumptions, but also requires high-order tensors as input, and runs in quasi-polynomial time.
This was followed by an analysis showing that, at least in the setting of random $a_1,\ldots,a_n$, the SoS algorithm can decompose substantially overcomplete tensors of order $3$ \cite{DBLP:conf/approx/GeM15}.
This line of work finally concluded with the work of Ma, Shi, and Steurer \cite{DBLP:conf/focs/MaSS16}, who give sum-of-squares based polynomial-time algorithms for tensor decomposition in the most general known settings: under mild algebraic assumptions on the components, and in the presence of adversarial noise, so long as the noise tensor has bounded spectral norm in its matrix reshapings.

These SoS algorithms have the best known polynomial-time guarantees, but they are formidably slow.
The work of \cite{DBLP:conf/focs/MaSS16} uses the degree-$8$ sum-of-squares relaxation, meaning that to find each of the $n$ components, one must solve an SDP in $\Omega (d^8)$ variables.
While these results are important in establishing that polynomial-time algorithms {\em exist} for these settings, their runtimes are far from efficient.

\paragraph{Spectral algorithms from Sum-of-Squares Analyses.}
Inspired by the mild assumptions needed by SoS algorithms, there has been a line of work that uses the {\em analyses} of SoS in order to design more efficient {\em spectral} algorithms, which ideally work for similarly-broad classes of tensors.

At a high level, these spectral algorithms use eigendecompositions of specific matrix polynomials to directly construct approximate primal and dual solutions to the SoS semidefinite programs, thereby obtaining the previously mentioned ``surrogate moments'' without having to solve an SDP.
Since the SoS SDPs are quite powerful, constructing (even approximate) solutions to them directly and efficiently is a nontrivial endeavor.
The resulting matrices are only approximately SDP solutions --- in fact, they are often far from satisfying most of the constraints of the SoS SDPs.
There is a tradeoff between how well these spectrally constructed solutions approximate the SoS output and how efficiently the algorithm can be implemented.
However, by carefully choosing which constraints to satisfy, these works are able to apply the SDP rounding algorithms to the approximate spectrally-constructed solutions (often with new analyses) to obtain similar algorithmic guarantees.

The work of \cite{DBLP:conf/stoc/HopkinsSSS16} was the first to adapt the analysis of SoS for random $a_1,\ldots,a_n$ presented by \cite{DBLP:conf/approx/GeM15} to obtain spectral algorithms for tensor decomposition, giving subquadratic algorithms for decomposing random overcomplete tensors with $n \le O(d^{4/3})$.
As SoS algorithms have developed, so too have their faster spectral counterparts.
In particular, \cite{DBLP:conf/colt/SchrammS17} adapted some of the SoS arguments presented in \cite{DBLP:conf/focs/MaSS16} to give robust subquadratic algorithms for decomposing orthogonal $4$-tensors in the presence of \emph{adversarial} noise bounded only in spectral norm.

Our result builds on the progress of both \cite{DBLP:conf/focs/MaSS16,DBLP:conf/colt/SchrammS17}.
The SoS algorithm of \cite{DBLP:conf/focs/MaSS16} was the first to robustly decompose generic overcomplete tensors in polynomial time.
The spectral algorithm of \cite{DBLP:conf/colt/SchrammS17} obtains a much faster running time for robust tensor decomposition, but sacrifices overcompleteness.
Our work adapts (and improves upon) the SoS analysis of \cite{DBLP:conf/focs/MaSS16} to give a spectral algorithm for the robust {\em and} overcomplete regime.
Our primary technical contribution is the efficient implementation of the {\em lifting} step in the SoS analysis of \cite{DBLP:conf/focs/MaSS16} as an efficient spectral algorithm to generate surrogate $6$th order moments; this is the subject of \cref{sec:lifting}, and we give an informal description in \cref{sec:overview}.

\paragraph{FOOBI}
The innovative FOOBI (Fourth-Order Cumulant-Based Blind Identification) algorithm of \cite{DBLP:journals/tsp/LathauwerCC07} was the first method with provable guarantees for overcomplete $4$-th order tensor decomposition under algebraic nondegeneracy assumptions.
Like our algorithm, FOOBI can be seen as a \emph{lifting} procedure (to an $8$-th order tensor) followed by a \emph{rounding} procedure. 
The FOOBI lifting procedure inspires ours -- although ours runs faster because we lift to a $6$-tensor rather than an $8$-tensor -- but the FOOBI rounding step is quite different, and proceeds via a clever simultaneous diagonalization approach.
The advantage our algorithm offers over FOOBI is twofold: first, it provides formal, strong robustness guarantees, and second, it has a faster asymptotic runtime.

To the first point:
for a litmus test, consider the case that $n=d$ and $a_1,\ldots,a_n \in \R^d$ are orthonormal.
On input $T = \sum_{i=1}^n a_i^{\tensor 4} + E$, our algorithm recovers the $a_i$ for arbitrary perturbations $E$ so long as they are bounded in spectral norm by $\|E\| \le 1/\poly\log d$.\footnote{In contrast, most iterative methods, such as power iteration, can only handle perturbations of spectral norm at most $\|E\| \le 1/\poly(d)$.}
We are not aware of any formal analyses of FOOBI when run on tensors with arbitrary perturbations of this form.
Precisely what degree of robustness should be expected from this modified FOOBI algorithm is unclear.
The authors of \cite{DBLP:journals/tsp/LathauwerCC07} do suggest (without analysis) a modification of their algorithm for the setting of nonzero error tensors $E$, involving an alternating-minimization method for computing an {\em approximate} simultaneous diagonalization. 
Because the problem of approximate simultaneous diagonalization is non-convex, establishing robustness guarantees for the FOOBI algorithm when augmented with the approximate simultaneous diagonalization step appears to be a nontrivial technical endeavor.
We think this is an interesting and potentially challenging open question.

Further, while the running time of FOOBI depends on the specific implementation of its linear-algebraic operations, we are unaware of any technique to implement it in time faster than $\tilde O(n^3 d^4)$.
In particular, the factor of $d^4$ appears essential to any implementation of FOOBI; it represents the side-length of a $d^4 \times d^4$ square unfolding of a $d$-dimensional $8$-tensor, which FOOBI employs extensively.
By contrast, our algorithm runs in time $\tilde O(n^2 d^3)$, which is (up to logarithmic factors) faster by a factor of $nd$.

    \section{Overview of algorithm}\label[section]{sec:overview}
\newcommand{\mathsc}[1]{\text{\textsc{#1}}}

We begin by describing a simple tensor decomposition algorithm for orthogonal $3$-tensors: Gaussian rounding (Jennrich's algorithm \cite{harshman1970foundations}).
We then build on that intuition to describe our algorithm.

\paragraph{Orthogonal, undercomplete tensors.}
Suppose that $u_1,\ldots,u_d \in \R^d$ are orthonormal vectors, and that we are given $T = \sum_{i\in[d]} u_i^{\otimes 3}$.
As a first attempt at recovering the $u_i$, one might be tempted to choose the first ``slice'' of $T$, the $d\times d$ matrix $T_1 = \sum_{i} u_i(1) \cdot u_i u_i^\top$, and compute its singular value decomposition (SVD).
However, if $|u_i(1)| = |u_j(1)|$ for some $i\neq j \in [d]$, the SVD will not allow us to recover these components.
In this setting, Gaussian rounding allows us to exploit the additional mode of $T$:
If we sample $g \sim \cN(0,\Id_d)$, then we can take the random flattening $T(g) = \sum_{i} \iprod{g,u_i} \cdot u_i u_i^\top$; because the $\iprod{g,u_i}$ are independent standard Gaussians, they are distinct with probability $1$, and an SVD will recover the $u_i$ exactly.
Moreover, this algorithm also solves $k$-tensor decomposition for orthogonal tensors with $k \geq 4$, by treating $\sum_{i \in [d]} u_i^{\tensor k}$ as the $3$-tensor $\sum_{i \in [d]} u_i^{\tensor k-1} \tensor u_i \tensor u_i$.

\paragraph{Challenges of overcomplete tensors.}
In our setting, we have unit vectors $\{a_i\}_{i\in[n]}\subset \R^d$ with $n > d$, and $T = \sum_{i} a_i^{\otimes 4}$ (focusing for now on the unperturbed case).
Since $n > d$, the components $a_1,\ldots,a_n$ are {\em not} orthogonal: they are not even linearly independent.
So, we cannot hope to use Gaussian rounding as a black box.
While the vectors $a_1\ot a_1,\ldots,a_n\ot a_n$ may be linearly independent, the spectral decompositions of the matrix $\sum_{i \in [n]} (a_i^{\tensor 2})(a_i^{\tensor 2})^\top$ are not necessarily useful, since its eigenvectors may not be close to any of the vectors $a_i$, and may be unique only up to rotation.

\paragraph{Challenges of perturbations.}
Returning momentarily to the orthogonal setting with $n \leq d$, new challenges arise when the perturbation tensor $E$ is nonzero.
For an orthogonal $4$-tensor $T = \sum_{i \in [d]} u_i^{\tensor 4} + E$, the Gaussian rounding algorithm produces the matrix $\sum_{i \in [d]} \iprod{g, u_i^{\tensor 2}} \dyad{u_i} + E_g$ for some $d \times d$ matrix $E_g$.
The difficulty is that even if the spectral norm $\|E\| \ll \sigma_{\min} (\sum_{i \in [d]} \dyad{(u_i^{\tensor 2})}) = 1$, the matrix $E_g$ sums many slices of the tensor $E$, and so the spectrum of $E_g$ can overwhelm that of $\sum_{i \in [d]} \iprod{g,u_i^{\tensor 2}} \dyad{u_i}$.

This difficulty is studied in \cite{DBLP:conf/colt/SchrammS17}, where it is resolved by SoS-inspired preprocessing of the tensor $T$.
We borrow many of those ideas in this work.

\paragraph{Algorithmic strategy.}
We now give an overview of our algorithm. \Cref{alg:sketch} gives a summarized version of the algorithm, with details concerning robustness and fast implementation omitted.

There are two main stages to the algorithm:
the first stage is \emph{lifting}, where the input rank-$n$ $4$-tensor over $\R^d$ is lifted to a corresponding rank-$n$ $3$-tensor over a higher dimensional space $\R^{d^2}$; this creates an opportunity to use Gaussian rounding on the newly-created tensor modes.
In the second \emph{rounding} stage, the components of the lifted tensor are recovered using a strategy similar to Gaussian rounding and then used to find the components of the input.

This parallels the form of the SoS-based overcomplete tensor decomposition algorithm of \cite{DBLP:conf/focs/MaSS16}, where both stages rely on SoS semidefinite programming.
Our main technical contribution is a spectral implementation of the lifting stage; our spectral implementation of the rounding stage reuses many ideas of \cite{DBLP:conf/colt/SchrammS17}, adapted round the output of our new lifting stage.

\paragraph{Lifting.}
The goal of the lifting stage is to transform the input $T = \sum_{i \in [n]} \dyad{(a_i^{\tensor 2}) }$ to an orthogonal $3$-tensor.
Let $W = T^{-1/2}$ and observe that the \emph{whitened} vectors $W(a_i^{\tensor 2})$ are orthonormal; therefore we will want to use $T$ to find the orthogonal 3-tensor $\sum_{i \in [n]} (W a_i^{\tensor 2})^{\tensor 3}$.

The lifting works by deriving $\Span(a_i^{\ot 3})_{i \in [n]}$ from $\Span(a_i^{\ot 2})_{i \in [n]}$, where the latter is simply the column space of the input $T$.
By transforming $\Span(a_i^{\ot 3})$ using $W = T^{-1/2}$, we obtain $\Span(W(a_i^{\ot 2}) \ot a_i^{\vphantom 1})$.
Since $\{W(a_i^{\tensor 2}) \tensor a_i\}_{i \in [n]}$ are orthonormal, the orthogonal projector to their span is in fact equal to $\sum_i \dyad{(W(a_i^{\ot 2}) \ot a_i^{\vphantom 1})}$, which is only a reshaping and a final multiplication by $W$ away from the orthogonal tensor $\sum_i (W(a_i^{\ot 2}))^{\ot 3}$.

The key step is the operation which obtains $\Span(a_i^{\ot 3})$ from $\Span(a_i^{\ot 2})$.
It rests on an algebraic ``identifiability'' argument, which establishes that for almost all problem instances (all but an algebraic set of measure $0$), the subspace $\Span(a_i^{\ot 3})$ is equal to $\Span(a_i^{\ot 2})\otimes \R^d$ intersected with the symmetric subspace $\Span(\{x \ot x \ot x\}_{x \in \R^{d}})$.
Since we can compute $\Span(a_i \ot a_i)$ from the input and since the symmetric subspace is easy to describe, we are able to perform this lifting step efficiently.
The simplest version of the identifiability argument is given in \Cref{lem:sketch}, and a more robust version that includes a condition number analysis is given in \Cref{sec:identifiability}.

\begin{lemma}[Simple Identifiability]
\label[lemma]{lem:sketch}
  Let $a_1,\ldots,a_n \in \R^d$ with $n \leq d^2$.
  Let $S$ denote $\Span(\{a_i^{\tensor 2}\})$ and let $T$ denote $\Span(\{a_i^{\tensor 3}\})$ and assume both have dimension $n$.
  Let $\Sym \subseteq (\R^{d})^{\tensor 3}$ be the linear subspace
  $\Sym = \Span (\{x \tensor x \tensor x\}_{x \in \R^d})\mper$
  For each $i$, let $\{b_{i,j}\}_{j \in [d-1]}$ be an arbitrary orthonormal basis the orthogonal complement of $a_i$ in $\R^d$.
  Let also
  \[ \transpose{K'} := \left[\begin{array}{l @{} l @{} c @{\;} c @{\;} l @{} c @{} l} a_1 &{}\ot a_1 \ot{}& b_{1,1} &{}-{} &a_1 \ot{}& b_{1,1} &{}\ot a_1 \\ &&&\vdots&&& \\ a_i &{}\ot a_i \,\ot{}& b_{i,j} &{}-{} &a_i \,\ot{}& b_{i,j} &{}\ot a_i \\ &&&\vdots&&& \\ a_n &{}\ot a_n \ot{}& b_{n,d-1} &{}-{} &a_n \ot{}& b_{n,d-1} &{}\ot a_n \end{array}\right]\,, \]
  Then if $K'$ has full rank $n(d-1)$, it follows that $(S \ot \R^d) \cap \Sym = T$.
\end{lemma}
\begin{proof}
To show that $T \subseteq (S \ot \R^d) \cap \Sym$, we simply note that $\{a_i \ot a_i \ot a_i\}_{i \in [n]}$ form a basis for $T$ and are also each in both $S \ot \R^d$ and $\Sym$.

To show that $(S \ot \R^d) \cap \Sym \subseteq T$, we take some $y \in (S \ot \R^d) \cap \Sym$.
Since $y$ is symmetric under mode interchange, we express $y$ in two ways as
\[ y = \sum_i a_i \ot a_i \ot c_i = \sum_i a_i \ot c_i \ot a_i \,. \]
Then by subtracting these two expressions for $y$ from each other, we find
\[ 0 = \sum_i a_i \ot (a_i \ot c_i - c_i \ot a_i). \]
We express $c_i = \iprod{a_i,c_i}a_i + \sum_{j} \gamma_{ij}b_{ij}$ for some vector $\gamma$.
Then the symmetric parts cancel out, leaving
\[ 0 = \sum_{ij} \gamma_{ij} \, a_i \ot (a_i \ot b_{ij} - b_{ij} \ot a_i) = K' \gamma\,. \]
Since $K'$ is full rank by assumption, this is only possible when $\gamma = 0$.
Therefore, $c_i \propto a_i$ for all $i$, so that $y \in T$.
\end{proof}
\begin{remark}
Although the condition number from the matrix $K'$ here is not the same as the one derived from $K$ from \Cref{def:kappa}, it is off by at most a multiplicative factor of $2\|T^{-1}\|^{-1/2}$.
To see this, $K$ in \Cref{def:kappa} is given as $K = \Pi_{2,3}^{\perp}\Pi_{\img(\transpose{H})}$, whereas we may write $K' = 2\Pi_{2,3}^{\perp}\transpose{H} = 2\Pi_{2,3}^{\perp}\Pi_{\img(\transpose{H})}(\dyad{H})^{1/2} = 2K(\dyad{H})^{1/2}$.
Therefore, $\|K'^{-1}\| \ge \tfrac{1}{2}\|K^{-1}\|\,\|H^{-1}\|$.
By \cite[Lemma 6.3]{DBLP:conf/focs/MaSS16}, $\|H^{-1}\|^2 \ge \|T^{-1}\|$.
\end{remark}

\paragraph{Robustness.}
To ensure that our algorithm is robust to perturbations $E$, we must argue that the column span of $T$ and $T + E$ are close to each other so long as $E$ is bounded in spectral norm, and furthermore than the lifting operation still produces a subspace $V$ which is close to $\Span(\{W(a_i\ot a_i)\ot a_i\})$.
This is done via careful application of matrix perturbation analysis to the identifiability argument.
By operating with $W$ \emph{only} on third-order vectors and matrices over $(\R^d)^{\ot 3}$, we also avoid incurring factors of the fourth-order operator norm $\|T\|$ in the condition numbers, instead only incurring a much milder \emph{sixth}-order penalty $\|\sum\dyad{a_i^{\ot 3}}\|$.
For details, see \cref{sec:lift-robust}.

\paragraph{Rounding.}
If we are given direct access to $T$ in the absence of noise, the rounding stage can be accomplished with Gaussian rounding.
However when we allow $T$ to be adversarially perturbed the situation becomes more delicate.
Our rounding stage is an adaptation of \cite{DBLP:conf/colt/SchrammS17}, though some modifications are required for the additional challenges of the overcomplete setting.
It recovers the components of an approximation of a $3$-tensor with $n$ orthonormal components, provided that said approximation is within $\eps\sqrt{n}$ in Frobenius norm distance.
The technique is built around Gaussian rounding, but in order to have this succeed in the presence of $\eps\sqrt{n}$ Frobenius norm noise, the large singular values are truncated from the rectangular matrix reshapings of the $3$-tensor: this ensures that the rounding procedure is not entirely dominated by any spectrally large terms in the noise.

After we recover approximations of the orthonormal components $b_i \approx Wa_i^{\otimes 2}$, we wish to extract the $a_i$.
Naively one could simply apply $W^{-1}$, but this can cause errors in the recovered vectors to blow up by a factor of $\|W^{-1}\|$.
Even when the $\{a_i\}$ are random vectors, $\|W^{-1}\| = \Omega(\poly(d))$.\footnote{This is in contrast to $\|W\|$, which is $O(1)$ in the random case.}
Instead, we utilize the projector to $\Span\{W(a_i \ot a_i) \ot a_i \}$ computed in the lifting step: we {\em lift} $b_i$, project it into the span to obtain a vector close to $W(a_i \ot a_i) \ot a_i$, and reshape it to a $d^2 \times d$ matrix whose top right-singular vector is correlated with $a_i$.
This extraction-via-lifting step allows us to circumvent a loss of $\|W^{-1}\|$ in the error.

\begin{algorithm}
\caption{\,Sketch of full algorithm, in the absence of noise}
\label[algorithm]{alg:sketch}
    Input: A $4$-tensor $\T \in (\R^{d})^{\ot 4}$, so that $\T = \sum_{i=1}^n a_i^{\ot 4}$ for unit vectors $a_i \in \R^d$.
\begin{enumerate}
\item Take the square reshaping $T \in \R^{d^2 \times d^2}$ of $\T$ and compute its whitening $W = T^{-1/2}$ and the projector $\Pi_2 = WTW$ to the image of $T$.
\item \emph{Lifting}: Compute the lifted tensor $\T' \in (\R^{d^2})^{\ot 3}$ so that $\T' = \sum_{i}(W a_i^{\ot 2})^{\ot 3}$. (See \cref{alg:lift} for full details).
\begin{enumerate}
\item Find a basis for the subspace $S_3 = (\img T) \ot \R^d  \cap \Sym$:
  take $S_3$ to be the top-$n$ eigenspace of $(\Pi_2 \ot \Id)\Pisym(\Pi_2 \ot \Id)$.
 Then by \Cref{lem:sketch}, $S_3 = \Span(a_i^{\ot 3})$ .
\item Find the projector $\Pi_3$ to the space $(W \ot \Id)\,S_3 = \Span(W a_i^{\ot 2} \ot a_i)$.

\item Compute the orthogonal $3$-tensor:
    since $\{W a_i^{\ot 2} \ot a_i\}$ is an orthonormal basis,
  \[
      \Pi_3 = \sum\nolimits_i\dyad{(W a_i^{\ot 2} \ot a_i)}\,.\]
  Therefore, reshape $\Pi_3$ as
  $\sum_{i}  (W a_i^{\ot 2}) \ot (W a_i^{\ot 2}) \ot (a_i^{\ot 2})$
  and multiply $W$ into the third mode to obtain $\T'$.
\end{enumerate}
\item \emph{Rounding}: Use Gaussian rounding to find the components $a_i$.
    (In the presence of noise, this step becomes substantially more delicate; see \cref{alg:round,alg:clean,alg:test}).
    \begin{enumerate}
	\item Compute a random flattening of $\T'$ by contracting with $g\sim \cN(0,\Id_{d^2})$ along the first mode, $T'(g) = \sum_{i} \iprod{g,(Wa_i^{\ot 2})}\cdot (Wa_i^{\ot 2})(Wa_i^{\ot 2})^\top$
	\item Perform an SVD on $T'(g)$ to recover the eigenvectors $(Wa_1^{\ot 2}), \ldots, (Wa_n^{\ot 2})$.
	\item Apply $W^{-1}$ to each eigenvector to obtain the $a_i^{\otimes 2}$, and re-shape $a_i^{\ot 2}$ to a matrix and compute its eigenvector to obtain $a_i$.
    \end{enumerate}
\end{enumerate}
\end{algorithm}

\subsection*{Organization.}
The full implementation details and the analysis of our algorithm are given in the following few sections.
First, \Cref{sec:toolbox} sets up some primitives for spectral subspace perturbation analysis and linear-algebraic procedures on which we build the full algorithm and its analysis.
Then \Cref{sec:lifting} covers the lifting stage of the algorithm in detail, while \Cref{sec:rounding} elaborates on the rounding stage.
Finally, in \cref{sec:all-tog} we combine these tools to prove \cref{thm:main}.
The appendices some linear-algebraic tools and simulations strongly suggesting that random tensors with $n \ll d^2$ components have constant condition number $\kappa$.

    \section{Preliminaries}

\paragraph{Linear algebra}
We use $\Id_d$ to denote the $d\times d$ identity matrix, or just $\Id$ if the dimension is clear from context. For any subspace $S$, we use $\Pi_S$ to denote the projector to that subspace.
For $M$ a matrix, $\img(M)$ refers to the image, or columnspace, of $M$.

We will, in a slight abuse of notation, use $M^{-1}$ to denote the Moore-Penrose pseudo-inverse of $M$.
Except where explicitly specified, this will never be assumed to be equal to the proper inverse, so that, e.g., in general $MM^{-1} = \Pi_{\img(M)} \ne \Id$ and $(AB)^{-1} \ne B^{-1}A^{-1}$.

For a matrix $B \in \R^{m \times n}$, we will use the whitening matrix $W = (BB)^{-1/2}$, which maps the columns of $B$ to an orthonormal basis for $\img(B)$, so that $(WB)(WB)^\top = \Pi_{\img(B)}$.

We denote by $\Sym \subseteq (\R^{d})^{\tensor 3}$ the linear subspace
  \[
  \Sym = \Span (\{x \tensor x \tensor x\}_{x \in \R^d})\mper
  \]
Note that $(\Pisym)_{(i,j,k);(i',j',k')}$ is $(|\{i,j,k\}|!)^{-1}$ when $\{i,j,k\} = \{i',j',k'\}$ and zero otherwise.

\paragraph{Tensor manipulations}
When working with tensors $T \in (\R^{d})^{\otimes k}$, we will sometimes reshape the tensors to lower-order tensors or matrices; in this case, if $S_1,\ldots,S_m$ are a partition of $k$, then $T_{(S_1,\ldots,S_m)}$ is the tensor given by identifying the modes in each $S_i$ into a single mode.
For $S \subset [d]^k$, we will also sometimes use the notation $T(S)$ to refer to the entry of $T$ indexed by $S$.

A useful property of matrix reshapings is that $u \ot v$ reshapes into the outer product $u\transpose{v}$.
Linearity allows us to generalize this so, e.g., the reshaping of $(U \ot V)M$ for $U \in \R^{n \times n}$ and $V \in \R^{m \times m}$ and $M \in \R^{(n\ot m) \times q}$ is equal to $UM'(V \ot \Id_q)$, where $M' \in \R^{n \times (m \ot q)}$ is the reshaping of $M$.
Since reshapings can be easily done and undone by exchanging indices, these identities will sometimes allow more efficient computation of matrix products over tensor spaces.

We will on occasion use a $\cdot$ as a placeholder in a partially applied multiple-argument function: for instance $\frac{\partial}{\partial y} f(\cdot,y) = \lim_{h \to 0} \tfrac{1}{h}(f(\cdot,y+h) - f(\cdot,y))$.

    \section{Tools for analysis and implementation}
\label[section]{sec:toolbox}
In this section, we briefly introduce some tools which we will use often in our analysis.

\subsection{Robustness and spectral perturbation}

A key tool in our analysis of the robustness of \Cref{alg:sketch} comes from the theory of the perturbation of eigenvalues and eigenvectors.

The lemma below combines the Davis-Kahan $\sin$-$\Theta$ theorem with Weyl's inequality to characterize how top eigenspaces are affected by spectral perturbation.
\begin{theorem}[Perturbation of top eigenspace]
	\label[theorem]{eigen-perturb}
	Suppose $Q \in \R^{D\times D}$ is a symmetric matrix with eigenvalues $\lambda_1 \ge \lambda_2 \ge \dots \ge \lambda_D$.
	Suppose also $\Qish \in \R^{D\times D}$ is a symmetric matrix with $\|Q - \Qish\| \le \eps$.
	Let $S$ and $\Sish$ be the spaces generated by the top $n$ eigenvectors of $Q$ and $\Qish$ respectively.
	Then,
	\begin{equation}
	\sin(S,\Sish)\defeq\|\Pi_S-\Pi_{\Sish}\Pi_{S\vphantom{\Sish}}\| = \|\Pi_{\Sish}-\Pi_{S\vphantom{\Sish}}\Pi_{\Sish}\|
		\le \frac{\eps}{\lambda_{n} - \lambda_{n+1} - 2 \eps}\mper\label{eqn:eigen-perturb}
	\end{equation}
	Consequently,
	\begin{equation}
	\|\Pi_{S\vphantom{\Sish}}-\Pi_{\Sish}\|\le \frac{2\eps}{\lambda_{n} - \lambda_{n+1} - 2 \eps}\mper
	\end{equation}
\end{theorem}
\begin{proof}
We first prove the theorem assuming that $Q$ and $\Qish$ are symmetric.
By Weyl's inequality for matrices~\cite{Weyl1912}, the $n$th eigenvalue of $\Qish$ is at least $\lambda_n - 2\eps$.
By Davis and Kahan's $\sin$-$\Theta$ theorem~\cite{MR0264450-Davis70}, since the top-$n$ eigenvalues of $\Qish$ are all at least $\lambda_n - 2\eps$ and the lower-than-$n$ eigenvalues of $Q$ are all at most $\lambda_{n+1}$, the sine of the angle between $S$ and $\Sish$ is at most $\|Q - \Qish\|/(\lambda_{n} - \lambda_{n+1} - 2 \eps)$.
The final bound on $\|\Pi_S - \Pi_{\Sish}\|$ follows by triangle inequality.

\end{proof}

\subsection{Efficient implementation and runtime analysis}

It is not immediately obvious how to implement \Cref{alg:sketch} in time $\tilde O(n^2d^3)$, since there are steps that require we multiply or eigendecompose $d^{3} \times d^{3}$ matrices, which if done naively might take up to $\Omega(d^9)$ time.

To accelerate our runtime, we must take advantage of the fact that our matrices have additional structure. We exploit the fact that in certain reshapings our tensors have low-rank representations. This allows us to perform matrix multiplication and eigendecomposition (via power iteration) efficiently, and obtain a runtime that is depends on the rank rather than on the dimension.

For example, the following lemma, based upon a result of \cite{allen2016lazysvd}, captures our eigendecomposition strategy in a general sense.

\begin{lemma}[Implicit gapped eigendecomposition]
\label[lemma]{lem:implicit-gapped-svd}
Suppose a symmetric matrix $M \in \R^{d\times d}$ has an eigendecomposition $M = \sum_j \lambda_j \, v_j\transpose{v_j}$, and that $Mx$ may be computed within $t$ time steps for $x \in \R^d$.
Then $v_1, \dots, v_n$ and $\lambda_1, \dots, \lambda_n$ may be computed in time
$\tilde O(\min(n(t + nd)\delta^{-1/2}, d^3))$, where $\delta = (\lambda_n - \lambda_{n+1})/\lambda_{n}$.
The dependence on the desired precision is polylogarithmic.
\end{lemma}
\begin{proof}
The $n(t + nd)\delta^{-1/2}$ runtime is attained by LazySVD in \cite[Corollary 4.3]{allen2016lazysvd}.
While LazySVD's runtime depends on $\operatorname{nnz}(M)$ where $\operatorname{nnz}$ denotes the number of non-zero elements in the matrix, in the non-stochastic setting $\operatorname{nnz}(M)$ is used only as a bound on the time cost of multiplying a vector by $M$, so in our case we may substitute $O(t)$ instead.

The $d^3$ time is attained by iterated squaring of $M$: in this case, all runtime dependence on condition numbers is polylogarithmic.
\end{proof}

The following lemma lists some primitives for operations with the tensor $\T' \in (\R^{d^2})^{\ot 3}$ in \Cref{alg:sketch}, by interpreting it as a $6$-tensor in $(\R^{d})^{\ot 6}$ and using a low-rank factorization of the square reshaping of that $6$-tensor.

\begin{lemma}[Implicit tensors]
\label[lemma]{lem:implicit-tensor}
\torestate{
\label{lem:implicit-tensor}
For a tensor $\T \in (\R^{[d]^2})^{\ot 3}$, suppose that the matrix $T \in \R^{[d]^3 \times [d]^3}$ given by $T_{(i,i',j),(k,k',j')} = \T_{(i,i'),(j,j'),(k,k')}$ has a rank-$n$ decomposition $T = U\transpose{V}$ with $U,V \in \R^{d^3 \times n}$ and $n \le d^2$.
Such a rank decomposition provides an implicit representation of the tensor $\T$.
This implicit representation supports:
\begin{description}
\item[Tensor contraction:] For vectors $x,y \in \R^{[d]^2}$, the computation of $(\transpose{x}\ot\transpose{y}\ot\Id)\T$ or $(\transpose{x}\ot\Id\ot\transpose{y})\T$ or $(\Id\ot\transpose{x}\ot\transpose{y})\T$ in time $O(nd^3)$ to obtain an output vector in $\R^{d^2}$.
\item[Spectral truncation:] For $R \in \R^{d^2 \times d^4}$ equal to one of the two matrix reshapings $T_{\{1,2\}\{3\}}$ or $T_{\{2,3\}\{1\}}$ of $\T$, an approximation to the tensor ${\T}^{\le 1}$, defined as $\T$ after all larger-than-$1$ singular values in its reshaping $R$ are truncated down to $1$.
Specifically, letting $\rho_k$ be the $k$th largest singular value of $R$ for $k \le O(n)$, this returns an implicit representation of a tensor $\T{\kern 0.06em}'$ such that $\|\T{\kern 0.06em}' - {\T}^{\le 1}\|_F \le (1+\delta)\rho_k\|\T\|_F$ and the reshaping of $\T{\kern 0.06em}'$ corresponding to $R$ has largest singular value no more than $1 + (1+\delta)\rho_k$.
The representation of $\T{\kern 0.06em}'$ also supports the tensor contraction, spectral truncation, and implicit matrix multiplication operations, with no more than a constant factor increase in runtime.
This takes time $\tilde O(n^2d^3+ k(nd^3+kd^2)\delta^{-1/2})$.
\item[Implicit matrix multiplication:] For a matrix $R \in \R^{[d]^2 \times [d]^2}$ with rank at most $O(n)$, an implicit representation of the tensor $(\transpose{R}\ot\Id\ot\Id)\T$ or $(\Id\ot\Id\ot\transpose{R})\T$, in time $O(nd^4)$.
This output also supports the tensor contraction, spectral truncation, and implicit matrix multiplication operations, with no more than a constant factor increase in runtime.
Multiplication into the second mode $(\Id\ot\transpose{R}\ot\Id)\T$ may also be implicitly represented, but without support for the spectral truncation operation.
\end{description}
}
\end{lemma}

The implementation of these implicit tensor operations consists solely of tensor reshapings, singular value decompositions, and matrix multiplication.
However, the details get involved and lengthy, and so we defer their exposition to \Cref{sec:toolbox-appendix}.

    \section{Lifting}
\label[section]{sec:lifting}

This section presents \Cref{alg:lift}, which lifts a well-conditioned $4$-tensor $\T$ of rank at most $d^2$ in $(\R^d)^{\otimes 3}$ to $\T'$, an orthogonalized version of the $6$-tensor in the same components in $(\R^{d^2})^{\otimes 3}$; that is, we obtain an orthogonal $3$-tensor $\T'$ whose components correspond to the orthogonalized Kronecker squares of the components of $\T$.
\Cref{sec:identifiability} presents the identifiability argument giving robust algebraic non-degeneracy conditions under which the algorithm succeeds.

Although we assume that the tensor components $a_i$ are unit vectors, throughout this section we will keep track of factors of $\|a_i\|$ so as to better elucidate the scaling and dimensional analysis.

\begin{algorithm}
  \caption{\;Function $\mathsc{lift}(\T,n)$}
  \label[algorithm]{alg:lift}
  \noindent \emph{Input:} $\T \in (\R^{d})^{\otimes 4}, n \in \N$ with $n \leq d^2$.
  \begin{enumerate}
  \item
      Use \Cref{lem:implicit-gapped-svd} to find the top-$n$ eigenvalues and corresponding eigenvectors of the square matrix reshaping of $\T$, and call the eigendecomposition $T = Q\Lambda \transpose{Q}$.
      This also yields $W = Q\Lambda^{-1/2}\transpose{Q}$ and $\Pi_S = Q \transpose{Q}$.
  \item
      Use \Cref{lem:implicit-gapped-svd} again to find the top-$n$ eigendecomposition of $\Pi_{S \ot \R^d} \Pisym \Pi_{S \ot \R^d}$, implementing multiplication by $\Pi_{S\ot\R^d}$ as $(\Pi_{S\ot\R^d}v)_{(\cdot,\cdot,i)} = Q\transpose{Q}v_{(\cdot,\cdot,i)}$ and implementing $\Pisym$ as a sparse matrix.
      Call the result $R\Sigma\transpose{R}$ and take $\Pi_{S_3} = \dyad{R}$.
  \item
      Find a basis $B'$ for the columnspace of $M_3 = (W \ot \Id)\Pi_{S_3}(W \ot \Id)$.
      Implement this as
	\[ (B')_{(\,\cdot\,,\,\cdot\,,i);\,\cdot\,} =  Q\Lambda^{-1/2}\transpose{Q}R_{(\,\cdot\,,\,\cdot\,,i);\,\cdot\,} \,.\]
  \item
      Use Gram-Schmidt orthogonalization to find an orthonormalization $B$ of $B'$.
      Call the projection operator to this basis $\Pi_3 = \dyad{B}$.
  \item
      Instantiate an implicit tensor in $(\R^{d^2})^{\ot 3}$ with \Cref{lem:implicit-tensor}, using $\dyad{B}$ as the SVD of its underlying $d^3 \times d^3$ reshaping.
      Output this as $(\Id \ot W^{-1} \ot \Id)\T'$, meaning a tensor which, when $W^{-1}$ is multiplied into its second mode, becomes equal to $\T'$.
  \end{enumerate}
  \noindent \emph{Output:} $(\Id \ot W^{-1} \ot \Id)\T' \in (\R^{d^2})^{\otimes 3}$, implicitly as specified by \Cref{lem:implicit-tensor}, and $\Pi_3 \in \R^{d^3 \times d^3}$.
\end{algorithm}

The following two lemmas will argue that the algorithm is correct, and that it is fast.
First, \Cref{lem:lift-correct} states that the output of \Cref{alg:lift} is an orthogonal $3$-tensor whose components are $W(a_i\ot a_i)$, where the $a_i$ are the components of the original $4$-tensor and $W$ is the whitening matrix for the $a_i \ot a_i$.
Furthermore, if the error in the input is small in spectral norm compared to some condition numbers, the Frobenius norm error in the output robustly remains within a small constant of $\sqrt{n}$.

The main work of the lemma is deferred to \Cref{lem:lift-perturb} in \Cref{sec:lift-robust}, which repeatedly applies Davis and Kahan's $\sin$-$\Theta$ theorem (\Cref{eigen-perturb}) to say that the top eigenspaces of various matrices in the algorithm are relatively unperturbed in spectral norm by small spectral norm error in the matrices.
After that, we simply bound the Frobenius norm error of a rank-$n$ matrix by $2\sqrt{n}$ times its spectral norm error, and reason that Frobenius norms are unchanged by tensor reshapings.
\begin{lemma}[Correctness of \textsc{lift}]
  \label[lemma]{lem:lift-correct}
  Let $a_1,\ldots,a_n \in \R^d$ and suppose that $\TTish = \sum_{i \in [n]} a_i^{\tensor 4} + \mathbf{E}$ satifies $\|E_{12;34}\| \leq \eps\,\sigma_n^2\mu^{-1}\kappa^2$ for some $\eps < 1/63$, where $\sigma_n$ is the $n$th eigenvalue of $\sum_{i \in [n]} \dyad{a_i^{\ot 2}}$ and
  $\mu$ is the operator norm of $\sum \|a_i\|^{-2}\dyad{a_i^{\ot 3}}$ and $\kappa$ is the condition number from \Cref{lem:identifiability}.
  Let also $W = \left[\sum_{i \in [n]} \dyad{(a_i^{\tensor 2})}\right]^{-1/2}$ and $\Wish = \left[\sum_{i \in [n]} \dyad{(a_i^{\tensor 2})} + E_{12;34}\right]^{-1/2}$
  Then the outputs $(\Id \ot W^{-1} \ot \Id)\TTish'$ and $\Piish_3$ of $\mathsc{lift}(\TTish,n)$ in \Cref{alg:lift} satisfy
  \[
  \left \| \TTish' - \sum_i \|a_i\|^{-2}(W(a_i \ot a_i))^{\ot 3} \right \|_F \leq 126\,\eps\, \sigma_n^{-1/2} \sqrt{n}
  \]
and
  \[
  \left \| \Piish_3 - \Pi_{\Span\left(Wa_i^{\ot 2}\ot a_i^{\vphantom{1}}\right)} \right \| \leq 63\,\eps\,.
  \]
\end{lemma}
\begin{proof}
We refer to all matrices and spaces computed in the algorithm with an overset tilde to reflect the fact that the algorithm only has access to approximations with error (so $\Sish$ instead of $S$ in the algorithm, etc.).
By \Cref{lem:lift-perturb}, the $\Pi_{\Sish_3}$ computed in step 2 as the projector to the top-$n$ eigenspace of $\Pi_{\Sish \ot \R^d} \Pisym \Pi_{\Sish \ot \R^d}$ satisfies $\|\Pi_{\Sish_3} - \Pi_{\Span(a_i^{\ot 3})}\| \le 18\,\eps\sigma_n\mu^{-1}$, and subsequently, the $\Piish_3$ computed in steps 3 and 4 as the projector to $(\Wish\ot\Id)\Sish_3$ satisfies $\|\Piish_{3} - \Pi_{\Span(Wa_i^{\ot 2}\ot a_i^{\vphantom{1}})}\| \le 63\,\eps$.

Since the rank of the error is at most $2n$, the Frobenius norm error is at most $126\,\eps\sqrt{n}$, and since $\{\|a_i\|^{-1}Wa_i^{\ot 2} \ot a_i\}$ is an orthonormal set of vectors, the projector to $\Span(Wa_i^{\ot 2} \ot a_i)$ is just the sum of the self-outer-products of vectors in that set, so
\[ \left\|\Piish_3 - \sum \|a_i\|^{-2}\dyad{(Wa_i^{\ot 2} \ot a_i)} \right\|_F \le 126\,\eps\sqrt{n}\,. \]
Reshaping the $d^3 \times d^3$ matrix $\Piish_3$ into a tensor in $(\R^{d^2})^{\ot 3}$ does not change the Frobenius norm error, and finally, multiplying in the last factor of $W$ may contribute a factor of $\|W\| = \sigma_n^{-1/2}$, so that in the end, $\|\TTish' - \sum_i \|a_i\|^{-2}(W(a_i \ot a_i))^{\ot 3}\|_F \le 126\,\eps\, \sigma_n^{-1/2} \sqrt{n}$.
\end{proof}

The next lemma states that the running time is $\tilde O(n^2d^3)$ multiplied by some condition numbers. We assume that asympotically faster matrix multiplications and pseudo-inversions are not used, so that, for instance, squaring a $d \times d$ matrix takes time $\Theta(d^3)$.
\begin{lemma}[Running time of \textsc{lift}]
  \label[lemma]{lem:lift-time}
  Let $a_1,\ldots,a_n \in \R^d$ and suppose that $\T = \sum_{i \in [n]} a_i^{\tensor 4} + \mathbf{E}$ satisfies the conditions stated in \Cref{lem:lift-correct}.
  Let $\sigma_n$ be the $n$th eigenvalue of $\sum_{i \in [n]} \dyad{a_i^{\ot 2}}$ and $\kappa$ the condition number from \Cref{lem:identifiability}.
  Then $\mathsc{lift}(\T,n)$ in \Cref{alg:lift} runs in time $\tilde O(nd^4\sigma_n^{-1/2} + n^2d^3\kappa^{-1})$, and the efficient implementation steps are correct.
\end{lemma}
\begin{proof}
  Step 1 of \textsc{lift} invokes \Cref{lem:implicit-gapped-svd} on a $d^2 \times d^2$ matrix $T$, recovering $n$ dimensions with a spectral gap of $\delta = \sigma_n$.
  This requires time $\tilde O((nd^4 + n^2d^2)\sigma_n^{-1/2})$.

  Step 2 again invokes \Cref{lem:implicit-gapped-svd}, this time on a $d^3 \times d^3$ matrix $\Pi_{S \ot \R^d}\Pisym\Pi_{S \ot \R^d}$, recovering $n$ dimensions with a spectral gap of at least $\kappa^2 - 2\eps \in \Omega(\kappa^2)$.
  Multiplying by $\Pi_{S \ot \R^d}$ may be done in time $O(nd^3)$ due to its expression as $(\Pi_{S\ot\R^d}v)_{(\cdot,\cdot,i)} = Q\transpose{Q}v_{(\cdot,\cdot,i)}$, since the third mode of $(\R^d)^{\ot 3}$ is unaffected by $\Pi_{S \ot \R^d} = \dyad{Q} \ot \Id$, and this is a concatenation of $d$ different matrix-vector multiplies that take $O(nd^2)$ time each.
Multiplying by $\Pisym$ takes $O(d^3)$ time, since the $(i,j,k)$th row of $\Pisym$ has at most $6$ nonzero entries corresponding to the different permutations of $(i,j,k)$.
Thus the overall time to multiply a vector by $\Pi_{S \ot \R^d}\Pisym\Pi_{S \ot \R^d}$ is $O(nd^3)$, so that \Cref{lem:implicit-gapped-svd} gives a runtime of $\tilde O((n^2d^3 + n^2d^3)\kappa^{-1})$ for this step.

Step 3 is a concatenation of $d$ different matrix products, each of which involves multiplying a $d^2 \times n$ matrix $R_{(\cdot,\cdot,i);\cdot}$ by a $n \times d^2$ matrix $\Lambda^{1/2}\transpose{Q}$ and then multiplying the resulting $n \times n$ matrix by a $d^2 \times n$ matrix $Q$.
Each product thus takes $O(n^2d^2)$ time, and since there are $d$ of them the entire step takes $O(n^2d^3)$ time.
The result is equal to $(Q\Lambda^{1/2}\transpose{Q} \ot \Id_d)R = (W \ot \Id)R$, whose columns form a basis for the columnspace of $M_3 = (W \ot \Id)R\Sigma\transpose{R}(W \ot \Id)$.

Step 4 applies Gram-Schmidt orthonormalization on $n$ vectors in $\R^{d^3}$, taking $\tilde O(n^2d^3)$ time.
And step 5 takes constant time.
Therefore, \textsc{lift} takes time $\tilde O(nd^4\sigma_n^{-1/2} + n^2d^3\kappa^{-1})$.
\end{proof}

\subsection{Algebraic identifiability argument}
\label[section]{sec:identifiability}

The main lemma in this section gives a more careful analysis of the algebraic identifiability argument from \Cref{lem:sketch}, in order to obtain a quantitative condition number bound.

\begin{lemma}[Main Identifiability Lemma]
  \label[lemma]{lem:identifiability}
  Let $a_1,\ldots,a_n \in \R^d$ with $n \leq d^2$.
  Let $S$ denote $\Span(\{a_i^{\tensor 2}\})$ and let $S_3$ denote $\Span(\{a_i^{\tensor 3}\})$ and assume both have dimension $n$.
  For each $i$, let $\{b_{i,j}\}_{j \in [d-1]}$ be an arbitrary orthonormal basis for vectors in $\R^d$ orthogonal to $a_i$, and let
  \[ \transpose{H} := \left[\begin{array}{c} a_1 \ot a_1 \ot b_{1,1} \\ \vdots \\ a_i \ot a_i \ot b_{i,j} \\ \vdots \\ a_n \ot a_n \ot b_{n,d-1} \end{array}\right]\,. \]
  Let $R = (\dyad{H})^{-1/2}H$ be a column-wise orthonormalization of $H$, and let $K = \tfrac{1}{2}(\Id - P_{2,3})R$, where $P_{2,3}$ is the permutation matrix that exchanges the 2nd and 3rd modes of $(\R^d)^{\ot 3}$.
  Then if $\kappa = \smin(K)$ is non-zero (so that $K$ is full rank),
  \[ (S \ot \R^d) \cap \Sym = S_3\,, \]
  and furthermore,
  \[ \| \Pi_{S \ot \R^d}\Pi_{\Sym}\Pi_{S \ot \R^d} - \Pi_{S_3} \| \le 1-\kappa^2 \,. \]
\end{lemma}
\begin{proof}
Let $W = (\sum_i \dyad{a_i^{\ot 2}})^{-1/2}$ and let $T$ denote the columnspace of $(W^2\ot\Id)H$.
The columns of $W^2H$ form a basis for the subspace of $S \ot \R^d$ orthogonal to $S_3$ since each column of $W^2H$ is orthogonal to every $a_i^{\ot 3}$.
Therefore,
\[ \Pi_{S \ot \R^d} = \Pi_{S_3} + \Pi_T\,. \]
Multiplying this with $\Pisym$ and itself and then applying the identities $\Pi_{\Sym}\Pi_{S_3} = \Pi_{S_3}\Pi_{\Sym} = \Pi_{S_3}$ and $\Pi_{S_3}\Pi_T = \Pi_T\Pi_{S_3} = 0$,
\[ \Pi_{S \ot \R^d}\Pi_{\Sym}\Pi_{S \ot \R^d} = \Pi_{S_3} + \Pi_T\Pi_{\Sym}\Pi_T\,. \]
Therefore,
\[ \|\Pi_{S \ot \R^d}\Pi_{\Sym}\Pi_{S \ot \R^d} - \Pi_{S_3}\| \le \|\Pi_{\Sym}\Pi_T\|^2 \,. \]
We would thus like to show that $\|\Pi_{\Sym}\Pi_T\|^2 \le 1 - \kappa^2$.

Since $\|\Pisym\Pi_T\|^2 = \max_{y' \in T} \|\Pisym y'\|^2/\|y'\|^2 = 1 - \min_{y' \in T} \|(\Id - \Pisym) y'\|^2/\|y'\|^2$, it is enough to show that $\min_{y' \in T} \|(\Id - \Pisym) y'\|/\|y'\| \ge \kappa$.
By \Cref{lem:condition-conversion}, that is implied by $\|(\Id - \Pisym)y\|/\|y\| \ge \kappa$ for $y \in \img(H)$.

Since $\Pisym \preceq \Pi_{2,3}$ where $\Pi_{2,3}$ is the projector to the space invariant under interchange of the second and third modes of $(\R^d)^{\ot 3}$ and $\Pi_{2,3} = \tfrac{1}{2}(\Id + P_{2,3})$, we see that $\|(\Id - \Pisym)y\|/\|y\| \ge \|\tfrac{1}{2}(\Id - P_{2,3})y\|/\|y\|$ for $y \in \img(H)$.
Since the columns of $R$ are an orthonormal basis for $\img(H)$, for $x = R^{-1} y$ spanning all of $\R^n$ we have
\[\frac{\|(\Id - P_{2,3})y\|}{2\|y\|} = \frac{\|(\Id - P_{2,3})Rx\|}{2\|Rx\|} = \frac{\|(\Id - P_{2,3})Rx\|}{2\|x\|} = \frac{\|Kx\|}{\|x\|}\,.\]
The expression on the right is the definition of $\kappa$.
Therefore, $\|(\Id - \Pisym)y\|/\|y\| \ge \|\tfrac{1}{2}(\Id - P_{2,3})y\|/\|y\| = \kappa$.
\end{proof}

\begin{lemma}
  \label[lemma]{lem:condition-conversion}
  For each $i$, let $\{b_{i,j}\}_{j \in [d-1]}$ be an arbitrary orthonormal basis for vectors in $\R^d$ orthogonal to $a_i$, and let
  \[ \transpose{H} := \left[\begin{array}{c} a_1 \ot a_1 \ot b_{1,1} \\ \vdots \\ a_i \ot a_i \ot b_{i,j} \\ \vdots \\ a_n \ot a_n \ot b_{n,d-1} \end{array}\right]\,. \]
  Let $H' = (W^2 \ot \Id)H$.
  If $\|(\Id - \Pisym)y\| \ge t\|y\|$ for all $y \in \img(H)$, then $\|(\Id - \Pisym)y'\| \ge t\|y'\|$ for all $y' \in \img(H')$.
\end{lemma}
\begin{proof}

Let $S = \Span(a_i^{\ot 2})$ and $S_3 = \Span(a_i^{\ot 3}) \subseteq \Sym$.
Observe that $\img(H') \subseteq S \ot \R^d = \img(H) + S_3$.
Therefore, for every $y' \in \img(H')$ there will be some $y \in \img(H)$ and some $z \in S_3$ such that $y' = y + z$.
Also, since $S_3 \perp \img(H')$, we have $z \perp y'$, and therefore $\|y\| = \|y' - z\| \ge \|y'\|$.

So if the premise of the lemma holds and $\|(\Id - \Pisym)y\| \ge t\|y\|$ for all $y \in \img(H)$, it will also be the case that $\|(\Id - \Pisym)y'\| = \|(\Id - \Pisym)(y + z)\| \ge t\|y\| \ge t\|y'\|$.
\end{proof}

\subsection{Robustness arguments}
\label[section]{sec:lift-robust}

The main lemma of this section gives all of the spectral eigenspace perturbation arguments needed to argue the correctness and robustness of \Cref{alg:lift}.
Here we essentially repeatedly apply Davis and Kahan's $\sin$-$\Theta$ theorem (\Cref{eigen-perturb}) through a sequence of linear algebraic transformations, along with triangle inequality and some adding-and-subtracting, to argue that the desired top eigenspace remains stable against the spectral-norm errors melded in at each step.

\begin{lemma}[Subspace perturbation for \textsc{lift}]
  \label[lemma]{lem:lift-perturb}
  Let $T = \sum_{i \in [n]} \dyad{a_i^{\ot 2}}$ and let $\Tish$ be a matrix with $\|T - \Tish\| \le \eps\,\sigma_{n}^2\mu^{-1}\kappa^2$ for some $\eps < 1/63$, where $\sigma_n$ is the $n$th eigenvalue of $T$ and $\mu$ is the operator norm of $\sum \|a_i\|^{-2}\dyad{a_i^{\ot 3}}$ and $\kappa$ is the condition number from \Cref{lem:identifiability}.
  Let $S = \Span(\{a_i^{\ot 2}\}) = \img(T)$ and let $\Sish = \img(\Tish)$.
  Also let $S_3 = \Span(\{a_i^{\ot 3}\})$.
  Then
\[ \| \topn_n\left(\Pi_{\Sish \ot \R^d}\Pi_{\Sym}\Pi_{\Sish \ot \R^d}\right) - \Pi_{S_3} \| \le 18\,\eps\,\sigma_n\mu^{-1}\,,\]
  where $\topn_n$ denotes the top-$n$ eigenspace.
  Furthermore, letting $\Sish_3 = \topn_n(\Pi_{\Sish \ot \R^d}\Pi_{\Sym}\Pi_{\Sish \ot \R^d})$ and $W = T^{-1/2}$ and $\Wish = \Tish^{-1/2}$, we have
\[ \| \Pi_{(\Wish \ot \Id)\Sish_3} - \Pi_{(W \ot \Id)S_3} \| \le 63\,\eps\,.\]
\end{lemma}
\begin{proof}
For brevity, let $\Pi = \Pi_{S \ot \R^d}$ and let $\Piish = \Pi_{\Sish \ot \R^d}$.
We write
\begin{equation}
\label{eq:decomp-diff}
  \Piish\,\Pi_{\Sym}\,\Piish \,-\, \Pi\,\Pi_{\Sym}\,\Pi
\;\,=\,\; \left(\Piish - \Pi\right)\Pi_{\Sym}\,\Piish \,+\, \Pi\,\Pi_{\Sym}\left(\Piish - \Pi\right)\mper
\end{equation}
Since $\|T - \Tish\| \le \eps\,\sigma_{n}^2\mu^{-1}\kappa^2$, by \Cref{eigen-perturb}, $\|\Piish - \Pi\| = \|\Pi_{\Sish} - \Pi_S\| \le 3\eps\sigma_{n}\mu^{-1}\kappa^2$.
Since projectors don't increase spectral norm, we conclude
\[ \| \Piish\,\Pi_{\Sym}\,\Piish - \Pi\,\Pi_{\Sym}\,\Pi   \| \le 6\eps\sigma_{n}\mu^{-1}\kappa^2 \,.\]
Furthermore, by \Cref{lem:identifiability}, $ \Pi\,\Pisym\,\Pi = \Pi_{S_3} + Z$,
where $Z$ is a symmetric matrix with $\|Z\| \le 1-\kappa^2$ whose columnspace is orthogonal to $S_3$ since $\Pi_{S_3}\Pi_{S \ot \R^d}\Pisym\Pi_{S \ot \R^d} = \Pi_{S_3}$.
Therefore,
\[ \|\Piish\Pisym\Piish - (\Pi_{S_3} + Z)\| \le 6\eps\sigma_{n}\mu^{-1}\kappa^2\,. \]
The top-$n$ eigenspace of $(\Pi_{S_3} + Z)$ is $S_3$ and the $n$th and $(n+1)$th eigenvalues of $(\Pi_{S_3} + Z)$ differ by at least $\kappa^2$.
So by \Cref{eigen-perturb},
\[ \| \topn_n(\Piish\Pisym\Piish) - \Pi_{S_3} \| \le 18\eps\sigma_{n}\mu^{-1}\,.\]
Multiplying by $W$ multiplies this error by at most a factor of $\|W\|^2 = \sigma_n^{-1}$, so that
\[ \|(W \ot \Id)\Pi_{\Sish_3}(W \ot \Id) - (W \ot \Id)\Pi_{S_3}(W \ot \Id)\| \le 18\,\eps\mu^{-1}\,.\]
And $\|(\Wish \ot \Id)\Pi_{\Sish_3}(\Wish \ot \Id) - (W \ot \Id)\Pi_{\Sish_3}(W \ot \Id)\| \le 3\eps$ by a decomposition similar to \eqref{eq:decomp-diff} since $\Pi_{\Sish_3}(W \ot \Id)$ has a spectral norm at most $\sigma_n^{-1/2}$, so that
\[ \|(\Wish \ot \Id)\Pi_{\Sish_3}(\Wish \ot \Id) - (W \ot \Id)\Pi_{S_3}(W \ot \Id)\| \le 21\,\eps\mu^{-1}\,.\]
By \Cref{lem:whitening-3}, the smallest eigenvalue of $(W \ot \Id)\Pi_{S_3}(W \ot \Id)$ is at least $\mu^{-1}$.
Therefore, by \Cref{eigen-perturb}, $\| \Pi_{(W \ot \Id)\Sish_3} - \Pi_{(W \ot \Id)S_3} \| \le 63\,\eps$.
\end{proof}

The following utility lemma is used to reduce the impact of condition numbers on the algorithm.
It shows that when multiplying a third-order tensor in the span of $a_i^{\ot 3}$ by the second-order whitener $W = T^{-1/2} = (\sum_i \dyad{a_i^{\ot 2}})^{-1/2}$, the penalty to the error may be expressed in terms of a sixth-order condition number -- the spectral norm of $U = \sum \|a_i\|^{-2} \dyad{a_i^{\ot 3}}$ -- instead of the fourth-order one given by $T$.

The reason this is important is that $\sum_i \dyad{a_i^{\ot 2}}$ suffers from spurious directions: directions $v \in \R^{\ot 2}$ in which $Tv$ may be very large, but $v$ is not close to any of the $a_i \ot a_i$, or in fact any rank-$1$\, $2$-tensor at all.
For example, for $n$ random Gaussian vectors, the spurious direction is given by $\Phi = \E_{g \sim \N(0,1)} g \ot g$, which will have $\|T\Phi\| \approx n/d$.

The sixth-order object $U = \sum \|a_i\|^{-2} \dyad{a_i^{\ot 3}}$ does not suffer with this problem for $n$ up to $\tilde O(n^2)$, due to cancellation with the odd number of modes.
For instance, $U\E_{g \sim \N(0,1)} g^{\ot 3} = 0$ and $\|U(\Phi \ot u)\| \approx n/d^2$ for all unit $u \in \R^d$ and $U$ generated from random Gaussian vectors.

\begin{lemma}[Sixth-order condition numbers]
  \label[lemma]{lem:whitening-3}
  Let $a_1,\ldots,a_n \in \R^d$ with $n \leq d^2$.
  Let $W = (\sum_i \dyad{a_i^{\ot 2}})^{-1/2}$ have rank $n$.
  Let $U$ be the matrix $\sum \|a_i\|^{-2} \dyad{a_i^{\ot 3}}$.
  Then for a vector $v \in \Span(a_i^{\ot 3})$, the following hold:
\[ \|(W \ot \Id)v\| \le \|U^{-1}\|^{1/2}\|v\|\,,\]
\[ \|(W \ot \Id)v\| \ge \|U\|^{-1/2}\|v\|\,.\]
\end{lemma}
\begin{proof}
Let $v = \sum \mu_i \|a_i\|^{-1}a_i^{\ot 3}$. Then
\[\|(W \ot \Id)v\|^2 = \sum \mu_i\mu_j\|a_i\|^{-1}\|a_j\|^{-1}\iprod{W(a_j \ot a_j),W(a_i \ot a_i)}\iprod{a_j, a_i} = \sum \mu_i^2  \,, \]
using the fact that $\{W(a_i \ot a_i)\}_i$ is an orthonormal set of vectors.
\end{proof}

    \section{Rounding}

\label[section]{sec:rounding}

In this section, we show how to ``round'' the lifted tensor to extract the components.
That is, assuming we are given the tensor
\[
    T = \sum_{i \in [n]} (Wa_i^{\otimes 2})^{\otimes 3} + E
\]
where $E$ is a tensor of Frobenius norm at most $\epsilon \sqrt{n}$, we show how to find the components $a_i$.

\begin{lemma}\label[lemma]{lem:round}
    Suppose $a_1,\ldots,a_n \in \R^d$ are unit vectors satisfying the {\em identifiability assumption} from \cref{lem:lift-correct}, and suppose we are given an implicit rank-$n$ representation of the tensor $T = \sum_{i} (Wa_i^{\otimes 2})^{\otimes 3} + E \in (\R^{d^2})^{\otimes 3}$, where $\|E\|_F \le \epsilon \sqrt{n}$, and an implicit rank-$n$ representation of a matrix $\Pi_3$ such that $\|\Pi_3 - \sum_{i}((Wa_i^{\otimes 2})\otimes a_i)((Wa_i^{\otimes 2})\otimes a_i)^\top\| \le \epsilon < \frac{1}{2}$.

    Then for any $\beta,\delta \in (0,1)$ so that $\beta\delta = \Omega(\epsilon)$ and $\delta = \Omega(\epsilon)$, there is a randomized algorithm that with high probability in time $O(\frac{1}{\beta}n^{1+O(\beta)}d^3)$ with $\tilde O(n^2 d^3)$ preprocessing time recovers a unit vector $u$ such that for some $i \in [n]$,
    \[
	\iprod{a_i,u}^2 \ge 1-\|W\|\cdot O\left(\frac{\epsilon}{\beta}\right)^{1/8},
    \]
    so long as $\|W\|\left(\frac{\epsilon}{\beta}\right)^{1/8} < C$ for a universal constant $C$.

    Further, there is an integer $m \ge (1-\delta)n$ so that repeating the above algorithm $\tilde O(n)$ times recovers unit vectors $u_1,\ldots,u_m$ so that $\iprod{u_i,a_i}^2 \ge 1 - \|W\|\cdot O\left(\frac{\epsilon}{\delta\beta}\right)^{1/8}$ for all $i \in [m]$ (up to re-indexing), again so long as $\|W\|\left(\frac{\epsilon}{\delta\beta}\right)^{1/8} < C$, and with a total runtime of $\tilde O(\frac{1}{\beta}n^{2+O(\beta)}d^3)$.
\end{lemma}

We will prove this theorem in four steps.
First, in \cref{sec:const-cor} we will show how to recover vectors that are (with reasonable probability) correlated with the whitened Kronecker squares of the components, $Wa_i^{\otimes 2}$.
In \cref{sec:boost}, we'll give an algorithm that given a vector close to the whitened square $Wa_i^{\otimes 2}$, recovers a vector close to the component $a_i$.
In \cref{sec:testing}, we give an algorithm that tests if a vector $a \in \R^d$ is close to one of the components $\{a_i\}_{i\in [n]}$.
In these first three sections, we omit runtime details; in \cref{sec:together-round} we put the arguments together and address runtime details as well.

\subsection{Recovering candidate whitened and squared components}\label{sec:const-cor}
Here, we give an algorithm for recovering components that have constant correlation with the $Wa_i^{\otimes 2}$.
In this subsection, our result applies in generality to arbitrary orthonormal vectors $b_1,\ldots,b_n \in \R^{d^2}$.
The algorithm and its analysis follow almost directly from \cite{DBLP:conf/colt/SchrammS17}; for completeness we re-state the important lemmas here, and detail what little adaptation is necessary.

\begin{algorithm}
\caption{Rounding to a whitened component}\label[algorithm]{alg:round}
  Function $\mathsc{round}(T,\beta,\epsilon)$:\\
  \noindent \emph{Input:} a tensor $T \in (\R^{d^2})^{\otimes 3}$, a spectral gap bound $\beta$, and an error tolerance $\epsilon$.
  \begin{enumerate}
      \item Decrease the spectral norm of the error term in rectangular reshapings:\begin{enumerate}
	      \item compute $T'$, the projection of $T_{\{1,2\}\{3\}}$ to $O$, the set of $d^4 \times d^2$ matrices with spectral norm at most $1$
	      \item compute $T^{\le 1}$, the projection of $T'_{\{1,3\}\{2\}}$ to $O$ (may be done up to $\epsilon \sqrt{n}$ Frobenius norm error).
      \end{enumerate}
  \item Compute a random flattening of $T^{\le 1}$ along the $\{1\}$ mode:
      for $g \sim \cN(0,\Id_{d^2})$, compute
	  \[
	      T(g) = \sum_{i \in [d^2]} g_{i}\cdot T(i,\cdot,\cdot).\]
      \item Recover candidate component vectors: compute $u_L(g)$ and $u_R(g)$, the top left- and right-singular vectors of $T(g)$ using $O(\frac{1}{\beta}\log d)$ steps of power iteration.
  \end{enumerate}
    \emph{Output:} the candidate components $u_L(g)$ and $u_R(g)$.
    \medskip
\end{algorithm}

\begin{lemma}\label[lemma]{lem:rounding}
    Suppose that $b_1,\ldots,b_n \in \R^{d^2}$ are orthonormal.
    Then if $T = \sum_{i \in [n]} b_i^{\tensor 3} + E$ for a tensor $E$ with $\|E\|_F \le \epsilon\sqrt{ n}$ and $\delta = \Omega(\epsilon)$, $\Omega(\frac{\epsilon}{\delta}) \le \beta < 1$, repeating steps 2 \& 3 of \cref{alg:round} $\tilde O(n^{O(\beta)})$ times will with high probability recover a unit vector $u$ such that $\iprod{u,b_i}^2 \ge 1 - \frac{\epsilon}{\delta\beta}$ for some $i \in [n]$.
    Furthermore, repeating steps 2 \& 3 of \cref{alg:round} $\tilde O(n^{1+ O(\beta)})$ times will with high probability recover $m \ge (1-\delta)\cdot n$ unit vectors $u_1,\ldots,u_m$ such that for each $u_i$ there exists $j \in [n]$ so that $\iprod{u_i,b_j}^2 \ge 1- \frac{\epsilon}{\delta\beta}$.\footnote{
	In particular, if we choose $\delta = \log\log n \cdot \epsilon$, we will will recover all but $\epsilon\cdot n\log\log n$ of the $b_i$ in $\tilde O(n)$ repetitions.}
\end{lemma}
The proof follows from two lemmas:

\begin{lemma}\label[lemma]{lem:rect}
    The tensor $T^{\le 1}$ computed in step 1 of \cref{alg:round} remains close to $S = \sum_{i} b_i^{\otimes 3}$ in Frobenius norm, $\|T^{\le 1} - S\|_F \le \epsilon \sqrt{ n}$, and furthermore
    \[
	\|T^{\le 1}_{\{1,2\}\{3\}}\| \le 1 \qquad \text{and} \qquad \|T^{\le 1}_{\{1,3\}\{2\}}\| \le 1.
    \]
\end{lemma}
The proof of \cref{lem:rect} is identical to the proof of \cite[Lemma 4.5]{DBLP:conf/colt/SchrammS17}, and uses the fact that distances decrease under projection to convex sets to control the error, and the fact that the truncation operation is equivalent to multiplication by a contractive matrix to argue that $T^{\le 1}$ has bounded norm in both reshapings.

\begin{lemma}\label[lemma]{lem:prob}
    Suppose that in spectral norm $\|T_{\{1,2\}\{3\}}\|,\|T_{\{1,3\}\{2\}}\| \le 1$, and also that $\|T - \sum_i b_i^{\otimes 3}\|_F \le \epsilon \sqrt{n}$.
    Let $T(g)$ be the random flattening of $T$ produced in step 2 of \cref{alg:round}, and let $u_L(g)$ and $u_R(g)$ be the top left- and right-signular vectors of $T(g)$ respectively.
    Then there is a universal constant $C$ such that for any $\delta > C \cdot \epsilon$ and $\Omega(\frac{\epsilon}{\delta}) \le \beta < 1$, for a $1-\delta$ fraction of $j \in [n]$,
    \[
	\Pr_{g \sim N(0,\Id)}\left( \iprod{u_L(g), b_j}^2 \ge 1-\frac{\epsilon}{\delta\beta} \quad \text{or}\quad \iprod{u_R(g),b_j}^2 \ge 1-\frac{\epsilon}{\delta\beta} \right) \ge \tilde\Omega\left(n^{-1 - O(\beta)}\right),
	\]
	and further when this event occurs the ratio of the first and second singular values of $T(g)$ is lower bounded by $\beta$, $\frac{\sigma_{1}(T(g))}{\sigma_2(T(g))} \ge 1 + \beta$.
\end{lemma}
\begin{proof}
    By assumption, $T^{\le 1} = S + E$ for $S = \sum_{i\in[n]} b_i^{\otimes 3}$, and $E$ is a tensor of Frobenius norm at most $\epsilon \sqrt{n}$ and spectral norms $\|E_{\{1,2\}\{3\}}\|\le 1$ and $\|E_{\{1,3\}\{2\}}\|\le 1$.
    For $g \sim \cN(0,\Sigma^{-1})$, we have
    \[
	T(g) = S(g) + E(g) = \left(\sum_{k \in [n]}\iprod{g, b_k}\cdot b_k b_k^\top\right) + \left(\sum_{i \in [d^2]} g_{i}\cdot E_{i}\right),
    \]
    where we use $E_{i} = E(i,\cdot,\cdot)$ to refer to the $d^2 \times d^2$ matrix given by taking the $\{2\},\{3\}$ flattening of $E$ restricted to coordinate $i$ in mode $1$.

    The proof of the lemma is now identical to that of \cite[Lemma 4.6 and Lemma 4.7]{DBLP:conf/colt/SchrammS17}.
    There are two primary differences: the first is that in \cite{DBLP:conf/colt/SchrammS17} the tensor has four modes, and our tensor effectively has 3 modes.
    This difference is negligible, since in \cite{DBLP:conf/colt/SchrammS17}, two of the four modes are always identified anyway.

    The second difference is that we choose parameters differently.
    We take the parameter $\beta$ appearing in \cite[Lemma 4.6]{DBLP:conf/colt/SchrammS17} so that $\beta = \Omega(\frac{\epsilon}{\delta})$\footnote{We comment that the parameter $c$ appearing in the statement of \cite[Lemma 4.6]{DBLP:conf/colt/SchrammS17} is larger than $\sqrt{2}$; this is necessary for the application of \cite[Lemma 4.7]{DBLP:conf/colt/SchrammS17}, and is not clear from the lemma statement but is implicit in the proof.}; this is to emphasize that for small $\epsilon \ll \frac{1}{n}$, one can recover all $m = n$ of the components.
    Because the proof is otherwise the same, we merely sketch an overview here.

   The first term in isolation is a random flattening of an orthogonal tensor, and so with probability $1$ the eigenvectors of the first term are precisely the $b_k$.
    The second term, which is the flattening of the noise term, introduces complications; however, the combination of the spectral norm bound and the Frobenius norm bound on $E$ is enough to argue (using a matrix Bernstein inequality, Markov's inequality and the orthogonality of the $b_i$) that the random flattening of $E$ cannot have spectral norm larger than $\epsilon/\delta$ in more than $1-\delta$ of the $b_k$'s directions.

    To finish the proof, we perform a large deviation analysis on the coefficients $\iprod{g,b_k}$, lower bounding the probability that for the $1-\delta$ fraction of the $b_k$ that are not too aligned with the spectrum of $E$, there is a sufficiently large gap between $\iprod{g,b_k}$ and the $\iprod{g,b_i}$ for $i \neq k$ so that $b_k$ is correlated with the top singular vectors of $T(g)$.\footnote{We note that to obtain correlation $1-\frac{\epsilon}{\delta\beta}$, one must directly use the proof of \cite[Lemma 4.7]{DBLP:conf/colt/SchrammS17}, rather than the statement of the lemma (which has assumed that $\frac{2\epsilon(1+\beta)}{\beta\delta} \le 0.01$, and replaced the expression $1-\frac{2\epsilon(1+\beta)}{\beta\delta})$ with the lower bound $0.99$).}
    The bound on the ratio of the singular values comes from \cite[Lemma 4.7]{DBLP:conf/colt/SchrammS17} as well.
\end{proof}

\begin{proof}[Proof of \cref{lem:rounding}]
    The proof simply follows by applying \cref{lem:rect}, then \cref{lem:prob}.
\end{proof}

\subsection{Extracting components from the whitened squares}\label{sec:boost}
We now present the following simple algorithm which recovers a vector close to $a_i$, given a vector close to $W(a_i^{\otimes 2})$.
For convenience we will again work with generic orthonormal vectors $b_i$ in place of the $W(a_i^{\otimes 2})$, and we will assume we have access to the matrix $\Pi_3$ (the approximate projector to $\Span\{W(a_i^{\otimes 2})\otimes a_i\}$) computed in \cref{alg:lift}.
\begin{algorithm}\caption{Extracting the component from the whitened square}\label[algorithm]{alg:clean}
  Function $\mathsc{extract}(u, \Pi_3)$:\\
    \noindent \emph{Input:} a unit vector $u \in (\R^{d})^{\otimes 2}$ such that $\iprod{u,b_i}^2 \ge 1 - \theta$ for some $i \in [n]$, and a projector $\Pi_3 \in \R^{d^3\times d^3}$ such that $\|\Pi_3 - \sum_{i} b_ib_i^\top \otimes a_ia_i^\top\| \le \epsilon$.
  \begin{enumerate}
      \item
	    Compute the matrix $M = \Pi_3 (uu^\top \otimes \Id)$.
	\item Compute the top-left singular vector $v$ of $M$.
	\item Taking the reshaping $V = v_{\{3\}\{1,2\}}$, let $a = Vu$.
  \end{enumerate}
    \emph{Output:} the vector $a \in \R^d$
    \medskip
\end{algorithm}

\begin{lemma}\label[lemma]{lem:clean}
    Suppose $b_1,\ldots,b_n \in \R^{d^2}$ are orthonormal vectors and $a_1,\ldots,a_n \in \R^d$, and $\Pi_3 \in \R^{d^3 \times d^3}$ is such that $\|\Pi_3 - \sum_{i} b_i b_i^\top \otimes a_ia_i^\top\| \le \epsilon$.
    Then if $u\in \R^{d^2}$ is a unit vector with $\iprod{u,b_i}^2 \ge 1-\theta$ for $\theta < \frac{1}{10}$, then the output $a\in \R^d$ of \cref{alg:clean} on $u$ has the property that $|\iprod{a,a_i}| \ge 1 - 4\theta^{1/4} - 4\sqrt{\epsilon}$.
\end{lemma}
\begin{proof}
    Let $P_3 = \sum_i b_ib_i^\top \otimes a_ia_i^\top$.
    By assmption we can write the approximate projector $\Pi_3 = P_3 + E$, for a matrix $E$ of spectral norm $\|E\|\le \epsilon$.
    Based on these expressions we can re-express the product,
    \begin{align*}
	M = \Pi_3 (uu^\top \otimes \Id)
	&= P_3(uu^\top \otimes \Id) + E(uu^\top \otimes \Id).
    \end{align*}
    By assumption, the second term is a matrix of spectral norm at most $\|E\| \le \epsilon$.

    We now consider the first term.
    If $u = c\cdot b_i + w$, then for the first term we have
    \[
	P_3(uu^\top\otimes \Id) = P_3\left(c^2\cdot  b_ib_i^\top\otimes \Id\right) + P_3\left((c\cdot b_iw^\top + c\cdot wb_i^\top + ww^\top)\otimes \Id\right)
    \]
    The second term is again a matrix of spectral norm at most $3c\cdot \|w\| = 3 c\cdot \sqrt{1-c^2}$.
    The first term can be further simplified as
    \[
	P_3(c^2 \cdot b_ib_i^\top \otimes \Id)
	=c^2 \cdot \sum_{i} \iprod{b_i,b_i} b_ib_i^\top \otimes a_ia_i^\top
	= c^2 \cdot b_ib_i^\top \otimes a_ia_i^\top,
    \]
    by the orthogonality of the $b_i$.
    This is a rank-1 matrix with singular value $c^2$.
    Therefore,
    $M = c^2 (b_i \otimes a_i)(b_i\otimes a_i)^\top + \tilde{E}$ where $\|\tilde{E}\| \le \epsilon + 3c\sqrt{1-c^2}$.
    It follows from \cref{lem:topsv} that if $v$ is the top unit left-singular vector of $M$, then $\iprod{v, b_i\otimes a_i}^2 \ge 1-\frac{2}{c^2}\|\tilde{E}\|$.

    Now, in step 3 when we re-shape $v$ to a $d \times d^2$ matrix $V$ of Frobenius norm $1$, because $v$ is a unit vector we have that
    $V = a_ib_i^\top + \tilde{V}$ for $\tilde{V}$ of spectral norm $\|\tilde{V}\| \le \|\tilde{V}\|_F \le \sqrt{\frac{2}{c^2}\|\tilde{E}\|}$.
    Therefore, \[
	Vu = (a_i b_i^\top )(c\cdot b_i + w) + \tilde{V}u = c(1-\iprod{w,b_i})\cdot a_i + \tilde{V}u,
	\]
    and  the latter vector has norm at most $\|\tilde{V}\|$, and $\iprod{w,b_i} \le  \|w\| \le \sqrt{1-c^2}$.
    Finally, substituting $c = \sqrt{1-\theta}$ and using our bound on $\|\tilde{E}\|$ and $\|\tilde{V}\|$ and some algebraic simplifications, the conclusion follows.
\end{proof}

\begin{lemma}\label[lemma]{lem:topsv}
    Suppose that $M = uv^\top + E$ for $u\in\R^{d},v \in \R^{k}$ unit vectors and $E \in \R^{d\times k}$ a matrix of spectral norm $\|E\| \le \epsilon$.
    Then if $x,y$ are the top left- and right-singular vectors of $M$, $|\iprod{x,u}|,|\iprod{y,v}| \ge 1-2\epsilon$.
\end{lemma}
\begin{proof}
    Let $M = \sum_{i} \sigma_{i} x_i y_i^\top$ be the singular value decomposition of $M$, with $\sigma_1\ge \cdots \ge \sigma_d$.
We have that
    \[
	1-\epsilon \le u^\top M v \le \sigma_1.
    \]
    On the other hand, if with $\iprod{x_1,u} = \alpha \le 1$ and $\iprod{y_1,v} = \beta\le 1$,
    \[
	\sigma_1 = x_1^\top M y_1 \le \alpha\beta + \epsilon.
    \]
    Therefore,
    \[
	|\alpha|,|\beta| \ge \alpha\beta ge 1 - 2\epsilon,
    \]
    and thus $\min\{|\alpha|,|\beta|\} \ge 1-2\epsilon$.
\end{proof}

\subsection{Testing candidate components}\label{sec:testing}

The following algorithm allows us to test whether a candidate component $u$ is close to some component $a_i$.
\begin{algorithm}\caption{Testing component membership}\label[algorithm]{alg:test}
    Function $\mathsc{test}(u,\theta,\Pi_{S_3})$:\\
    {\em Input:} A unit vector $\hat{u}$, and the correlation parameter $\theta$.
    Also, $\Pi_{3}$, an approximate projector to $\Span\{(Wa_i^{\otimes 2})\otimes a_i\}$.
    \begin{enumerate}
	\item Compute $\rho = ((W\hat{u}^{\otimes 2})\otimes \hat{u})$.
	\item If $\|\Pi_3\rho\|_2^2 < (1-\theta)\|\rho\|_2^2$, return {\sc false}.
	    Otherwise, return {\sc true}.
    \end{enumerate}
\end{algorithm}
\begin{lemma}\label[lemma]{lem:test}
    Let $P_3$ be the projector to $\Span\{(Wa_i^{\otimes 2})\otimes a_i\}$, and suppose that we have $\Pi_3$ such that $\|\Pi_3 - P_3\| \le \epsilon < \frac{1}{2}$.
    Then if \cref{alg:test} is run on a vector $\hat{u}$ such that $\iprod{\hat{u},a_i}^2 \le 1-\theta - 2\epsilon$ for all $i \in [n]$, then \cref{alg:test} returns {\sc false}.

    Converseley, if \cref{alg:test} is run on a vector $\hat{u}$ with $\iprod{\hat{u},a_i}^2 \ge 1-\left(\frac{\theta - \epsilon}{10\|W\|}\right)^2 \ge 1 - \frac{1}{10}$ for some $i\in[n]$, then when run on a unit vector $\hat{u}$, \cref{alg:test} returns {\sc true}.
\end{lemma}
\begin{proof}
    By assumption, we can write $\Pi_{3} = P_3 + E$ for $P_3$ the projector to $\Span\{(Wa_i^{\otimes 2})\otimes a_i\}$ and $E$ a matrix of spectral norm at most $\epsilon$.
    From this, we have
    \begin{equation}
	\Pi_{3}(W\hat{u}^{\otimes 2})\otimes\hat{u} = P_3(W\hat{u}^{\otimes 2})\otimes\hat{u} + E(W\hat{u}^{\otimes 2})\otimes\hat{u},\label{eq:firstone}
    \end{equation}
    and  $\|E(W\hat{u}^{\otimes 2})\otimes\hat{u}\|\le \epsilon \|W\hat{u}^{\otimes 2}\|$.
    Now, we can write $W\hat{u}^{\otimes 2} = \sum_{i} c_i W(a_i\otimes a_i) + e$, where $e$ is orthogonal to $\Span\{Wa_i^{\otimes 2}\}$, and we can further write
    \[
	(W \hat{u}^{\otimes 2})\otimes \hat{u} = \sum_{i \neq j} c_i\gamma_j\cdot (Wa_i^{\otimes 2})\otimes b^{(i)}_j + \sum_{i} c_i\gamma_i \cdot (Wa_i^{\otimes 2})\otimes a_i + e \otimes \hat{u},
    \]
    where $\{b^{(i)}_j\}_{j \neq i}$ is an orthogonal basis for the orthogonal complement of $a_i$ in $\R^d$.
    By definition, $P_3(Wa_i^{\otimes 2})\otimes b_{j}^{(i)} = 0$, as this is orthogonal to every vector in $\Span\{(Wa_i^{\otimes 2})\otimes a_i\}$.
    Therefore,
    \[
	P_3(W \hat{u}^{\otimes 2})\otimes \hat{u} = \sum_{i} c_i\gamma_i \cdot (Wa_i^{\otimes 2})\otimes a_i.
    \]

    Now, if $\iprod{\hat{u},a_i}^2 \le \tau$ for all $i \in [n]$, then $\gamma_i^2 \le \tau$ for all $i \in [n]$.
    It thus follows that $\|P_3(W\hat{u}^{\otimes 2})\otimes \hat{u}\|^2 \le\max_{i} \gamma_i^2 \cdot \sum_{j} c_j^2 \le \tau \cdot\|W\hat{u}^{\otimes 2}\|^2_2$.
    Combining this with \cref{eq:firstone}, we have that
    \[
	\|\Pi_3(W\hat{u}^{\otimes 2})\otimes \hat{u}\|_2^2 \le (\tau +\epsilon\sqrt{\tau}+\epsilon^2)\cdot \|W\hat{u}^{\otimes 2}\|^2_2 \le (\tau + 2\epsilon)\|W\hat{u}^{\otimes 2}\|_2^2,
    \]
    for $\epsilon < \frac{1}{2}$.
    It follows that if $\iprod{\hat{u},a_i}^2 = \tau < 1-\theta - 2\epsilon$ for all $i \in [n]$, then the algorithm returns {\sc false}.

    Converseley, if without loss of generality $\hat{u} = \zeta\cdot a_1 + \hat{e}$ for $\hat{e} \in \R^d$ orthogonal to $a_1$, then $W\hat{u}^{\otimes 2} = \zeta^2 \cdot Wa_1^{\otimes 2} + W e'$ with $\|We'\|_2^2 \le (1-\zeta^2)\cdot\|W\|^2$.
    Measuring the correlation of $W\hat{u}^{\otimes 2}$ with $Wa_1^{\otimes 2}$, we have that $c_1 \ge \zeta^2$.
    Also $\gamma_1 \ge \zeta$, which implies
    \[
P_3(W\hat{u}^{\otimes 2})\otimes\hat{u}=
	\zeta^3 \cdot (Wa_1^{\otimes 2})\otimes a_1 + \tilde{e}.
    \]
    where $\tilde{e}$ is a leftover term with $\|\tilde{e}\| \le \sqrt{1-\zeta^2}\cdot \|W\hat{u}^{\otimes 2}\| + \zeta^2\sqrt{1-\zeta^2}$ (where we have used the PSDness of $W$).
    Combining this with \cref{eq:firstone},
    \[
	\Pi_3(W\hat{u}^{\otimes 2})\otimes\hat{u} = \zeta^3 (Wa_1^{\otimes 2}) \otimes a_1 + \tilde{e} + E(W\hat{u}^{\otimes 2}) \otimes \hat{u}.
    \]
    For convenience let $\hat{\rho} = \tilde{e} + E(W\hat{u}^{\otimes 2})\otimes \hat{u}$; from our previous observations, we have $\|\hat{\rho}\|\le (\epsilon +\sqrt{1-\zeta^2})\|W\hat{u}^{\otimes 2}\| + \zeta^2\sqrt{1-\zeta^2}$.

    Now, if $\iprod{\hat{u},a_1}^2 = \zeta^2 \ge 1-\eta$, we have that
    \[
	(1-\eta) - \sqrt{\eta}\|W\| \le \|W\hat{u}^{\otimes 2}\| \le (1-\eta) + \sqrt{\eta}\|W\|.
    \]
    From this,
    \begin{align*}
	\|\Pi_3(W\hat{u}^{\otimes 2})\|
	\ge \zeta^3 - \|\hat\rho\|
	&\ge (1-\eta)^{3/2} - (\epsilon + \sqrt{\eta})\|W\hat{u}^{\otimes 2}\| - (1-\eta)\sqrt{\eta}\\
	&\ge \sqrt{1-\eta}(\|W\hat{u}^{\otimes 2}\| - \sqrt{\eta}\|W\|) - (\epsilon + \sqrt{\eta})\|W\hat{u}^{\otimes 2}\| - (1-\eta)\sqrt{\eta}\\
	&\ge (1-\epsilon - 2\sqrt{\eta})\|W\hat{u}^{\otimes 2}\| - 2\sqrt{\eta}\|W\|,\\
	&\ge (1-\epsilon - 5\sqrt{\eta}\|W\|)\|W\hat{u}^{\otimes 2}\|.
    \end{align*}
    where we have used that $\eta < \frac{1}{10}$.
    Thus, if $\eta < \left(\frac{\theta - \epsilon}{10\|W\|}\right)^2$, \cref{alg:test} does not return {\sc false}.
\end{proof}

\subsection{Putting things together}\label{sec:together-round}

Finally, we prove \cref{lem:round}.
\begin{proof}[Proof of \cref{lem:round}]
    By the assumptions of the theorem, we have access to an implicit rank-$n$ representation of $T = \sum_{i \in [n]} (W(a_i^{\otimes 2}))^{\otimes 3} + E \in (\R^{d^2})^{\otimes 3}$,
    where $W = \left(\sum_{i \in [n]} (a_i^{\otimes 2})(a_i^{\otimes 2})^\top\right)^{-1/2}$, and with $\|T - E\|_F \le \epsilon \sqrt{n}$.
    For convenience we denote $b_i = W(a_i^{\otimes 2})$.
    Note that the $b_i$ are orthonormal vectors in $\R^{d^2}$.
    We also have implicit access to a rank-$n$ representation of $\Pi_3$, where $\|\Pi_3 - \sum_{i} (b_i \otimes a_i)(b_i \otimes a_i)^\top\| \le \epsilon$.

    We first run step 1 of \cref{alg:round} to produce the tensor which we will round.
    Then, for $\ell = \tilde O(n^{1+O(\beta)})$ independent iterations, we run steps 2 \& 3 of \cref{alg:round} to produce candidate whitened squares $u_1,\ldots,u_{\ell}$, then run \cref{alg:clean} on the $u_i$ to produce candidate components $\hat{u}_i$, and finally run \cref{alg:test} to check if $\hat{u}_i$ is close to $a_j$ for some $j \in [n]$.

    We show that step 1 of \cref{alg:round} takes time $\tilde O(n^2d^3)$.
Since $T$ is at most $\eps\sqrt{n}$ in Frobenius norm away from a tensor that is a rank-$n$ projector in both rectangular reshapings $T_{\{1,2\},\{3\}}$ and $T_{\{2,3\},\{1\}}$, the $(2n)$th singular values in either reshaping must be at most $\eps$: otherwise the error term would have over $n$ singular values more than $\eps$ and therefore Frobenius norm more than $\eps\sqrt{n}$.
Also $\|T\|_{F} = \sqrt{n}$ because it is a rank-$n$ projector in its square matrix reshaping.
    Therefore, by \cref{lem:implicit-tensor}, step 1 requires time $\tilde O(n^2d^3 + n(nd^3 + nd^2))$ to return an $\eps \sqrt{n}$-approximation in Frobenius norm to the projected matrix.\footnote{Some of the lemmas we apply, out of concerns for compatibility with \cite{DBLP:conf/colt/SchrammS17}, assume that the maximum singular value of $T^{\le 1}$ is at most $1$.
    Though one could re-do the previous analysis with minimal consequences under the assumption that the spectral norm is at most $1+\eps$, for brevity we note that we may instead multiply the whole tensor by $\frac{1}{1-\epsilon}$, and because the tensor has Frobenius norm at most $(1+3\epsilon)\sqrt{n}$, this costs at most $4\eps \sqrt{n}$ additional Frobenius norm error.}
    Note that this step only needs to be carried out once regardless of how many times the algorithm is invoked for a specific input $T$, so the $\tilde O(n^2d^3)$ runtime is incurred as a preprocessing cost.

    Then, again by \cref{lem:implicit-tensor}, steps 2 \& 3 require time $\tilde O(\frac{1}{\beta}nd^3)$, since the ratio of the first and second singular values of the the matrix is $1 + \Omega(\beta)$, and since $O(\frac{1}{\beta}\log d)$ steps of power iteration with $T(g)$ can be implemented by choosing the random direction $g \sim \cN(0,\Id_{d^2})$, the starting direction $v_1 \in \R^{d^2}$, and then computing $v_{t+1} = (\Id \otimes g^\top \otimes v_t)T^{\le 1}$ where $T^{\le 1}$ is the truncated tensor.

    Thus, if we choose $\beta, \delta$ satisfying the requirements of \cref{lem:rounding}, after $\tilde O(n^{O(\beta)})$ iterations of steps 2 \& 3 we will recover a vector $u \in \R^{d^2}$ such that $\iprod{u,b_i}^2 \ge 1-3\frac{\eps}{\beta}$, and after $\tilde O(n^{1+O(\beta)})$ iterations of steps 2 \& 3 we will recover vectors $u_{t_1},\ldots,u_{t_m}$ so that $\iprod{u_{t_i},b_i}^2 \ge 1-3\frac{\eps}{\beta\delta}$ for $m \ge (1-\delta)n$ of the $i \in [n]$.

    Next, applying \cref{lem:clean} to each of the good candidate vectors obtained in \cref{alg:round}, \cref{alg:clean} will give us candidate components $\hat{u}_{t_1},\ldots,\hat{u}_{t_m}$ so that $\iprod{\hat{u_{t_i}},a_i}^2 \ge 1-4\sqrt{\epsilon} - 4\left(\frac{3\epsilon}{\delta\beta}\right)^{1/4}$.
    Since $\Pi_3$ has rank $n$, we write it as $\dyad{U}$ for $U \in \R^{d^3 \times n}$.
    Then we may reshape $(\dyad{u} \otimes \Id)\Pi_3$ as $\dyad{u}U'(\transpose{U} \ot \Id)$, where $U'$ is the $d^2 \times nd$ reshaping of $U$.
    Multiplying $\transpose{u}$ through takes $O(nd^3)$ time and then reshaping the result back results in $(\dyad{u} \otimes \Id)\Pi_3$.
Therefore, by \Cref{lem:implicit-gapped-svd}, each invocation of \cref{alg:clean} requires $\tilde O(nd^3)$ operations.

    Finally, from \cref{lem:test}, we know that if we run \cref{alg:test} with $\theta = 10\|W\|\cdot \left(2\epsilon^{1/4} + 2\left(\frac{3\epsilon}{\delta\beta}\right)^{1/8}\right) + 2\epsilon$, we will reject any $\hat u$ such that $\iprod{\hat{u},a_i}^2 \le 1-\theta-2\epsilon$ for all $i \in [n]$, and will keep all of the good outputs of \cref{alg:clean}.
    Each iteration of \cref{alg:test} requires time $ O(d^4 + nd^3 + d^3)$, since we form the vector $(W\hat{u}^{\otimes 2})\otimes \hat{u}$, then multiply with the rank-$n$ matrix $\Pi_3$, and ultimately compute a norm.

    This completes the proof.
\end{proof}

    \section{Combining \textsc{lift} and \textsc{round} for final algorithm}\label{sec:all-tog}

In this section we describe and analyze our final tensor decomposition algorithm, proving our main theorem.

\begin{algorithm}
\caption{Main algorithm for overcomplete 4-tensor decomposition}\label[algorithm]{alg:main}
  Function $\mathsc{decompose}(T)$:\\
  \noindent \emph{Input:} a tensor $T \in (\R^{d})^{\otimes 4}$, numbers $\beta,\delta,\epsilon \in (0,1)$, numbers $\sigma, \kappa_0 \in \R_{\geq 0}$, and $n \leq d^2$.
  \begin{enumerate}
      \item Run $\textsc{lift}(T,n)$ from \cref{alg:lift} to obtain an implicit tensor $T'$ and an implicit matrix $\Pi_3$, using $\sigma, \kappa_0$ as upper bounds on condition numbers $\sigma_n, \kappa$.
      \item Run the algorithm specified by \cref{lem:round} on input $(T',\Pi_3, \e,\beta,\delta)$ with independent randomness $t=\tilde O(n)$ times, to obtain vectors $u_1,\ldots,u_t$.
  \end{enumerate}
    \emph{Output:} $u_1,\ldots,u_t$
    \medskip
\end{algorithm}

\begin{definition}[Signed Hausdorff distance]
  For sets of vectors $a_1,\ldots,a_n \in \R^d$ and $b_1,\ldots,b_m \in \R^d$, we define the signed Hausdorff distance to be the maximum of the following two quantities.
  (1) $\max_{i \in [n]} \min_{j \in [m], \sigma \in \pm 1} \|a_i - \sigma b_j\|$ and (2) $\max_{i \in [m]} \min_{j \in [n], \sigma \in \pm 1} \|b_i - \sigma a_j\|$.
\end{definition}

\begin{definition}[Condition number of $a_1,\ldots,a_n$]\label[definition]{def:condition-kappa}
  Let $a_1,\ldots,a_n \in \R^d$.
  Let $\{b_{ij}\}_{j \in [d-1]}$ be an arbitrary orthonormal basis for the orthogonal complement of $a_i$ in $\R^d$.
  Let
  \[ \transpose{H} := \left[\begin{array}{c} a_1 \ot a_1 \ot b_{1,1} \\ \vdots \\ a_i \ot a_i \ot b_{i,j} \\ \vdots \\ a_n \ot a_n \ot b_{n,d-1} \end{array}\right]\,. \]
  Let $R = (\dyad{H})^{-1/2}H$ be a column-wise orthonormalization of $H$, and let $K = \tfrac{1}{2}(\Id - P_{2,3})R$, where $P_{2,3}$ is the permutation matrix that exchanges the 2nd and 3rd modes of $(\R^d)^{\ot 3}$.
  The condition number $\kappa$ of $a_1,\ldots,a_n$ is the minimum singular value of $K$.
\end{definition}

\begin{theorem}\label[theorem]{thm:main}
  For every $d,n \in \N$ and $\e,\beta,\delta \in (0,1)$ and $\sigma,\kappa_0 \in \R_{\geq 0}$ there is a randomized algorithm $\mathsc{decompose}_{d,n,\e,\beta,\delta,\sigma,\kappa_0}(T)$ with the following guarantees.
  For every set of unit vectors $a_1,\ldots,a_n \in \R^d$ and every $E \in (\R^{d})^{\tensor 4}$  such that
  \begin{enumerate}
    \item the operator norm of the square matrix flattening of $E$ satisfies $\tfrac{\|E_{12;34}\|}{\sigma_n^{7} \mu^{-1} \kappa^2} \leq \e$,
    \item $\kappa = \kappa(a_1,\ldots,a_n) \geq \kappa_0$
    \item $\sigma_n \geq \sigma$
  \end{enumerate}
  where
  \begin{enumerate}
    \item $\sigma_n$ is the $n$-th singular value of the matrix $\sum_{i \leq n} \dyad{(a_i^{\tensor 2})}$,
    \item $\mu$ is the operator norm of $\sum_{i \leq n} \dyad{(a_i^{\tensor 3})}$, and
    \item $\kappa$ is the condition number of $a_1,\ldots,a_n$ as in \cref{def:condition-kappa}.
  \end{enumerate}
  there is a subset $S \subseteq \{a_1,\ldots,a_n\}$ of size $|S| \geq (1-\delta)n$ such that given input $T = \sum_{i \leq n} a_i^{\tensor 4} + E$ the algorithm produces a set $B = \{b_1,\ldots,b_t\}$ of $t = \tilde O(n)$ vectors which with probability at least $0.99$ over the randomness in the algorithm has
  \[
  \textsc{signed-Hausdorff-distance}(S,B) \leq O\Paren{\frac \e {\delta \beta}}^{1/16}\mper
  \]
  Furthermore, the algorithm $\mathsc{decompose}_{d,n,\e,\beta,\delta,\sigma,\kappa_0}$ runs in time
  \[
    \tilde O\Paren{ \frac{n d^4}{\sqrt{\sigma}} + \frac{n^2 d^3}{\kappa_0} + \frac{n^{2+O(\beta)}d^3}{\beta} }\mper
  \]
\end{theorem}

We record some intuitive explanations of the parameters in \cref{thm:main}.
\begin{itemize}
\item $\sigma,\kappa_0$ are bounds on the minimum singular values of matrices associated to $a_1,\ldots,a_n$, used to determine the necessary precision of linear-algebraic manipulations performed by the algorithm. Decreasing $\sigma,\kappa_0$ yields an algorithm tolerating less well-conditioned tensors, at the expense of running time and/or accuracy guarantees.
\item $\delta$ determines what fraction of the vectors $a_1,\ldots,a_n$ the algorithm is allowed to fail to return. By decreasing $\delta$ the algorithm recovers a larger fraction of $a_1,\ldots,a_n$, at the cost of increasing running time and/or decreasing per-vector accuracy.
\item $\beta$ determines the per-vector accuracy of the algorithm. Increasing $\beta$ improves the accuracy of the algorithm, but with exponential cost in the running time.
\item $\e$ governs the magnitude of allowable noise $E$. Increasing $\e$ yields a more noise-tolerant algorithm, at the expense of the accuracy of recovered vectors.
\end{itemize}

We record the following corollary, which follows from \cref{thm:main} by choosing parameters appropriately.

\begin{corollary}\label{cor:main-with-params}
  For every $n,d \in \N$ and $\sigma > 0$ (independent of $n,d$) there is an algorithm with the following guarantees.
  The algorithm takes input $T = \sum_{i \leq n} a_i^{\tensor 4} + E$, and so long as
  \begin{enumerate}
    \item $\kappa(a_1,\ldots,a_n) \geq \sigma$
    \item the minimum nonzero eigenvalue of $\sum_{i \leq n} (a_i^{\tensor 2})(a_i^{\tensor 2})^\top$ is at least $\sigma$
    \item $\|\sum_{i \leq n} (a_i^{\tensor 3})(a_i^{\tensor 3})^\top \| \leq 1/\sigma$, and
    \item $\|E_{12;34}\| \leq \poly(\sigma) / (\log n)^{O(1)}$,
  \end{enumerate}
  with high probability the algorithm recovers $\tilde O(n)$ vectors $b_1,\ldots,b_t$ such that there is a set $S \subseteq \{a_1,\ldots,a_n\}$ with $|S| \geq (1-o(1)) n$ such that the signed Hausdorff distance from $S$ to $\{b_1,\ldots,b_t\}$ is $o(1)$, in time $\tilde O(n^2 d^3 / \poly(\sigma))$.

  Furthermore, hypotheses (2),(3) hold for random unit vectors $a_1,\ldots,a_n$ with $\sigma = 0.1$ so long as $n \leq d^2 / (\log n)^{O(1)}$, and experiments in Appendix~\ref{sec:sims} strongly suggest that (1) does as well.
\end{corollary}

\begin{proof}[Proof of \cref{thm:main}]
  Let $W = (\sum_{i \leq n} a_i^{\tensor 2} (a_i^{\tensor 2})^\top)^{-1/2}$.
  By \cref{lem:lift-correct}, the implicit tensor $T'$ and matrix $\Pi_3$ returned by \textsc{lift} satisfy
  \[
    \left \| T' - \sum_{i \leq n} (W(a_i \tensor a_i))^{\tensor 3} \right \|_F \leq O(\e \sigma_n^{9/2} \sqrt n)
  \]
  and
  \[
  \left \| \Pi_3 - \Pi_{\Span(W a_i^{\tensor 2} \tensor a_i)} \right \| \leq O (\e \sigma_n^{4})\mper
  \]
  So, by \cref{lem:round}, with high probability there is a subset $S \subseteq \{a_1,\ldots,a_n\}$ of size $m \geq (1-\delta)n$ such for each $a_i \in S$ there is $u_j$ among the vectors $u_1,\ldots,u_t$ returned by the rounding algorithm with
  \[
  \iprod{a_i,u_j}^2 \geq 1- O\Paren{ \frac 1 {\delta^{1/8}} \cdot \frac 1 {\beta^{1/8}} \cdot \e^{1/8} \cdot \|W\| \cdot \sqrt{\sigma_n}} = 1 - O \Paren{\frac \e {\delta \beta}}^{1/8}
  \]
  where the equality follows because $\|W\| = \sigma_n^{-1/2}$.
  Furthermore, each of the vectors $u_1,\ldots,u_t$ is similarly close to some $a_i \in S$.
  This proves the claimed upper bound on the Hausdorff distance.

  The running time follows from putting together \cref{lem:lift-time} and the running time bounds of \cref{lem:round}.
\end{proof}

    \section{Condition number of random tensors}
\label[section]{sec:random-tensors}

\newcommand{\shortfrac}[2]{\mathopen{} \vphantom{g} \smash{\tfrac{{#1}}{{#2}}} \mathclose{}}

\begin{definition}
  \label[definition]{def:Pi23}
  Let $\Pi_{2,3}: (\R^d)^{\ot 3} \to (\R^d)^{\ot 3}$ be the orthogonal projector to the subspace $\Span(x \ot y \ot y \mid x,y \in \R^d)$ that is invariant under interchange of the second and third tensor modes.
  Let $\Pi_{1,2}$ be defined similarly.
  Let $\Pi_{2,3}^{\perp} = \Id - \Pi_{2,3}$ and $\Pi_{1,2}^{\perp} = \Id - \Pi_{1,2}$.
  
  Note that $\Pi_{2,3} = \tfrac{1}{2}(\Id + P_{2,3})$, where $P_{2,3}$ is the orthogonal operator that interchanges the second and third modes, and $\Pi_{2,3}^{\perp} = \tfrac{1}{2}(\Id - P_{2,3})$.
  This follows from the Projection Formula in representation theory, whereby for any group $G$ of linear operators, $\tfrac{1}{|G|}\sum_{g \in G} g$ is equal to the projection to the common invariant subspace of $G$.
\end{definition}

\begin{lemma}[Condition number of basic swap matrix]
  \label[lemma]{lem:basic-swap-cond}
  Let $a_1,\ldots,a_n$ be independent random $d$-dimensional unit vectors.
  Let $B_i \in \R^{(d-1) \times d}$ be a random basis for the orthogonal complement of $a_i$ in $\R^d$.
  Let $P \in \R^{d^3 \times d^3}$ be the permutation matrix which swaps second and third modes of $(\R^d)^{\tensor 3}$.
  Let
  \[
    A = \E_{a} (a \tensor a \tensor \Id)(a \tensor a \tensor \Id)^\top\mper
  \]
  Let $R \in \R^{d^3 \times n(d-1)}$ have $n$ blocks of dimensions $d^3 \times (d-1)$, where the $i$-th block is
  \[
  R_i = A^{-1/2} (a_i \tensor a_i \tensor B_i) - P A^{-1/2} (a_i \tensor a_i \tensor B_i )
  \]
  where we abuse notation and denote the PSD square root of the pseudoinverse of $A$ by $A^{-1/2}$.
  Then there is a function $d'(d) = \Theta(d^2)$ such that $\E \|R^\top R - d'(d) \cdot \Id\| \leq O(\log d)^2 \cdot \max(d\sqrt{n}, n, d^{3/2})$.
  In particular, if $d \ll n \ll d^2$,
  \[
  \E \| \tfrac 1 {d'(d)} R^\top R - \Id \| \leq O( n (\log d)^2 / d^2)\mper
  \]
\end{lemma}

\begin{corollary}
  \label[corollary]{cor:basic-swap-cond}
  Let $a_1,\ldots,a_n$ be independent random $d$-dimensional unit vectors with $d \ll n$.
  Then with probability $1-o(1)$, the condition number $\kappa$ of $a_1,\ldots,a_n$ as defined in \Cref{lem:identifiability} is at least $\sqrt{\tfrac{1}{4} - \tfrac{1}{4\sqrt{2}}} - O(\tfrac{1}{d}) -  \tilde O(\tfrac{n}{d^2})$, the matrix $T = \sum \dyad{a_i^{\ot 2}}$ has $n$-th eigenvalue $\sigma \in 1 - \tilde O(n/d^2) - O(1/d)$, and the matrix $U = \sum \dyad{a_i^{\ot 3}}$ has spectral norm $1 + \tilde O(n/d^2)$.
  
  Therefore, $\mathsc{decompose}_{d,n,\e,\beta,\delta,\sigma,\kappa_0}$ when run with error $\eps$, recovers $(1-\sigma)n$ components $a_i$ with signed Hausdorff distance $O(\eps/(\delta\beta))^{1/16}$ in time $\tilde O(nd^4 + n^2d^3 + n^{2+O(\beta)}d^3/\beta)$.
\end{corollary}

To prove the corollary, we will need an elementary fact.

\begin{fact}\label[fact]{fact:U-projector}
  Let $a_1,\ldots,a_n$ be independent random unit vectors in $\R^d$.
  Let $U = \sum_{i=1}^n \dyad{a_i^{\tensor 3}}$.
  With probability $1-o(1)$,
  $\| U - \Pi_{\img(U)} \| \leq \tilde O(n/d^2)$ if $n \gg d$.
\end{fact}

The proof may be found in \cref{sec:omitted-proofs}.

\begin{fact}\label[fact]{fact:A-T-projector}
  Let $a_1,\ldots,a_n$ be independent random unit vectors in $\R^d$.
  Let $\hat{\Sigma} = \E_a \dyad{(a \ot a)}$ for $a$ drawn from the uniform distribution over unit vectors in $\R^d$.
  Then \[\|2d^{-2}\hat{\Sigma}^{-1/2}T\hat{\Sigma}^{-1/2} - \Pi_{\hat{\Sigma}^{-1/2}\img(T)}\| \le \tilde O(n/d^2) \,.\]
\end{fact}
\begin{proof}
The fact follows by a slight modification of \cite[Lemma 5.9]{DBLP:conf/stoc/HopkinsSSS16}.
While Lemma 5.9 in \cite{DBLP:conf/stoc/HopkinsSSS16} applies to Gaussian random vectors, not uniform random unit vectors, and gives a bound of $\tilde O(n/d^{3/2})$, with minor changes, it holds for unit vectors and with a bound of $\tilde O(n/d^2)$).
\end{proof}

\begin{fact}\label[fact]{fact:cosine-similarity-and-difference-of-projectors}
  Suppose $A$ and $B$ are both subspaces of dimension $n$.
  Suppose $\theta \in [0, \pi/2)$.
  Then $\|\Pi_A - \Pi_B\Pi_A\| \le \sin \theta$ if and only if for every $x \in A$ there is a $y \in B$ so that
  \[ \frac{\iprod{x,y}}{\|x\|\,\|y\|} \ge \cos \theta\,. \]

  Futhermore, when this holds, since $\|\Pi_A - \Pi_B\Pi_A\| = \|\Pi_B - \Pi_A\Pi_B\|$ by symmetry, the triangle inequality yields $\|\Pi_A - \Pi_B\| \le 2\sin \theta$.
\end{fact}
\begin{proof}
Consider the product $R = \Pi_A\Pi_B$.
Let $R = U\Sigma\transpose{V}$ be its singular value decomposition, with $u_i$ and $v_i$ its $i$th left- and right-singular vectors and $\sigma_i$ its $i$th singular value.
We show as an intermediate step that $\sigma_n$ is at least $\cos \theta$ if and only if for every $x \in A$ there is a $y \in B$ so that
$\iprod{x,y} \ge \|x\|\,\|y\|\,\cos \theta$.

In one direction, suppose $\sigma_n$ is at least $\cos \theta$.
Then take $y = \transpose{R}x$.
We see that $\iprod{x,y} = \transpose{x}\Pi_A\Pi_B y = \transpose{x}Ry = \|\transpose{R}x\|^2$.
Since $x \in A$ and $\dim A = n$ and $\img(R) \subseteq A$ and singular values are non-negative, if $\sigma_n > 0$ then the first $n$ left-singular vectors of $R$ must span $A$.
Therefore, decomposing $x = \sum \alpha_i u_i$, we must have $\alpha_i = 0$ for $i > n$.
So, $\transpose{R}x = \sum \alpha_i V\Sigma\transpose{U} u_i = \sum \alpha_i \sigma_i v_i$, and
$\|\transpose{R}x\|^2 = \|\sum \alpha_i \sigma_i v_i\|^2 = \sum \alpha_i^2 \sigma_i^2 \|v_i\|^2 \ge \sum \alpha_i^2 \sigma_n^2 = \|x\|^2\sigma_n^2$.
So
\[ \frac{\iprod{x,y}}{\|x\|\,\|y\|} = \frac{\|\transpose{R}x\|^2}{\|x\|\,\|\transpose{R} x\|} = \frac{\|\transpose{R}x\|}{\|x\|} \ge \sigma_n \ge \cos \theta\,. \]

In the other direction, suppose $\sigma_n < \cos\theta$.
Then take $x = u_n$ and for any $y \in B$, decompose $y = \sum \beta_i v_i$ where again $\beta_i = 0$ for all $i > n$.
So $\iprod{x,y} = \transpose{x}\Pi_A\Pi_B y = \transpose{(\transpose{R}x)}y = \sigma_n \transpose{v_n}y = \sigma_n \beta_n$.
Since $\beta_n \le \|y\|$, for any $y$ we must have
\[ \frac{\iprod{x,y}}{\|x\|\,\|y\|} = \frac{\sigma_n \beta_n}{\|y\|} \le \sigma_n < \cos \theta\,. \]

Finally, we show that $\|\Pi_A - \Pi_B\Pi_A\| \le \sin \theta$ if and only if $\sigma_n$ is at least $\cos \theta$.
This follows from the Pythagorean theorem, as $\Pi_A = (\Pi_B\Pi_A) + (\Pi_A - \Pi_B\Pi_A)$ and also $\transpose{(\Pi_B\Pi_A)}(\Pi_A - \Pi_B\Pi_A) = 0$.
Thus for any vector $v$, we see $\|(\Pi_A - \Pi_B\Pi_A)v\|^2 = \|\Pi_Av\|^2 - \|\Pi_B\Pi_Av\|^2 = \|\Pi_Av\|^2 - \|\transpose{R}\Pi_Av\|^2$.
\end{proof}

\begin{fact}\label[fact]{fact:phi-against-sym23}
  Let $\Phi = \sum e_i \ot e_i$ for $e_i$ the elementary basis vectors in $\R^d$.
  If $x \in \Phi \ot \R^d$ then $\|\Pi_{2,3}x\| = \tfrac{1}{\sqrt{2}}\sqrt{1 + \shortfrac{1}{d}}\|x\|$.
\end{fact}
\begin{proof}
Since $x \in \Phi \ot \R^d$, we may write it as $x = \Phi \ot u$ for some $u \in \R^d$.
We directly calculate:
\begin{align*}
\|\Pi_{2,3}x\|^2
&= \|\tfrac{1}{2}(\Id + P_{2,3})x\|^2
\\&= \tfrac{1}{4}\|x\|^2 + \tfrac{1}{4}\|P_{2,3}x\|^2 + \tfrac{1}{2}\iprod{x, P_{2,3}x}
\\&= \tfrac{1}{2}\|x\|^2 + \tfrac{1}{2}\iprod{\Phi \ot u,\, P_{2,3}(\Phi \ot u)}
\\&= \tfrac{1}{2}\|x\|^2 + \tfrac{1}{2}\iprod{\sum e_i \ot e_i \ot u,\, \sum e_j \ot u \ot e_j}
\\&= \tfrac{1}{2}\|x\|^2 + \tfrac{1}{2}\sum\sum \iprod{e_i,e_j}\iprod{e_i,u}\iprod{e_j,u}
\\&= \tfrac{1}{2}\|x\|^2 + \tfrac{1}{2}\sum\iprod{e_i,u}^2
\\&= \tfrac{1}{2}\|x\|^2 + \tfrac{1}{2}\|u\|^2
\\&= \tfrac{1}{2}\|x\|^2(1 + \tfrac{1}{d})\,.
\end{align*}
\end{proof}

\begin{fact}\label[fact]{fact:phi-against-sym23-and-sym12}
  Let $\Phi = \sum e_i \ot e_i$ for $e_i$ the elementary basis vectors in $\R^d$.
  Let $u \in \Sym_2 \ot \R^d$ and decompose $u = x + y$ with $x \in \Phi \ot \R^d$ and $y \perp x$.
  Then if $\|\Pi_{2,3}^{\perp}y\| \ge \tfrac{1}{\sqrt{2}}\|y\|$, it holds that
  \[ \|\Pi_{2,3}^{\perp}u\|^2 \ge \left(\tfrac{1}{4} - \tfrac{1}{4\sqrt{2}} - O(\tfrac{1}{d})\right)\|u\|^2.\]
\end{fact}
\begin{proof}
Let $P_{3,2,1} = (P_{1,2,3})^{-1}$ be the orthogonal linear operator permuting the tensor modes, so that the 3-cycle $(3\; 2\; 1)$ replaces the third mode with the second, the second mode with the first, and the first mode with the third again.
By a unitary similarity transform conjugating by $P_{3,2,1}$, since $x \in \Phi \ot \R^d$, we have
\begin{align}
\|\Pi_{1,2}P_{2,3}x\| 
&= \|(P_{1,2,3}\Pi_{1,2}P_{3,2,1})(P_{1,2,3}P_{2,3})x\|  \nonumber
\\&= \|\Pi_{2,3}P_{1,2}x\|  \nonumber
\\&= \|\Pi_{2,3}x\|  \nonumber
\\& = \tfrac{1}{\sqrt{2}}\sqrt{1 + \shortfrac{1}{d}}\|x\|
\label[equation]{eq:Pi12-P23-with-Phi-Rd}
\,,\end{align}
where the last step uses \Cref{fact:phi-against-sym23}.
We write, since $\Pi_{1,2}x = x$,
\begin{align}
\|\Pi_{1,2}\Pi_{2,3}^{\perp}x\| 
&= \tfrac{1}{2}\|\Pi_{1,2}(x - P_{2,3}x)\|  \nonumber
\\&= \tfrac{1}{2}\|x - \Pi_{1,2}P_{2,3}x\|  \nonumber
\\&\le \tfrac{1}{2}\|x\| + \tfrac{1}{2}\|\Pi_{1,2}P_{2,3}x\|  \nonumber
\\&=\tfrac{1}{2}\|x\| + \tfrac{1}{2\sqrt{2}}\sqrt{1 + \shortfrac{1}{d}}\|x\| \nonumber
\\&=\left(\tfrac{1}{2} + \tfrac{1}{2\sqrt{2}} + O(\tfrac{1}{d})\right)\|x\|
\label[equation]{eq:Pi12-PiP23-with-Phi-Rd}
\,,\end{align}
where the second-to-last step substitutes in \eqref{eq:Pi12-P23-with-Phi-Rd}.

Therefore, since $u = x + y$ and \Cref{fact:phi-against-sym23} implies $\|\Pi_{2,3}^{\perp}x\|^2  \ge (\tfrac{1}{2} - O(\tfrac{1}{d}))\|x\|^2$ and also by assumption $\|\Pi_{2,3}^{\perp}y\|^2  \ge \tfrac{1}{2}\|y\|^2$,
\begin{align*}
\|\Pi_{2,3}^{\perp}u\|^2
&= \|\Pi_{2,3}^{\perp}x\|^2 + \|\Pi_{2,3}^{\perp}y\|^2 + \iprod{y, \Pi_{2,3}^{\perp}x}
\\&= (\tfrac{1}{2} - O(\tfrac{1}{d}))\|x\|^2  + \tfrac{1}{2}\|y\|^2 + \iprod{y, \Pi_{2,3}^{\perp}x}
\,.\end{align*}
Since $y \in \Sym_2 \ot \R^d$ and therefore $\Pi_{1,2} y = y$, Cauchy-Schwarz implies $\iprod{y, \Pi_{2,3}^{\perp}x} \ge -\|\Pi_{1,2}\Pi_{2,3}^{\perp}x\| \|y\|$ and then by \eqref{eq:Pi12-PiP23-with-Phi-Rd}, this is at least $-(\tfrac{1}{2} + \tfrac{1}{2\sqrt{2}} + O(\tfrac{1}{d}))\|x\|$.
We substitute this in and then apply Young's inequality:
\begin{align*}
\|\Pi_{2,3}^{\perp}u\|^2
&\ge \left(\tfrac{1}{2} - O(\tfrac{1}{d})\right)\|x\|^2+  \tfrac{1}{2}\|y\|^2 - \left(\tfrac{1}{2} + \tfrac{1}{2\sqrt{2}} + O(\tfrac{1}{d})\right)\|x\|\,\|y\|
\\&\ge \left(\tfrac{1}{2} - O(\tfrac{1}{d})\right)\|x\|^2 + \tfrac{1}{2}\|y\|^2 - \left(\tfrac{1}{2} + \tfrac{1}{2\sqrt{2}} + O(\tfrac{1}{d})\right)(\tfrac{1}{2}\|x\|^2 + \tfrac{1}{2}\|y\|^2)
\\&\ge \left(\tfrac{1}{4} - \tfrac{1}{4\sqrt{2}} - O(\tfrac{1}{d})\right)\|x\|^2  + \left(\tfrac{1}{4} - \tfrac{1}{4\sqrt{2}} - O(\tfrac{1}{d})\right)\|y\|^2
\\&= \left(\tfrac{1}{4} - \tfrac{1}{4\sqrt{2}} - O(\tfrac{1}{d})\right)\|u\|^2
\,.\end{align*}
\end{proof}

\begin{fact}\label[fact]{fact:static-whitening-not-making-things-that-much-better}
  Let $u \in \Sym_2 \ot \R^d$. If
  \[\frac{\|\Pi_{2,3}^{\perp}A^{-1/2}u\|}{\|A^{-1/2}u\|} \ge 1 - \mu\]
  for $\mu \le 1 - \tfrac{1}{\sqrt{2}} - \tfrac{2\sqrt{2}}{\sqrt{d}}$, then
  \[\frac{\|\Pi_{2,3}^{\perp}u\|}{\|u\|} \ge \sqrt{\tfrac{1}{4} - \tfrac{1}{4\sqrt{2}}} - O(\tfrac{1}{d})\,.\]
\end{fact}
\begin{proof}
Write $u = x + y$ where $x = (\tfrac{1}{d}\dyad{\Phi} \ot \Id)u$ and $y \perp x$.
If $\|x\| > 2\|y\|$, then by triangle inequality and \Cref{fact:phi-against-sym23},
\begin{align*} \frac{\|\Pi_{2,3}^{\perp}u\|}{\|u\|}
&\ge \frac{\|\Pi_{2,3}^{\perp}x\| - \|\Pi_{2,3}^{\perp}y\| }{\|u\|}
\\&\ge \frac{\|\Pi_{2,3}^{\perp}x\| - \|y\| }{\|u\|}
\\&> \frac{\left(\tfrac{1}{\sqrt{2}} - O(\tfrac{1}{d})\right) \|x\| - \tfrac{1}{2}\|x\| }{\|u\|}
\\&\ge \frac{\left(\tfrac{1}{\sqrt{2}} - \tfrac{1}{2} - O(\tfrac{1}{d})\right)\|x\|}{\|x\| + \|y\|}
\\&> \frac{\left(\tfrac{1}{\sqrt{2}} - \tfrac{1}{2} - O(\tfrac{1}{d})\right)\|x\|}{\tfrac{3}{2}\|x\|}
\\&> \sqrt{\tfrac{1}{4} - \tfrac{1}{4\sqrt{2}}} - O(\tfrac{1}{d})\,.
 \end{align*}
 Thus for the remainder of the argument, we assume $\|x\| \le 2\|y\|$.

Then $A^{-1/2}u = d_1 y + \sqrt{d}\,x$ and $\|A^{-1/2}u\|^2 = d_1^2\|y\|^2 + d\|x\|^2 \ge d_1^2\|y\|^2$, so by triangle inequality,
\begin{align*}
\|\Pi_{2,3}^{\perp}y\|
&\ge d_1^{-1}\|\Pi_{2,3}^{\perp}A^{-1/2}u\| - \tfrac{\sqrt{d}}{d_1}\|\Pi_{2,3}^{\perp}x\|
\\&\ge d_1^{-1}\|\Pi_{2,3}^{\perp}A^{-1/2}u\| - \tfrac{\sqrt{d}}{d_1}\|x\|
\\&\ge d_1^{-1}\|\Pi_{2,3}^{\perp}A^{-1/2}u\| - 2\tfrac{\sqrt{d}}{d_1}\|y\|
\\&\ge d_1^{-1}(1-\mu)\|A^{-1/2}u\| - 40\tfrac{\sqrt{d}}{d_1}\|y\|
\\&\ge (1-\mu)\|y\| - 2\tfrac{\sqrt{d}}{d_1}\|y\|
\\&\ge \left(1 - \mu - \tfrac{2\sqrt{2}}{\sqrt{d}}\right)\|y\|
\,.\end{align*}
Therefore, the lemma follows by \Cref{fact:phi-against-sym23-and-sym12}, as long as $1 - \mu - \tfrac{2\sqrt{2}}{\sqrt{d}} \ge \tfrac{1}{\sqrt{2}}$.
\end{proof}

\begin{proof}[Proof of \cref{cor:basic-swap-cond}]
To lower bound $\kappa$, we need a lower bound on the least singular value of $\Pi_{2,3}^{\perp}(\dyad{H})^{-1/2}H$, where $H$ is the matrix with columnwise blocks of $a_i \ot a_i \ot B_i$.
By \Cref{lem:basic-swap-cond}, with probability $1-o(1)$, it holds that $\tfrac{\sqrt{2}}{d_1}\Pi_{2,3}^{\perp}A^{-1/2}H$ has all singular values within $1 \pm \tilde O(n/d^2)$.
This means that for all $u \in \img(H)$, we have $\|\Pi_{2,3}^{\perp}A^{-1/2}u\|/\|A^{-1/2}u\| \ge 1 -  \tilde O(n/d^2)$.
Therefore, by \Cref{fact:static-whitening-not-making-things-that-much-better}, $\|\Pi_{2,3}^{\perp}u\|/\|u\| \ge \sqrt{\tfrac{1}{4} - \tfrac{1}{4\sqrt{2}}} - O(\tfrac{1}{d}) -  \tilde O(n/d^2)$.
This inequality holding for all $u \in \img(H)$ is equivalent to $\kappa$, the smallest singular value of $\Pi_{2,3}^{\perp}(\dyad{H})^{-1/2}H$, being at least $\sqrt{\tfrac{1}{4} - \tfrac{1}{4\sqrt{2}}} - O(\tfrac{1}{d}) -  \tilde O(n/d^2)$.
\end{proof}

The remainder of this section is devoted to the proof of \Cref{lem:basic-swap-cond}.
  At a high level, this proof follows the strategy laid out in 
in \cite{MR2963170-Vershynin12} to prove Theorem 5.62 there, but the random matrix we need to control is much more complicated than is handled there.

\subsection{Notation} Throughout we use the following notation.
\begin{enumerate}
  \item $d, n \in \N$ are natural numbers.
  \item $d_1 = \sqrt{(d^2 + 2d)/2} = \Theta(d)$ and $d_2 = d^{-1/2} (\sqrt{(d+2)/2} -1) = \Theta(1)$.
  \item $a_1,\ldots,a_n \in \R^d$ are iid random unit vectors.
  \item $B_i \in \R^{d \times (d-1)}$ for $i \leq n$ is a matrix with columns which form a random orthonormal basis for the orthogonal complement of $a_i$ in $\R^d$. (Chosen independently from $a_{j},B_j$ for $j \neq i$.)
  \item $\Sigma = \E_a (aa^\top \tensor aa^\top) \in \R^{d^2 \times d^2}$ is the $4$-th moment matrix of a random $d$-dimensional unit vector.
  \item $A = \E_a (aa^\top \tensor aa^\top \tensor \Id) = \Sigma \tensor \Id \in \R^{d^2 \times d^2 \times d^2}$ is $\Sigma$ ``lifted'' to a 6-tensor.
   \item $\Phi \in \R^{d^2}$ is the vector $\Phi = \sum_{i \leq d} e_i^{\tensor 2}$, where $e_i$ is the $i$-th standard basis vector.
   \item $\Pisym \in \R^{d^2 \times d^2}$ is the projector to the symmetric subspace of $\R^{d^2}$ (i.e. the span of vectors $x^{\tensor 2}$ for $x \in \R^d$).
  \item $P \in \R^{d^3 \times d^3}$ is the permutation matrix which swaps second and third tensor modes.
  Concretely, $(Px)_{ijk} = x_{ikj}$ for $i,j,k \in [d]$.
  \item $S_i$, for $i \leq n$, is the $d^3 \times d-1$ matrix given by $S_i = A^{-1/2} (a_i \tensor a_i \tensor B_i)$
  \item $R_i$, for $i \leq n$, is the $d^3 \times d-1$ matrix given by $R_i = S_i - P S_i$.
  \item $R_T$, for any $T \subseteq [n]$, is the $d^3 \times |T|(d-1)$ matrix with $|T|$ blocks of columns, given by $\{R_i\}_{i \in T}$.
  \item $R = R_{[n]} \in \R^{d^3 \times n(d-1)}$ contains all blocks of columns $R_i$.
\end{enumerate}

\subsection{Fourth Moment Identities}

\begin{fact}
  \label[fact]{fact:pseudoinv}
  \[
    \Paren{\E aa^\top \tensor aa^\top}^{-1/2} = d_1 \Pisym - d_2 \Phi \Phi^\top
  \]
  and for any unit $x$ and matrix $X$,
  \[
  A^{-1/2}(x \tensor x \tensor X) = \Brac{d_1 (x \tensor x) - d_2 \Phi } \tensor X\mper
  \]
\end{fact}
\begin{proof}
  The first statement follows from Fact C.4 in \cite{DBLP:conf/stoc/HopkinsSSS16}.
  For the second, notice that $A = \Paren{\E aa^\top \tensor aa^\top}^{-1/2} \tensor \Id$, so
  \[
  A^{-1/2} = \Brac{\sqrt{\frac{d^2 + 2d}{2}} \Pisym + \frac 1 {\sqrt d} \Paren{1 - \sqrt{\frac{d+2}{2}}} \Phi \Phi^\top} \tensor \Id\mper
  \]
  So we can expand $A^{-1/2}(x \tensor x \tensor Y)$ as
  \[
    \Brac{\sqrt{\frac{d^2 + 2d}{2}} \Pisym (x \tensor x) + \frac 1 {\sqrt d} \Paren{1 - \sqrt{\frac{d+2}{2}}} \Phi \Phi^\top (x \tensor x)} \tensor X \mper
  \]
  Since $\Pisym (x \tensor x) = (x \tensor x)$ and $\Phi^\top (x \tensor x) = \|x\|^2 = 1$, this simplifes to
  \[
    \Brac{\sqrt{\frac{d^2 + 2d}{2}} (x \tensor x) + \frac 1 {\sqrt d} \Paren{1 - \sqrt{\frac{d+2}{2}}} \Phi } \tensor X \mper
  \]
\end{proof}

\subsection{Matrix Product Identities}

\begin{fact}
  \label[fact]{fact:Si-explicit}
    $S_i = (d_1 (a_i \tensor a_i) - d_2 \Phi ) \tensor B_i$.
\end{fact}
\begin{proof}
  Follows from the definition of $S_i$ and \cref{fact:pseudoinv}.
\end{proof}

\begin{fact}
  \label[fact]{fact:SiSj-explicit}
  $S_i^\top S_j = (d_1^2 \iprod{a_i,a_j}^2 - 2 d_1 d_2 + d_2^2 d) B_i^\top B_j$
\end{fact}
\begin{proof}
Expanding via \cref{fact:Si-explicit}, 
  \begin{align*}
    S_i^\top S_j & = [(d_1 (a_i \tensor a_i) - d_2 \Phi) \tensor B_i]^\top [(d_1 (a_j \tensor a_j) - d_2 \Phi) \tensor B_j]\\
    & = (d_1^2 \iprod{a_i,a_j}^2 - 2 d_1 d_2 + d_2^2 d) B_i^\top B_j\mper
  \end{align*}
  Here we have used that $\Phi^\top (a_i \tensor a_i) = \|a_i\|^2 = 1$ and that $\Phi^\top \Phi = \|\Phi\|^2 = d$.
\end{proof}

\begin{fact}
  \label[fact]{fact:mode-swap-unitary}
  $P = P^\top = P^{-1}$ and hence $P^2 = \Id$
\end{fact}
\begin{proof}
  Exercise.
\end{proof}

\begin{fact}
  \label[fact]{fact:mode-swap-1}
  For any matrics $B,B' \in \R^{d \times m}$, we have $(\Phi \tensor B)^\top P (\Phi \tensor B') = B^\top B'$
\end{fact}
\begin{proof}
  The $s,t$-th entry of $(\Phi \tensor B)^\top P (\Phi \tensor B')$ is given by $\sum_{uvw \leq d} \Phi_{uv} B_{sw} \Phi_{uw} B'_{tv}$. The only nonzero terms come from $u=v=w$, because otherise $\Phi_{uv}\Phi_{uw} = 0$. So this simplifes to $\sum_{u \leq d} B_{su} B'_{tu} = (B^\top B')_{st}$.
\end{proof}

\begin{fact}
  \label[fact]{fact:mode-swap-2}
  For any vectors $a,a' \in \R^d$ and matrices $B,B' \in \R^{d \times (d-1)}$, we have $(a \tensor a \tensor B)^\top P (a' \tensor a' \tensor B') = \iprod{a,a'} (B^\top a')([B']^\top a)^\top$
\end{fact}
\begin{proof}
  Since $P$ does not touch the first mode of $(R^d)^{\tensor 3}$ it is enough to compute $(a \tensor B)^\top P'(a' \tensor B')$ where $P'$ is the mode-swap matrix for $(\R^d)^{\tensor 2}$.
  The $s,t$-th entry of this matrix is given by 
  \[
  \sum_{uv \leq d} a_u B_{sv} a'_v B'_{tu} = (\sum_{u \leq d} a_v B'_{tu})(\sum_{v \leq d} a'_v B_{sv})\mper
  \]
\end{proof}

\begin{fact}
  \label[fact]{fact:mode-swap-3}
  For any vector $a \in \R^d$ and matrices $B,B' \in \R^{d \times m}$, we have $(a \tensor a \tensor B)^\top P (\Phi \tensor B') = (B^\top a) ([B']^\top a)^\top$
\end{fact}
\begin{proof}
  The $s,t$-th entry of the product is given by
  \[
  \sum_{uvw \leq d} a_u a_v B_{sw} \Phi_{uw} B'_{tv} = \sum_{uv} a_u a_v B_{su} B'_{tv} = (B^\top a)_s ([B']^\top a)_t
  \]
\end{proof}

\begin{fact}
  \label[fact]{fact:SiPSj-explicit}
  $S_i^\top P S_j = d_1^2 \iprod{a_i, a_j} (B_i^\top a_j)(B_j^\top a_i)^\top + d_2^2 B_i^\top B_j$.
\end{fact}
\begin{proof}
  Expanding $S_i,S_j$ using \cref{fact:Si-explicit}, we obtain
  \begin{align*}
  S_i^\top P S_j & = [d_1(a_i \tensor a_i) \tensor B_i]^\top P  [d_1 (a_j \tensor a_j) \tensor B_j]\\
  & - [d_1(a_i \tensor a_i) \tensor B_i]^\top P [d_2 \Phi \tensor B_j]\\
  & - [d_2 \Phi \tensor B_i]^\top P [d_1 (a_j \tensor a_j) \tensor B_j]\\
  & + [d_2 \Phi \tensor B_i]^\top P [d_2 \Phi \tensor B_j]
  \end{align*}
  Simplifying the terms individually using \cref{fact:mode-swap-1}, \cref{fact:mode-swap-2}, and \cref{fact:mode-swap-3},
  \begin{align*}
  (a_i \tensor a_i \tensor B_i)^\top P (a_j \tensor a_j \tensor B_j) & = \iprod{a_i, a_j} (B_i^\top a_j)(B_j^\top a_i)^\top\\
  (\Phi \tensor B_i)^\top P (\Phi \tensor B_j) & = B_i^\top B_j\\
  (a_i \tensor a_i \tensor B_i)^\top P (\Phi \tensor B_j) & = 0\\
  (\Phi \tensor B_i)^\top P (a_j \tensor a_j \tensor B_j) & = 0 \mper
  \end{align*}
  where the last two equalities follow from $B_i^\top a_i = 0, B_j^\top a_j = 0$.
\end{proof}

\begin{fact}
  \label[fact]{fact:RiRj-explicit}
$R_i^\top R_j = 2(d_1^2 \iprod{a_i, a_j}^2 - 2 d_1 d_2 + d_2^2 (d-1) ) B_i^\top B_j - 2 d_1^2 \iprod{a_i, a_j} (B_i^\top a_j)(B_j^\top a_i)^\top$,
  and in particular,
    $R_i^\top R_i = 2(d_1^2 - 2 d_1 d_2 + d_2^2(d-1) ) \Id$.
\end{fact}
\begin{proof}
  By the definition of $R_i, R_j$ we expand
  \[
    R_i^\top R_j = (S_i - P S_i)^\top (S_j - P S_j)\mper
  \]
  We can expand the product and use \cref{fact:mode-swap-unitary} to get
  \[
    R_i^\top R_j = 2 S_i^\top S_j - 2 S_i^\top P S_j\mper
  \]
  Applying \cref{fact:SiSj-explicit} and \cref{fact:SiPSj-explicit}, we get
  \[
    R_i^\top R_j = 2 (d_1^2 \iprod{a_i,a_j}^2 - 2 d_1 d_2 + d_2^2 d) B_i^\top B_j - 2 d_1^2 \iprod{a_i, a_j} (B_i^\top a_j)(B_j^\top a_i)^\top - 2 d_2^2 B_i^\top B_j\mper
  \]
  Simplifying finishes the proof.
 \end{proof}

\subsection{Naive Spectral Norm Estimate}
We will need an upper bound on the spectral norm of the matrix $R$, which we obtain by a matrix Chernoff bound.
To prove that, we need spectral bounds on certain second moments.

\begin{fact}
  \label[fact]{fact:Si-spectral}
  For all $i \leq n$,
  \[
  \| \E S_i S_i^\top \| \leq 1\mper
  \]
\end{fact}
\begin{proof}
  Expanding by the definition of $S_i$,
  \[
  \E S_i S_i^\top = A^{-1/2} \E (a_ia_i^\top \tensor a_ia_i^\top \tensor B_i B_i^\top) A^{-1/2} \preceq A^{-1/2} (\E a_ia_i^\top \tensor a_ia_i^\top \tensor \Id) A^{-1/2} \preceq \Id.
  \]
  since $B_iB_i^\top = \Id - a_i a_i^\top \preceq \Id$.
  The last equality uses the definition of $A$.
\end{proof}

\begin{fact}
  \label[fact]{fact:Ri-spectral}
  For all $i \leq n$,
  \[
   \| \E R_i R_i^\top \| \leq 4\mper
  \]
\end{fact}
\begin{proof}
  Follows from \cref{fact:Si-spectral} and the definition of $R_i$, and the fact that $\|P\| \leq 1$, by the manipulations
  \[
  \| \E (R_i R_i^\top) \| =  \| \E S_i S_i^\top - P S_i S_i^\top - S_i S_i^\top P + P S_i S_i^\top P \| \leq 4 \| \E S_i S_i^\top \| \leq 4\mper
  \]
\end{proof}

\begin{fact}
  \label[fact]{fact:R-spectral-1}
  $\|\E RR^\top \| \leq O(n)$
\end{fact}
\begin{proof}
  Since $\E RR^\top = \E \sum_{i \leq n} R_i R_i^\top$, this follows from the triangle inequality and \cref{fact:Ri-spectral}.
\end{proof}

\begin{fact}
  \label[fact]{fact:R-spectral-2}
  $\| \E R^\top R \| \leq 2 (d_1^2 - 2 d_1 d_2 + d_2^2(d-1)) \leq O(d^2)$
\end{fact}
\begin{proof}
  $R^\top R$ is a block matrix with $ij$-th block being $R_i^\top R_j$. If $i \neq j$ we have $\E R_i^\top R_j = 0$, since $R_i^\top R_j$ is independent of $R_j$ and has expectation zero.
  At the same time $R_i^\top R_i = 2 (d_1^2 - 2 d_1 d_2 + d_2^2(d-1)) \Id$.
  So, $\E R^\top R = 2 (d_1^2 - 2 d_1 d_2 + d_2^2(d-1)) \Id$.
\end{proof}

\begin{fact}
  \label[fact]{fact:R-spectral}
  $\E \|R\| \leq O(\log d \cdot \max(d,\sqrt{n}))$.
\end{fact}
\begin{proof}
  First of all, note that $\E R = 0$, because $\E (a_i \tensor a_i \tensor B_i) = \E_a a_i \tensor a_i \tensor (\E[ B_i \, | \, a_i ]) = 0$, since each column of $B_i$ is a random unit vector in the orthogonal complement of $a_i$.

  Also, note that with probability $1$, each of the column blocks $R_i$ has $\|R_i\| \leq O(d)$, because $\|a_i \tensor a_i \tensor B_i\| \leq 1$ and $\|A^{-1/2}\| \leq O(d)$.

  The matrix $R$ is a sum of independent random matrices.
  To apply Matrix Chernoff, we need the bounds on its second moments, provided by \cref{fact:R-spectral-1} and \cref{fact:R-spectral-2}.
  The proof is concluded by applying Matrix Chernoff.
\end{proof}

\subsection{Off-Diagonal Second Moment Estimates}
We will eventually need to bound the norms of some off-diagonal blocks of the matrix $R^\top R$.
We prove some useful inequalities for that effort now.

\begin{fact}
  \label[fact]{fact:matrix-cs}
  Let $X,Y$ be real matrices.
  Then $(X-Y)(X-Y)^\top \preceq 2 (XX^\top + YY^\top)$.
\end{fact}
\begin{proof}
  By expanding, $(X-Y)(X-Y)^\top = XX^\top + YY^\top - XY^\top - X Y^\top$.
  Since $(X+Y)(X+Y)^\top \succeq 0$, we obtain $XX^\top + YY^\top \succeq -XY^\top - YX^\top$, finishing the proof.
\end{proof}

\begin{fact}
  \label[fact]{fact:R-var-1}
  For any $S \subseteq [n]$ and fixed $a_i$ for $i \notin S$, we have $\| \E_{R_S} R_{\overline S}^\top R_S R_S^\top R_{\overline S} \| \leq O(|S| \cdot \|R_{\overline S}\|^2)$.
\end{fact}
\begin{proof}
  We know from \cref{fact:Ri-spectral} that $\E R_i R_i^\top \preceq 4 \Id$.
  Hence $\E R_S R_S^\top = \sum_{i \in S} \E R_i R_i^\top \preceq 4|S|\Id$.
  To prove the final bound we push the expectation inside the matrix product:
  \[
  \E_{R_S} R_{\overline S}^\top R_S R_S^\top R_{\overline S} = R_{\overline S}^\top (\E_{R_S} R_S R_S^\top) R_{\overline S} \preceq 4 |S| \cdot R_{\overline S}^\top R_{\overline S} \preceq 4 \cdot |S| \cdot \|R_{\overline S}\|^2 \cdot \Id \mper
  \]
\end{proof}

\begin{fact}
  \label[fact]{fact:R-var-2}
  For any $S \subseteq [n]$ and fixed $a_i$ for $i \notin S$, we have  $\| \E_{R_S} R_{S}^\top R_{\overline S} R_{\overline S}^\top R_{S} \| \leq O(|S| \cdot d^2) + O(d^3 \cdot \|\sum_{i \in \overline S} a_i a_i^\top \|)$.
\end{fact}
\begin{proof}
  Consider the $j,k$-th block of $\E_{R_S} R_{S}^\top R_{\overline S} R_{\overline S}^\top R_{S}$, which expands to
  \[
  \E_{R_j, R_k} \sum_{i \in \overline S} R_{j}^\top R_i R_i^\top R_k\mper
  \]
  Since $\E R_j = \E R_k = 0$, unless $j=k$ the whole expression vanishes in expectation.
  Consider the case $j=k$.
  Here we have the matrix $\sum_{i \in \overline S} \E R_j^\top  R_i R_i^\top R_j$.
  We expand the matrix $R_j^\top R_i$ according to \cref{fact:RiRj-explicit} to get
  \[
    R_j^\top R_i = \underbrace{2(d_1^2 \iprod{a_i, a_j}^2 - 2 d_1 d_2 + d_2^2 (d-1) ) B_j^\top B_i}_{\defeq X_{ij}} - \underbrace{2 d_1^2 \iprod{a_i, a_j} (B_j^\top a_i)(B_i^\top a_j)^\top}_{\defeq Y_{ij}}
  \]
  We will need the following two spectral bounds.
  \begin{enumerate}
  \item $\| \E 4 (d_1^2 \iprod{a_i, a_j}^2 - 2 d_1 d_2 + d_2^2 (d-1))^2 B_j^\top B_i B_i^\top B_j \| \leq O(d^2)$.

  We note that $B_j^\top B_i B_i^\top B_j \preceq \Id$, so it is enough to bound $\E 4 (d_1^2 \iprod{a_i, a_j}^2 - 2 d_1 d_2 + d_2^2 (d-1))^2$.
  By definition, $|2 d_1 d_2|, d_2^2 (d-1) \leq O(d)$ and $d_1^2 \leq O(d^2)$, so we have
  \[
  \E 4 (d_1^2 \iprod{a_i, a_j}^2 - 2 d_1 d_2 + d_2^2 (d-1))^2 \leq O(d^4) \cdot \E (\iprod{a_i,a_j}^2 + O(1/d))^2 \leq O(d^2)\mper
  \]

  \item $4 d_1^4 \iprod{a_i, a_j}^2 (B_j^\top a_i)(B_i^\top a_j)^\top (B_i^\top a_j) (B_j^\top a_i)^\top \preceq 4 d_1^4 \iprod{a_i, a_j}^2 B_j^\top a_i a_i^\top B_j\ $.

  We note that $(B_i^\top a_j)^\top (B_i^\top a_j) = \|B_i^\top a_j\|^2 \leq 1$.
  So,
  \[
    4 d_1^4 \iprod{a_i, a_j}^2 (B_j^\top a_i)(B_i^\top a_j)^\top (B_i^\top a_j) (B_j^\top a_i)^\top \preceq 4 d_1^4 \iprod{a_i, a_j}^2 B_j^\top a_i a_i^\top B_j\mper
  \]

  We return to bounding $\sum_{i \in \overline S} \E R_j^\top R_i R_i^\top R_j$, where we recall that each $R_i$ in the sum is fixed and the expectation is over $R_j$.

  By \cref{fact:matrix-cs},
  \begin{align*}
    \sum_{i \in \overline S} \E R_j^\top R_i R_i^\top R_j
    & \preceq 2 \sum_{i \in \overline{S}} \E X_{ij} X_{ij}^\top + 2 \E Y_{ij} Y_{ij}^\top\\
    & \preceq 2 \sum_{i \in \overline S} \| \E X_{ij} X_{ij}^\top \| \cdot \Id + 2 \sum_{i \in \overline S} Y_{ij} Y_{ij}^\top \text{ by triangle inequality}\\
    & \preceq O(|S| \cdot d^2) \cdot \Id + 2 \sum_{i \in \overline S} Y_{ij} Y_{ij}^\top \text{ by (1) above}\\
    & \preceq O(|S| \cdot d^2) \cdot \Id + 2 \E B_j \sum_{i \in \overline S} 4 d_1^4 \iprod{a_i,a_j}^2  a_i a_i^\top B_j \text{ by (2) above}\\
    & \preceq O(|S| \cdot d^2) \cdot \Id + O(d_1^4) \cdot \E\iprod{a_i,a_j}^2 \cdot \|\sum_{i \in \overline S} a_i a_i^\top \| \cdot \Id  \text{ by $\|B_j\| \leq 1$}\\
    & \preceq O(|S| \cdot d^2) \cdot \Id + O(d^3) \cdot \cdot \|\sum_{i \in \overline S} a_i a_i^\top \| \cdot \Id  \text{ by $\E \iprod{a_i,a_j}^2 \leq O(1/d)$ }
    \end{align*}

  \end{enumerate}
\end{proof}

\subsection{Matrix Decoupling}
\begin{fact}[Block Matrix Decoupling, similar to Lemma 5.63 of \cite{MR2963170-Vershynin12}]
  \label[fact]{fact:matrix-decoupling}
  Let $R$ be an $N \times nm$ random matrix, consisting of $n$ blocks $R_i$ of dimension $N \times m$.
  Suppose that the blocks satisfy $R_i^\top R_i = \Id$.
  For a subset $S \subseteq [n]$, let $R_S \in \R^{N \times n|S|}$ matrix consisting of only the blocks in $S$.
  Let $T \subseteq [n]$ be uniformly random.
  Then
  \[
  \E_R \|R^\top R - \Id\| \leq 4 \E_{R,T} \|R_T R_{[n] \setminus T}^\top \|\mper
  \]
\end{fact}
\begin{proof}
  Following the argument of Vershynin \cite{MR2963170-Vershynin12}, we note that
  \begin{align*}
    \|R^\top R - \Id\| = | \sup_{\|x\|=1} \|Rx\|^2 - 1 |
  \end{align*}
  and that for $x = (x_1,\ldots,x_n) \in \R^{nm}$,
  \[
    \|Rx\|^2 = \sum_{i \in [n]} x_i^\top R_i^\top R_i x_i + \sum_{i \neq j \in [n]} x_i^\top R_i^\top R_j x_j = 1 + \sum_{i \neq j \in [n]} x_i^\top R_i^\top R_j x_j 
  \]
  since $R_iR_i^\top = \Id$.
  By scalar decoupling (see Vershynin, Lemma 5.60 \cite{MR2963170-Vershynin12}), for a uniformly random subset $T \subseteq [n]$,
  \[
    \sum_{i \neq j \in [n]} x_i^\top R_i^\top R_j x_j = 4 \E_T \sum_{i \in T, j \in [n] \setminus T} x_i^\top R_i^\top R_j x_j\mper
  \]
  So,
  \begin{align*}
    \E_R \|R^\top R - \Id\|
    & = \E_R | \sup_{\|x\|=1} \|Rx\|^2 - 1 |\\
    & = \E_R \Abs{ \sup_{\|x\|=1} 4 \E_T \sum_{i \in T, j \in [n] \setminus T} x_i^\top R_i^\top R_j x_j }\\
    & \leq 4 \E_{R,T} \Abs{ \sup_{\|x\|=1} \sum_{i \in T, j \in [n] \setminus T} x_i^\top R_i^\top R_j x_j }\\
    & = 4 \E_{R,T} \| R_T^\top R_{[n] \setminus T} \|
  \end{align*}
  where the inequality above follows by Jensen's inequality.
\end{proof}

\subsection{Putting It Together}
We are ready to prove \cref{lem:basic-swap-cond}.

\begin{proof}[Proof of \cref{lem:basic-swap-cond}]
  We will show that $R^\top R  \in \R^{n(d-1) \times n(d-1)}$ is close to $\Theta(d^2) \cdot \Id$.
  The $(i,j)$-th block of $R^\top R$ is given by $R_i^\top R_j$.
  Let us first consider the diagonal blocks, $R_i^\top R_i$.

  Using \cref{fact:RiRj-explicit} and the bounds $d_1 = \Theta(d), d_2 = \Theta(1)$, we obtain,
  \[
    R_i R_i^\top = \Theta(d) \cdot \Id\mper
  \]
  Next, we bound the norm of the off-diagonal part of the matrix.
  Let $\R_{\diag}$ be the matrix equal to $R^\top$ only on diagonal blocks and zero elsewhere.
  We will use matrix decoupling, \cref{fact:matrix-decoupling}, to bound $\E \| R^\top R - \R_{\diag} \|$.
  We find that for a uniformly random subset $T \subseteq [n]$,
  \[
  \E_R \|R^\top R - R_{\diag} \| \leq 4 \E_{R,T} \|R_T^\top R_{\overline T} \|
  \]

  Now fix $T \subseteq [n]$ and for $i \in \overline T$ fix unit vectors $a_i$ and orthonormal bases $B_i$ to obtain the $R_i$ for $i \in \overline {T}$.
  We regard the matrix $R_T^\top R_{\overline T}$ as a sum of independent matrices, one for each $i \in T$.
  The $i$-th such matrix $V_i$ has $|\overline {T}|$ blocks; the $j$-th block is $R_j^\top R_i$.

  First we note that $\|V_i\| \leq O(d) \cdot \|R_{\overline T}\|$ with probability one over $R_i$, since $\|R_i\| \leq O(d)$ with probability 1, by \cref{fact:pseudoinv}.

  Next, we note the bounds on the variance of $R_T^\top R_{\overline{T}}$ afforded by \cref{fact:R-var-1} and \cref{fact:R-var-2}.
  We conclude that by Matrix Bernstein,
  \[
  \E_{R_T} \| R_T^\top R_{\overline T} \| \leq \log d \cdot O(\max(\sqrt{|T|} \cdot \|R_{\overline T}\|, \sqrt{|T|} \cdot d, d^{3/2} \|\sum_{i \in \overline{T}} a_i a_i^\top \|^{1/2}, d)\mper
  \]
  So,
  \[
  \E_R \| R^\top R - R_{\diag} \| \leq O(\log d) \cdot \Paren{ \E_{T,R} |T|^{1/2} \cdot \|R_{\overline T}\| + d \E_{T,R} \sqrt{|T|} + d^{3/2} \E \| \sum_{i \in \overline{T}} a_i a_i^\top \|^{1/2} + d}
  \]
  Using \cref{fact:R-spectral}, we have $\E |T|^{1/2} \|R_{\overline T}\| \leq \sqrt{n} \cdot O(\log d \cdot \max(d,\sqrt{n}))$.
  By standard matrix concentration, $\E \|\sum_{i \in \overline{T}}\| \leq \log d \cdot O(\max(1, n/d))$.
  So putting it together,
  \[
  \E_R \|R^\top R - R_{\diag} \| \leq O(\log d)^2 \cdot \max(d\sqrt{n}, n, d^{3/2})\mper
  \]
\end{proof}

\subsection{Omitted Proofs}\label[section]{sec:omitted-proofs}

We turn now to the proof of \cref{fact:U-projector}.
The strategy is much the same as the proof of \cref{lem:basic-swap-cond}, which is in turn an adaptation of an argument due to Vershynin for concentration of matrices with independent columns \cite{MR2963170-Vershynin12}.

We will use the following simple matrix decoupling inequality, with $B$ being the random matrix having columns $a_i^{\tensor 3}$.

\begin{lemma}[Matrix decoupling, Lemma 5.63 in \cite{MR2963170-Vershynin12}]
\label[lemma]{lem:simple-matrix-decoupling}
  Let $B$ be an $N \times n$ random matrix whose columns $B_j$ satisfy $\|B_j\| = 1$.
  For any $T \subseteq [n]$, let $B_T$ be the restriction of $B$ to the columns $T$.
  Let $T$ be a uniformly random set of columns.
  Then
  \[
  \E \| B^\top B - \Id \| \leq 4 \E_T \E_B \| (B_T)^\top B_{[n] \setminus T} \| \mper
  \]
\end{lemma}

\begin{proof}[Proof of \cref{fact:U-projector}]
  Let $B$ have columns $a_i^{\tensor 3}$.
  Fix $T \subseteq [n]$.
  We will bound $\E \| (B_T)^\top B_{[n] \setminus T} \|$, with the goal of applying \cref{lem:simple-matrix-decoupling}.

  For $i \in T$, the $i$-th row of $B_T^\top B_{[n] \setminus T}$ has entries $\iprod{a_j, a_i}^3$ for $j \notin T$.
  Let us temporarily fix $a_i$ for $i \notin T$; then these rows become independent due to independence of $\{a_j\}_{j \in T}$.

  We think of the matrix $B_T^\top B_{[n] \setminus T}$ as consisting of a sum of $|T|$ independent matrices where only the $i$-th row of the $i$-th matrix $M_i$ is nonzero, and it consists of entries $M_{ij} = (B_T^\top B_{[n] \setminus T})_{ij}$.
  We are going to apply the matrix Bernstein inequality to the sum $B_T^\top B_{[n]\setminus T} = \sum_{i \in T} M_i$.

  To do so, we need to compute the variance of the sum: we need to bound
  \[
  \sigma^2 = \max \Paren{ \E \sum_{i \in T} M_i M_i^\top, \E \sum_{i \in T} M_i^\top M_i }\mper
  \]
  (Here the expectation is over $a_i$ for $i \in T$; we are conditioning on $a_i$ for $i \notin T$.)
  For the former, consider that $M_i M_i^\top$ has just one nonzero entry,
  \[
  \E (M_i M_i^\top)_{i,i} = \E \sum_{i \notin T} \iprod{a_i, a_j}^6 \leq O(|[n] \setminus T| / d^3) \mper
  \]
  Hence $\E \sum_{i \in T} M_i M_i^\top \preceq O(|[n] \setminus T| / d^3)$.

  Next we bound $\E \sum_{i \in T} M_i^\top M_i$.
  Let $r$ be a random vector with entries $\iprod{a,a_i}^3$ for $i \notin T$ and $a$ a random unit vector.
  Then $\E \sum_{i \in T} M_i^\top M_i = |T| \cdot \E rr^\top$.
  We may compute that
  \[
  (\E rr^\top)_{ij} = \E \iprod{a, a_i}^3 \iprod{a, a_j}^3 = \frac C {d^3} \iprod{a_i, a_j} + \frac 1 {d^3} \cdot O(\iprod{a_i, a_j}^3)
  \]
  where $C$ is a universal constant.
  (This may be seen by comparison to the case that $a$ is replaced by a standard Gaussian and using Wick's theorem on moments of a multivariate Gaussian.)
  Letting $C_1 = (\sum_{i,j \notin T} \iprod{a_i, a_j}^6/d^6)^{1/2}$ and $C_2 = \|A_{[n] \setminus T}\|$ (where $A$ has columns $a_i$), we find that
  $\sigma^2 \leq O(\max(|[n] \setminus T| / d^3, |T| \cdot d^{-3} \cdot C_2 , |T| C_1))$.

  By applying Matrix Bernstein, for each $t > 0$ we obtain
  \[
  \E \Norm{ \sum_{i \in T} M_i \cdot \Ind(\|M_i\| \leq t) } \leq O(\log d) \cdot \sqrt{\max(|[n] \setminus T| / d^3, |T| \cdot d^{-3} \cdot C_2, C_1,t^2)}
  \]
  where again the expectation is over only $a_i$ for $i \in T$.
  Choosing $t = \tilde O(\sqrt{n /d^3})$, by standard scalar concentration we obtain that $\Pr(\|M_i\| > t) \leq (dn)^{-\omega(1)}$.
  Hence by Cauchy-Schwarz we find $\E \| \sum_{i\in T} M_i (1 - \Ind(\|M_i\| \leq t) ) \| \leq (dn)^{-\omega(1)}$, so all in all,
  \[
  \E \Norm{ \sum_{i \in T} M_i } \leq (\log nd)^{O(1)} \cdot \sqrt{\max(|[n] \setminus T| / d^3, |T| \cdot d^{-3} \cdot C_2, |T| C_1, n/d^3)}\mper
  \]

  Finally, we have to bound
  \[
  \E_T \E_{a_i, i \notin T} \sqrt{\max(|[n] \setminus T| / d^3, |T| \cdot d^{-3} \cdot C_2, |T| C_1, n/d^3)}
  \]
  which is an upper bound on $\E \| B_T^\top B_{[n] \setminus T}\|$.
  By Cauchy-Schwarz, we may upper bound this by
  \[
  (\E |[n] \setminus T| / d^3)^{1/2} + (\E |T| \cdot d^{-3} \cdot C_2)^{1/2} + (\E |T| C_1)^{1/2} + (n/d^3)^{1/2}\mper
  \]
  By standard matrix concentration, $\E C_2 \leq O(1 + n/d)$.
  Clearly $\E |[n] \setminus T|$ and $\E |T|  \leq O(n)$.
  Finally, by straightforward computation, $\E C_1 \leq n^2 / d^{4.5}$.

  All together, applying \cref{lem:simple-matrix-decoupling}, we have obtained $\E \|B^\top B - \Id \| \leq \tilde O(n/d^2)$ for $n \gg d$.
  Since $B^\top B$ has the same eigenvalues as $B B^\top = U$, we are done.
\end{proof}

    \section*{Acknowledgements}
We thank David Steurer for many helpful conversations regarding the technical content and presentation of this work.


  \phantomsection
  \addcontentsline{toc}{section}{References}
  \bibliographystyle{amsalpha}
  \bibliography{bib/mathreview,bib/dblp,bib/custom,bib/scholar}

\appendix


    \section{Tools for analysis and implementation}
\label[section]{sec:toolbox-appendix}

\restatelemma{lem:implicit-tensor}
\begin{proof}\
\nopagebreak[4]
\vspace{-1em}
\paragraph{Tensor contraction}
We start with multiplication of two vectors $x,y \in \R^{d^2}$ into two of the modes of $\T$.
Without loss of generality (by interchange of $U$ and $V$), there are two cases: we want either to compute the vector flattening of $\transpose{(x \ot \Id_d)}U\transpose{V}(y \ot \Id_d)$, or, expressing $x = \sum_{i=0}^d r_i\ot s_i$, we want $\sum_i \transpose{(\Id_d \ot \Id_d \ot r_i)}U\transpose{V}(y \ot s_i)$.
For both these cases, we first compute $\transpose{V}(y \ot \Id_d)$.

We compute $\transpose{V}(y \ot \Id_d)$ as
$[\transpose{V}(y \ot \Id_d)]_{\cdot;i} = \transpose{V}{}_{\cdot;(\cdot,\cdot,i)}y$.
This is a concatenation of $d$ different matrix-vector multiplications using $n \times d^2$ matrices, and so it takes $O(nd^3)$ time.

Then to find $(\transpose{x} \ot \Id_d)U\transpose{V}(y \ot \Id_d)$, we simply repeat the above procedure to find $(\transpose{x} \ot \Id_d)U$ and then multiply the $d \times n$ and $n \times d$ matrices together in $O(nd^2)$ time.

To find $\sum_i \transpose{(\Id_d \ot \Id_d \ot r_i)}U\transpose{V}(y \ot s_i)$ after finding the rank decomposition $x = \sum_{i=0}^d r_i\ot s_i$ which takes $O(d^3)$ time by SVD, we multiply each $s_i$ into our computed value of $\transpose{V}(y \ot \Id_d)$ to obtain $d$ different $n$-dimensional vectors $t_i = \transpose{V}(y \ot s_i)$.
Since there are $d$ of these vectors and each is a matrix-vector multiplication with an $n \times d$ matrix, this takes $O(nd^2)$ time.
Then $\sum_i \transpose{(\Id_d \ot \Id_d \ot r_i)}Ut_i$ can be reshaped as a multiplication of a $d^2 \times nd$ reshaping of $U$ with the vector $\sum_i t_i \otimes r_i$.
It takes $O(nd^3)$ time to perform the matrix-vector multiplication, and $O(nd^2)$ time to sum up $\sum_i t_i \otimes r_i$.

\paragraph{Spectral truncation}Next, we truncate the larger-than-$1$ singular values of the $(\{1\},\{2,3\})$ and $(\{3\},\{1,2\})$ matrix reshapings $R \in \R^{d^2 \times d^4}$ of $\T$.
Without loss of generality, suppose we are in the $(\{3\},\{1,2\})$ case.
In this case, we would like to find the right-singular vectors and singular values of the operator that takes $y \in \R^{d^2}$ to the vector flattening of the $d \times d^3$ matrix $U\transpose{V}(y \ot \Id_d)$.
Letting $Z$ be the $nd \times d^2$ reshaping of $V$, this is the same as $(U \ot \Id_d)Zy$, which shares its right-singular vectors with $M:= \transpose{Z}(\transpose{U}U \ot \Id_d)Z$.

We claim that matrix-vector multiplication by $M$ can be implemented in $O(nd^3)$ time, with $O(n^2d^3)$ preprocessing time for computing the product $\transpose{U}U$.
The matrix-vector multiplications by $Z$ and $\transpose{Z}$ take time $O(nd^3)$, and then multiplying $Zy$ by $\transpose{U}U \ot \Id_d$ is reshaping-equivalent to multiplying $\transpose{U}U$ into the $n \times d$ matrix reshaping of $Zy$, which takes $O(n^2d)$ time with the precomputed $n \times n$ matrix $\transpose{U}U$.
Therefore, LazySVD~\cite[Corollary 4.4]{allen2016lazysvd} takes time $\tilde O(n^2d^3\delta^{-1/2})$ to yield a rank-$k$ eigendecomposition $P\Lambda\transpose{P}$ such that $\|M^{1/2} - P\Lambda^{1/2}\transpose{P}\| \le (1+\delta)\rho_k$.

To obtain the output $\T'$ of this procedure, let $(P\Lambda^{1/2}\transpose{P} - \Id)^{>0}$ be $P\Lambda^{1/2}\transpose{P} - \Id$ with all of its nonpositive eigenvalues removed: this may be implemented by removing nonpositive entries from $\Lambda^{1/2} - \Id$.
Then implicitly multiply $(\Id - (P\Lambda^{1/2}\transpose{P} - \Id)^{>0})$ into the third mode of $\T$ (although this matrix has rank larger than $n$, we may implement it by implicitly subtracting $(\Id \ot \Id \ot (P\Lambda^{1/2}\transpose{P} - \Id)^{>0})\T$ from $\T$).
We are trying to approximate multiplying $(\Id - (M^{1/2}-\Id)^{>0})$ into the third mode of $\T$, so let $\Delta = (M^{1/2}-\Id)^{>0} -(P\Lambda^{1/2}\transpose{P} - \Id)^{>0}$ be the difference.
Then $\|\Delta\| \le \|M^{1/2} - P\Lambda^{1/2}\transpose{P}\| \le (1 + \delta)\rho_k$, so that we suffer an additive error of at most $(1+\delta)\rho_k$ in spectral norm.
And the final error in the low-rank representation is $(\Delta \ot \Id)U\transpose{V}$.
Since $\|\Delta \ot \Id\| \le (1 + \delta)\rho_k$ and $U\transpose{V}$ has Frobenius norm $\|\T\|_F$, we find a final error of $(1 + \delta)\rho_k\|\T\|_F$ in Frobenius norm.

\paragraph{Implicit matrix multiplication}Finally, to implicitly multiply a $d^2 \times d^2$ rank-$n$ matrix $R$ into a mode of $\T$, simply store the singular value decomposition $R = P\Sigma\transpose{Q}$.
Whenever a vector needs to be multiplied into that mode in the future, multiply that vector by $R$ before carrying out the implicit tensor operation as previously specified, and if a vector needs to be output from that mode, multiply it by $\transpose{R}$ before outputting.
This incurs a time cost of $O(nd^2)$ per operation.

A special case arises in the spectral truncation operation, where we do not allow implicit multiplication to have been done in the second mode.
Suppose then without loss of generality that $R$ was multiplied into the first mode of $\T$ and we truncate the $(\{3\},\{1,2\})$ matrix reshaping.
Then we will have to compute $\transpose{U}(\dyad{R} \ot \Id_d)U$ instead of $\transpose{U}U$ in the preprocessing step.
This can be done by multiplying $\transpose{R} = Q\Sigma\transpose{P}$ with the $d^2 \times nd$ reshaping of $U$, which takes $O(n^2d^3)$ time per future spectral truncation operation.
\end{proof}

\section{Notes on \cref{fig:algs}}
\label[section]{sec:table-notes}
We record a few notes on parameter regimes used to compare various algorithms for tensor decomposition in \cref{fig:algs}.

\begin{itemize}
  \item Robust algorithms with algebraic assumptions often require $\|E\| \leq \sigma(a_1,\ldots,a_n)$, where $\sigma(a_1,\ldots,a_n)$ is some measure of well-conditioned-ness of $a_1,\ldots,a_n$, the details of which may vary from algorithm to algorithm. In this table we report results for the setting that $\sigma(a_1,\ldots,a_n) \geq \Omega(1)$; such values of $\sigma$ (for all the notions of well-conditioned-ness represented) are achieved by random $a_1,\ldots,a_n$.
  \item The algorithm of \cite{DBLP:journals/jmlr/AnandkumarGJ17} is phrased for 3-tensors rathern than 4-tensors; this is the origin of the rank bound $n \leq d^{1.5}$ rather than $d \leq n^2$ achieved by algorithms for $4$-tensors.
  In general for $k$-tensors one expects efficient algorithms to tolerate overcompleteness $n \leq d^{k/2}$ (despite tensor rank factorizations remaining unique for much larger $n$), so the overcompleteness guarantee of \cite{DBLP:journals/jmlr/AnandkumarGJ17} is comparable to the other algorithms.
  \item We have estimated the running time of the SoS algorithm of \cite{DBLP:conf/focs/MaSS16} by assuming that the semidefinite programs involved are solved using standard black-box techniques (e.g the ellipsiod method).
  \end{itemize}

    \section{Simulations for condition number of random tensors}
\label[section]{sec:sims}

In this section we report on computer simulations which strongly suggest that if the components $a_1,\ldots,a_n$ are $n \ll d^2$ random unit vectors from a variety of ensembles, then with high probability $\kappa(a_1,\ldots,a_n) \geq \Omega(1)$. 
The ensembes include:
\begin{enumerate}
\item {\em Spherical measure:} $a_1,\ldots,a_n \in \R^d$ are i.i.d. random unit vectors (see \cref{fig:sphere}).
\item {\em Sparse:} $a_1,\ldots,a_n \in \R^d$ are sampled i.i.d. by choosing $\frac{1}{4}d$ coordinates in $[d]$ uniformly at random, sampling each of those coordinates from $\cN(0,1)$, and setting the rest to $0$ (see \cref{fig:sparse}).
\item {\em Hypercube:} $a_1,\ldots,a_n \in \R^d$ are i.i.d. samples from $\{0,1\}^d$ (see \cref{fig:cube}).
\item {\em Spiked covariance:} $a_1,\ldots,a_n \in \R^d$ are sampled from $\cN(0,\Id + \lambda \cdot uu^\top)$ for a random unit vector $u$ and $\lambda >0$. 
We note that in this case, though the covariance matrix of $a_1,\ldots,a_n$ has condition number $O(\frac{1}{\lambda})$, our experimental results support the hypothesis that $\kappa(a_1,\ldots,a_n) = \Omega(1)$ for $\lambda$ as large as $\lambda = \frac{1}{2}d$ (see \cref{fig:spike}).
\end{enumerate}
These ensembles are designed to capture a number of characteristics of real data which we would like the condition number to be robust to: sparsity, discrete values, and correlations (of relatively extreme magnitude).

In each of these cases, we computed $\kappa$ for several values of $n,d$ on with $a_1,\ldots,a_n$ taken to be i.i.d. uniformly random unit vectors.
Our results are consistent with the hypothesis that (with high probability) $\kappa(a_1,\ldots,a_n) \ge c - \tilde{O}(n/d^2)$ for some absolute constant $c \approx \frac{1}{2}$.

\begin{figure}
\centering
\includegraphics[scale=0.7]{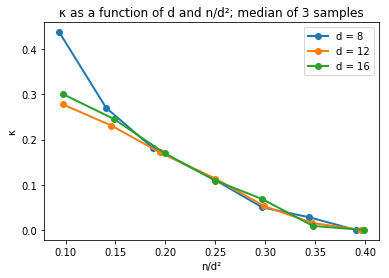}
\caption{Condition number as a function of dimension $d$ and lifted overcompleteness $n/d^2$ for vectors sampled from the spherical measure.}
\label{fig:sphere}
\end{figure}
\begin{figure}
\centering
\includegraphics[scale=0.7]{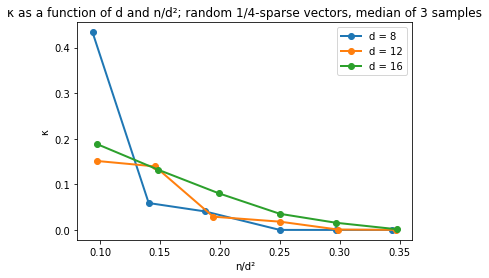}
\caption{Condition number as a function of dimension $d$ and lifted overcompleteness $n/d^2$ for random sparse vectors.}
\label{fig:sparse}
\end{figure}
\begin{figure}
\centering
\includegraphics[scale=0.7]{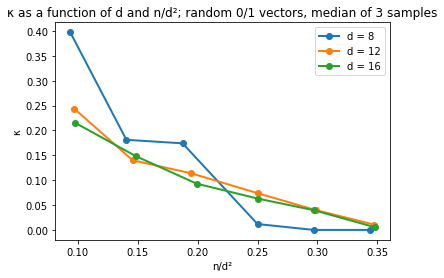}
\caption{Condition number as a function of dimension $d$ and lifted overcompleteness $n/d^2$ for vectors sampled from the Boolean hypercube.}
\label{fig:cube}
\end{figure}
\begin{figure}
\centering
\includegraphics[scale=0.7]{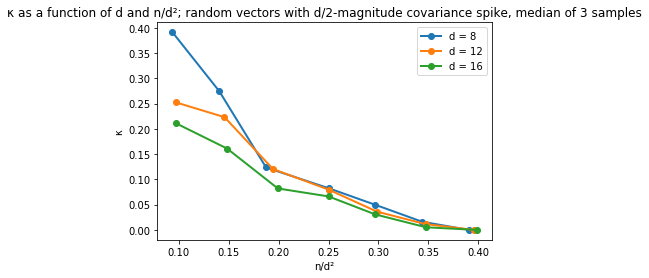}
\caption{Condition number as a function of dimension $d$ and lifted overcompleteness $n/d^2$ for vectors sampled from $\cN(0,\Id + \frac{1}{2}d\cdot uu^\top)$ for a random unit vector $u$.}
\label{fig:spike}
\end{figure}

We expect the values of $d,n$ employed here -- $d \approx 10, n \approx 100$, so that $\kappa$ is the condition number of a certain random matrix of dimensions about $10^3 \times 10^3$ -- to be predictive of the asymptotic behavior of $\kappa(a_1,\ldots,a_n)$, because the spectra of random matrices display strong concentration even in relatively small dimensions.

We also note that the hypothesis that $\kappa > c - \tilde{O}(n/d^2)$ is well supported by the fact that relatively standard techniques from random matrix theory yield the same bound for a closely related random matrix to $K(a_1,\ldots,a_n)$ from \cref{def:kappa}.
In particular, the following may be proved by a long but standard calculation, using Matrix Bernstein and decoupling inequalities:

\begin{lemma}[Condition number of basic swap matrix]
  \label[lemma]{lem:basic-swap-cond}
  Let $a_1,\ldots,a_n$ be independent random $d$-dimensional unit vectors.
  Let $B_i \in \R^{(d-1) \times d}$ be a random basis for the orthogonal complement of $a_i$ in $\R^d$.
  Let $P \in \R^{d^3 \times d^3}$ be the permutation matrix which swaps second and third modes of $(\R^d)^{\tensor 3}$.
  Let
  \[
    A = \E_{a} (a \tensor a \tensor \Id)(a \tensor a \tensor \Id)^\top\mper
  \]
  Let $R \in \R^{d^3 \times n(d-1)}$ have $n$ blocks of dimensions $d^3 \times (d-1)$, where the $i$-th block is
  \[
  R_i = A^{-1/2} (a_i \tensor a_i \tensor B_i) - P A^{-1/2} (a_i \tensor a_i \tensor B_i )
  \]
  where we abuse notation and denote the PSD square root of the pseudoinverse of $A$ by $A^{-1/2}$.
  Then there is a function $d'(d) = \Theta(d^2)$ such that $\E \|R^\top R - d'(d) \cdot \Id\| \leq O(\log d)^2 \cdot \max(d\sqrt{n}, n, d^{3/2})$.
  In particular, if $d \ll n \ll d^2$,
  \[
  \E \| \tfrac 1 {d'(d)} R^\top R - \Id \| \leq O( n (\log d)^2 / d^2)\mper
  \]
\end{lemma}

The matrix $R$ from this lemma differs from $K$ only in the use of $A^{-1/2}$ in place of $(H_1^\top H_1)^{-1/2}, (H_2^\top H_2)^{-1/2}$.
While we expect $A^{-1/2}$ (a non-random matrix) to be close to both $(H_1^\top H_1)^{-1/2}, (H_2^\top H_2)^{-1/2}$ (at least in subspaces close to $Im (H_1)$ and $Im(H_2)$, respectively) establishing this is a challenging task in random matrix theory -- in particular, both inverses of random matrices and spectra of random matrices with dependent entries are notoriously difficult to analyze.
We leave this challenge to future work.

\end{document}